\documentclass[11pt]{article}

\usepackage{geometry}
\usepackage{setspace}
\usepackage{amsmath, amssymb, amsfonts, bm, mathtools}
\usepackage{amsthm}
\usepackage[dvipsnames]{xcolor}
\definecolor{darkblue}{rgb}{0,0,.5}
\usepackage{graphicx}
\usepackage{subfigure}
\usepackage[numbers]{natbib}
\usepackage[colorlinks=true,allcolors=darkblue]{hyperref}       % hyperlinks
\usepackage{url}            % simple URL typesetting
\usepackage{booktabs}       % professional-quality tables
\usepackage{amsfonts}       % blackboard math symbols
\usepackage{nicefrac}       % compact symbols for 1/2, etc.
\usepackage{microtype}      % microtypography

\allowdisplaybreaks

\usepackage{float}
\usepackage{multirow}
\usepackage{footnote}
\usepackage{dsfont}

\usepackage{algorithm}
\usepackage{algorithmic}
\usepackage{nicefrac}
\usepackage{tikz}
\usepackage{overpic}

\usepackage{dsfont}
\usepackage{hyperref}
\usepackage[capitalize]{cleveref}
\usepackage{crossreftools}

\newcommand{\ind}{\mathds{1}}
\newcommand{\var}{\mathsf{Var}}

\newcommand{\brc}[1]{\left\{{#1}\right\}}
\newcommand{\prn}[1]{\left({#1}\right)} % parentheses
\newcommand{\brk}[1]{\left[{#1}\right]} % bracket
\newcommand{\norm}[1]{\left\|{#1}\right\|} % norm
\newcommand{\abs}[1]{\left|{#1}\right|} % norm
\newcommand{\what}[1]{\widehat{#1}}
\newcommand{\wtilde}[1]{\widetilde{#1}}

\newcommand{\normal}{\mathsf{N}}
\newcommand{\bindist}{\mathsf{Binomial}}
\newcommand{\<}{\langle} % Angle brackets
\renewcommand{\>}{\rangle}

\newcommand{\simiid}{\stackrel{\textup{iid}}{\sim}}

\newcommand{\cp}{\stackrel{p}{\rightarrow}}
\newcommand{\cas}{\stackrel{a.s.}{\rightarrow}}
\newcommand{\cweak}{\rightsquigarrow}
\def\gA{{\mathcal{A}}}
\def\gB{{\mathcal{B}}}

\def\gD{{\mathcal{D}}}

\def\gF{{\mathcal{F}}}

\def\gH{{\mathcal{H}}}
\def\gI{{\mathcal{I}}}

\def\gM{{\mathcal{M}}}

\def\gP{{\mathcal{P}}}

\def\gS{{\mathcal{S}}}
\def\gT{{\mathcal{T}}}

\def\gX{{\mathcal{X}}}

\def\RB{{\mathbb R}}
\def\EB{{\mathbb E}}

\def\DB{{\mathbb D}}
\def\PB{{\mathbb P}}

\def\NB{{\mathbb N}}
\def\GB{{\mathbb G}}

\def\ie{{\em i.e.\/}}

\newcommand{\KS}{\textup{KS}}
\newcommand{\TV}{\textup{TV}}

\makeatletter
\long\def\@makecaption#1#2{
  \vskip 0.8ex
  \setbox\@tempboxa\hbox{\small {\bf #1:} #2}
  \parindent 1.5em  %% How can we use the global value of this???
  \dimen0=\hsize
  \advance\dimen0 by -3em
  \ifdim \wd\@tempboxa >\dimen0
  \hbox to \hsize{
    \parindent 0em
    \hfil 
    \parbox{\dimen0}{\def\baselinestretch{0.96}\small
      {\bf #1.} #2
    } 
    \hfil}
  \else \hbox to \hsize{\hfil \box\@tempboxa \hfil}
  \fi
}
\makeatother

\newcommand{\FW}{\mathcal{F}_{W_1}}
\newcommand{\FKS}{\mathcal{F}_{\textup{KS}}}
\newcommand{\FTV}{\mathcal{F}_{\textup{TV}}}

\newcommand{\loc}{\textup{loc}}

\newcommand{\cov}{\mathsf{Cov}}

\usepackage{enumerate}

\newtheorem{claim}{Claim}[section]
\newtheorem{lemma}[claim]{Lemma}
\newtheorem{assumption}{Assumption}

\newtheorem{theorem}{Theorem}[section]
\newtheorem{example}{Example}[section]
\newtheorem{proposition}{Proposition}[section]
\newtheorem{remark}{Remark}[section]
\newtheorem{corollary}{Corollary}[section]

\usepackage{xr}
\title{Estimation and Inference in Distributional Reinforcement Learning}
\author{Liangyu Zhang\thanks{School of Statistics and Management, Shanghai University of Finance and Economics; email: \texttt{zhangliangyu@sufe.edu.cn}.} \and
Yang Peng\thanks{School of Mathematical Sciences, Peking University; email: \texttt{pengyang@pku.edu.cn}.} \and
Jiadong Liang\thanks{School of Mathematical Sciences, Peking University; email: \texttt{jdliang@pku.edu.cn}.} \and
Wenhao Yang\thanks{Management Science and Engineering, Stanford University; email: \texttt{yangwh@stanford.edu}.} \and
Zhihua Zhang\thanks{School of Mathematical Sciences, Peking University; email: \texttt{zhzhang@math.pku.edu.cn}.}
}

\begin{document}
\maketitle

\begin{abstract}
  In this paper, we study distributional reinforcement learning from the perspective of statistical efficiency.
  %a statistically efficient perspective.
  %both the non-asymptotic and asymptotic perspectives. 
  We investigate distributional policy evaluation, aiming to estimate the complete return distribution (denoted $\eta^\pi$) attained by a given policy $\pi$.
  We use the certainty-equivalence method to construct our estimator $\hat\eta^\pi$, given a generative model is available.
  %We show that 
  In this circumstance we need a dataset of size $\wtilde O\prn{\frac{|\gS||\gA|}{\varepsilon^{2p}(1-\gamma)^{2p+2}}}$ to guarantee the $p$-Wasserstein metric between $\hat\eta^\pi$ and $\eta^\pi$  less than $\varepsilon$ with high probability.
  This implies the distributional policy evaluation problem can be solved with sample efficiency. 
  Also, we show that under different mild assumptions a dataset of size $\wtilde O\prn{\frac{|\gS||\gA|}{\varepsilon^{2}(1-\gamma)^{4}}}$ suffices to ensure the Kolmogorov metric and total variation metric between $\hat\eta^\pi$ and $\eta^\pi$ is below $\varepsilon$ with high probability.
  Furthermore, we investigate the asymptotic behavior of $\hat\eta^\pi$.
  We demonstrate that the ``empirical process'' $\sqrt{n}(\hat\eta^\pi-\eta^\pi)$ converges weakly to a Gaussian process in the space of bounded functionals on Lipschitz function class $\ell^\infty(\FW)$, also in the space of bounded functionals on indicator function class  $\ell^\infty(\FKS)$ and bounded measurable function class $\ell^\infty(\FTV)$ when some mild conditions hold.
  %Here $\FW$ represents the Lipschitz function class and $\FKS$ represents the indicator function class.
  Our findings give rise to a unified approach to statistical inference of a wide class of statistical functionals of $\eta^\pi$.
  % We back our theoretical results by numerical simulations.
\end{abstract}

% \tableofcontents
\section{Introduction}
% -*- Mode: latex -*- %
Reinforcement learning has achieved remarkable advancements in various fields, including game-playing \citep{silver2018general,vinyals2019grandmaster}, robotics systems \cite{kober2013reinforcement}, large language models \citep{ouyang2022training, openai2023gpt4}, among others.
In classical reinforcement learning which relies on the reward hypothesis \cite{sutton2004,sutton2018reinforcement}, one evaluates the performance of a learning agent by the expected returns (\ie, the expected cumulative sum of a received reward).
However, in many applications of reinforcement learning, it is not enough to simply consider the expected returns because other factors such as uncertainty or risks might also be crucial.
For example, when we ask a large language model a question we not only expect a useful answer but want to know how reliable the answer is.
An investor should consider the risk-return tradeoff when making investment decisions in financial markets, as high expected returns usually mean greater risks \cite{ghysels2005there}.
In the area of healthcare, we are not only interested in the expected performance of a dynamic treatment regime but care about its long-tail performance. 
Otherwise, it would have the potential to cause serious consequences for patients \cite{lavori2004dynamic}.

Distributional reinforcement learning \cite{morimura2010nonparametric,bellemare2017distributional} goes beyond the notion of expected returns and proposes to learn the complete distribution of the random returns.
Unlike the classical approach, the distributional perspective offers a comprehensive depiction of the inherent uncertainty (known as aleatoric uncertainty) in the performance of learning agents, due to both the stochastic nature of environments and the actions taken by the agents.
By employing the distributional perspective, we might obtain a better understanding of the consequences of the agents' behaviors and have a unified approach for dealing with the issues with regard to risk, uncertainty, and robustness \cite{pmlr-v120-singh20a, Lim2022RiskSensitive}.
In fact, \citet{dabney2020distributional} believed that human brains may also rely on a distributional code for future rewards to make decisions.

In the setting of reinforcement learning, we are usually not fully aware of the stochastic environment and must rely on a dataset to evaluate or train a learning agent.
This induces another kind of uncertainty that we call statistical uncertainty (also known as epistemic uncertainty) \cite{clements2019estimating,lockwood2022review}, which stems from limited data.
In this paper, we seek to simultaneously address the two kinds of uncertainties aforementioned by developing statistical understandings for distributional reinforcement learning.
Specifically, we aim to answer the following two fundamental questions: a) How many samples are required to learn the full distribution of random returns? b) Is it possible to perform statistical inferences from the learned return distributions?
We give affirmative answers to both questions with the benefit of a statistical analysis of distributional reinforcement learning presented in our paper.
We believe that our paper can offer new opportunities in uncertainty quantification in reinforcement learning.

\begin{figure}
    \centering
    \includegraphics[width=0.5\textwidth]{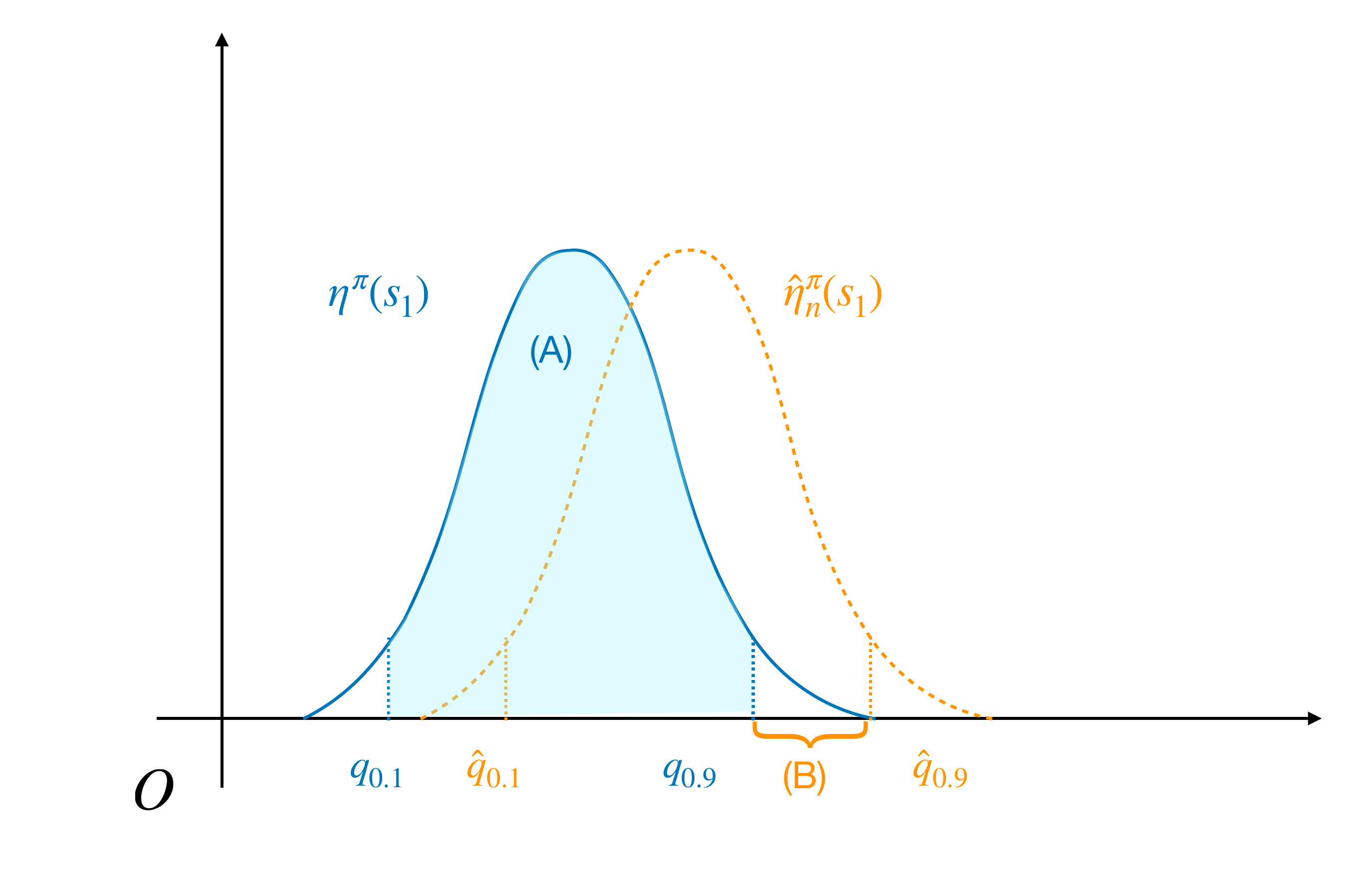}
    \caption{An illustration of two types of uncertainty in RL. Blue distribution: ground-truth return distribution with quantiles $q_{0.1}$ and $q_{0.9}$. Orange distribution: estimated return distribution with quantiles $\hat q_{0.1}$ and $\hat q_{0.9}$. Shaded area (A): intrinsic uncertainty in RL. Error (B): error caused by statistical uncertainty in RL.}
    \label{fig:twoTypesOfUncertainty}
\end{figure}

\subsection{Our Contributions}

In this paper, we focus on the problem of distributional policy evaluation, which lies at the core of distributional reinforcement learning \footnote{Indeed, the control problem in distributional reinforcement learning can be solved by a two-stage procedure.
First, we estimate a near-optimal policy $\hat\pi$ using some policy learning subroutines.
Then it remains to solve a distributional policy evaluation problem, \ie~ the return distribution of $\hat\pi$. 
See Section 7.3 of \cite{bdr2022}. }.
Consider a $\gamma$-discounted infinite-horizon Markov decision process (MDP).
The MDP is assumed to be tabular, \ie,~it has finite state space $\gS$ and action space $\gA$.
Let $\eta^\pi(s)$ denote the random return gained by running the policy $\pi$ from the initial state $s$ and define $\eta^\pi:=(\eta^\pi(s_1), \ldots, \eta^\pi(s_{|\gS|}))$.
When the underlying MDP is known, we may find $\eta^\pi$ by solving the distributional Bellman equation $\eta^\pi=\gT^\pi\eta^\pi$ with the distributional dynamic programming algorithm.
Here, $\gT$ is called the distributional Bellman operator.

Our goal is to estimate $\eta^\pi$ when the underlying MDP is unknown. 
Following common practice in the literature on reinforcement learning, we assume that the distribution of the random reward is fully known and that the transition probability of that MDP is unknown.
\footnote{We also consider the more challenging setting where the reward is unknown, either, in the later part of the paper.}
Our estimator $\hat\eta^\pi$ is constructed via the certainty-equivalence approach \cite{simon1956dynamic, theil1957note}.
% In particular, suppose the underlying MDP is known. Then we may form the distributional Bellman operator $\gT^\pi$, which is an analog of the classical Bellman operator.
% \citet{bellemare2017distributional} showed that $\gT^\pi$ is a $\gamma$-contraction in the supereme of $p$-Wasserstein metric with $\eta^\pi:=(\eta^\pi(s_1), \ldots, \eta^\pi(s_{|\gS|}))$ as the unique fixed point.
% Here $\eta^\pi(s_1)$ means the distribution of returns attained by running policy $\pi$ given the initial state $s_1$.
% This contracting property leads to the distributional dynamic programming algorithm to compute $\eta^\pi$ given that the MDP is already known.
% However, for the problem of distributional policy evaluation, the underlying MDP is unknown.
% Therefore, we choose to solve this problem via the certainty-equivalence approach \cite{simon1956dynamic, theil1957note}.
More specifically, we first build an explicit model of the underlying transition dynamics denoted by $\what P$ using an offline dataset of $n |\gS||\gA|$ entries obtained by a generative model.\footnote{We also consider the more challenging setting where the offline dataset is not perfectly explorative, in the later part of the paper.}
Then we may form an empirical MDP $\what M$ whose transition dynamic is $\what P$.
An estimator of $\eta^\pi$ is then formulated as the return distribution of $\what M$, which we denote as $\hat\eta^\pi_n:=(\hat\eta^\pi_n(s_1), \ldots,\hat\eta^\pi_n(s_{|\gS|}))$.
Note that we consider a fully non-parametric setting where $\hat\eta_n$ is not restricted in some parametric model and can be any probability distribution.

In this paper, we analyze the statistical performance of the estimated return distribution $\hat\eta^\pi_n$.
Concretely, we would like to: a) prove non-asymptotic bounds for the $l_\infty$-type estimation error $\sup_{s}d(\eta^\pi(s),\hat\eta^\pi(s))$, where $d$ is some probability metrics; b) study the asymptotics of $\hat\eta^\pi_n(s)$, particularly identify the limit of the process $\sqrt{n}(\eta^\pi(s)-\hat\eta^\pi_n(s))$.

Our analysis has close relationships with the literature on the perturbation theory of Markov chains.
Given the MDP $M$ and policy $\pi$, we may define a Markov chain on the state space of the MDP, and we use $P^\pi$ to denote its transition dynamic.
Our problem can be restated as how the gap between $\eta^\pi$ and $\hat\eta_n^\pi$ is related to the gap between the transition dynamics $P^\pi$ and $\what P_n^\pi$.
Or, more broadly, how the ``characteristics" of a Markov chain would change when we perturb its transition dynamics.
This is the core question in the field of perturbation theory of Markov chains.
While existing works mostly focus on ``characteristics" like invariant distributions and $t$-step distributions, the unique feature of this work is that the object of interest is the distribution of the cumulative rewards.
We would elaborate on the differences between this work and prior works on perturbation analysis of Markov chain in the section of related works.

To the best of our knowledge, we are among the first to develop statistical theories for distributional reinforcement learning.
Our main contributions are outlined below.
\begin{enumerate}
    \item We show that under different mild conditions, the distributional dynamic programming algorithm converges to the fixed point when measured by the Kolmogorov-Smirnov metric (KS metric) and the total variation metric (TV metric).
    Interestingly, this convergence occurs despite the distributional Bellman operator no longer being (strictly) contractive. 
    Our findings correct the misconception that distributional dynamic programming is not guaranteed to converge in the KS and TV metrics.
    \item We provide non-asymptotic bounds for the $p$-Wasserstein metric, the KS metric and the TV metric between $\eta^\pi$ and $\hat\eta^\pi_n$.
    Specifically, we prove $\sup_{s\in\gS}W_p(\eta^\pi(s),\hat\eta^\pi_n(s))=\wtilde O\prn{\brk{\frac{1}{n(1-\gamma)^{2p+2}}}^{1/2p}}$\footnote{$\wtilde{O}$ hides all terms of logarithmic order.}, $\sup_{s\in\gS} \KS (\eta^\pi(s),\hat\eta^\pi_n(s))=\wtilde O\prn{\sqrt{\frac{1}{n(1-\gamma)^{4}}}}$ and $\sup_{s\in\gS} \TV (\eta^\pi(s),\hat\eta^\pi_n(s))=\wtilde O\prn{\sqrt{\frac{1}{n(1-\gamma)^{4}}}}$ with high probability.
    % $W_p$, ${\KS}$ and $\TV$ represent the $p$-Wasserstein metric ,the KS metric, and the TV metric respectively.
    % And 
    Our non-asymptotic results translate to an $\wtilde O\prn{\frac{1}{\varepsilon^{2p}(1-\gamma)^{2p+2}}}$ complexity bound for the case of the $p$-Wasserstein metric and  $\wtilde O\prn{\frac{1}{\varepsilon^{2}(1-\gamma)^{4}}}$ complexity bound for the cases of the KS and TV metric.
    We generalize our results to more challenging settings with less exploratory datasets and unknown reward distributions.
    \item We give a characterization of the asymptotic behavior of $\hat\eta_n^\pi$.
    We show that for any $s\in\gS$, $\sqrt{n}(\hat\eta_n^\pi(s)-\eta^\pi(s))$ converges weakly to a Gaussian process in $\ell^\infty(\FW)$.
   Under different mild conditions, $\sqrt{n}(\hat\eta_n^\pi(s)-\eta^\pi(s))$ also converges weakly to a Gaussian process in both $\ell^\infty(\FKS)$ and $\ell^\infty(\FTV)$ for each $s\in\gS$.
    Here, $\FW$, $\FKS$, and $\FTV$ represent the $1$-Lipschitz function class, the indicator function class, and the bounded measurable function class, respectively.
    We generalize the asymptotic results to the settings with less exploratory datasets and unknown reward distributions.
    These asymptotic results enable us to perform statistical inference for 
    % a general class of statistical functionals of
    $\eta^\pi$.
    Concretely, we construct asymptotically valid confidence sets for $\eta^\pi$ in the forms of $W_1$, KS, and TV balls, and asymptotically valid confidence intervals for $\phi(\eta^\pi(s))$, where $\phi$ can be any Hadamard differentiable functional. 
    % Concretely, we construct asymptotically valid confidence intervals for $W_p(\hat\eta_n^\pi(s),\eta^\pi(s))$, $\KS(\hat\eta_n^\pi(s),\eta^\pi(s))$, and $\phi(\eta^\pi(s))$ where $\phi$ can be any Hadamard differentiable functional. 

% Besides, we also perform a series of numerical experiments to validate our theoretical FKSings as well as statistical inference procedures.

\item At the technical level, our main challenge is that we must work in the infinite-dimensional space of probability measures.
Therefore, most of the techniques developed for classical reinforcement learning theory are not valid anymore.
We address the challenge through an analysis of the concentration behaviors as well as asymptotics of $(\gI-\what\gT_n^\pi)^{-1}(\what\gT_n-\gT^\pi)\eta^\pi$.
Here $\what\gT^\pi$ is the distributional Bellman operator associated with the empirical MDP $\what M$ and $(\gI-\gT^{\pi})^{-1}\colon=\sum_{i=0}^\infty \prn{\gT^{\pi}}^i$ is defined on a subspace of interest.
We achieve this by carefully examining the properties of the distributional Bellman operator $\gT^\pi$ on the vector space of signed measures equipped with different metrics to decouple the dependencies between operators $(\gI-\what\gT_n^\pi)^{-1}$ and $(\what\gT_n^\pi-\gT^\pi)$.
% Specifically, we give a definition of $(\gI-\gT^{\pi})^{-1}\colon=\sum_{i=0}^\infty \prn{\gT^{\pi}}^i$ on a subspace of interest.
% the product of vector spaces consisting of signed measures $\mu$ supported on $\brk{0,\frac{1}{1-\gamma}}$ with $\mu\prn{\brk{0,\frac{1}{1-\gamma}}}=0$.
% This helps to relate $(\what\gT_n-\gT^\pi)\eta^\pi$ with $(\hat\eta_n-\eta)$, the error term of interest.

\end{enumerate}

\subsection{Related Works}

\paragraph{Distributional Reinforcement Learning}
Distributional reinforcement learning has achieved remarkable success in fields such as communications \cite{hua2019gan}, transportation systems \cite{naeem2020generative}, and algorithm discovery \cite{fawzi2022discovering}.
Notable distributional reinforcement learning algorithms include categorical temporal-difference learning \cite{bellemare2017distributional}, quantile temporal-difference learning \cite{dabney2018distributional, dabney2018implicit}, GAN-based methods \cite{freirich2019distributional,doan2018gan}, actor-critic methods \cite{ma2020dsac}, etc.
For a comprehensive treatment of distributional reinforcement learning, readers could  refer to a very recent book by \citet{bdr2022}.

Despite its empirical success, there is a relative lack of theoretical understanding of distributional reinforcement learning.
\citet{rowland2018analysis} analyzed the convergence properties of categorical temporal-difference learning. 
However, their convergence analysis is asymptotic and does not consider sample complexities as well as asymptotic distributions.
Recently, \citet{rowland2023analysis} presented similar consistency results for quantile temporal-difference learning.
\citet{wu2023distributional} showed that the distribution of returns can be learned using an algorithm called Fitted Likelihood Estimation (FLE) given an offline dataset.
They also proposed non-asymptotic bounds for the statistical distance between the learned distribution and the ground truth.
The biggest difference between \cite{wu2023distributional} and this work is that \cite{wu2023distributional} focused on parametrized return distributions, while we study the non-parametric case.
Both the FLE algorithm and the associated non-asymptotic statistical bounds would be invalid under the non-parametric scenario.
% However, their analysis is modular, assuming that an MLE procedure can achieve good generalization bounds and focusing less on the statistical aspects of distributional policy evaluations.
Another line of work treats learning the distribution of returns as an auxiliary task and aims to understand how this auxiliary task can improve policy learning within the framework of classical reinforcement learning.
\citet{sun2022interpreting} found that such auxiliary tasks can be viewed as a form of regularization and can make the optimization process more stable.
\citet{wang2023benefits} explored the statistical benefits of distributional reinforcement learning.
They showed that distributional reinforcement learning can yield better non-asymptotic bounds than classical reinforcement learning in the ``small loss" scenarios.

\paragraph{Statistical Inference in Reinforcement Learning}
% Uncertainty quantification is of vital importance in decision-making problems, especially in high-stake applications like healthcare \cite{begoli2019need}.
Statistical inference in the context of reinforcement learning has drawn growing interest in the community.
There are a number of works studying the statistical inference problems for expected returns (or value functions).
\citet{thomas2015high} and \citet{jiang2016doubly} proposed high-confidence bounds for value functions in the setting of off-policy evaluation.
\citet{hao2021bootstrapping} devised a bootstrapping procedure to perform statistical inference in off-policy evaluation.
\citet{shi2022statistical} modeled the value function with the series/sieve methods and devised confidence intervals for value functions in both the settings of policy evaluation and policy learning.
\citet{zhu2023uncertainty} also constructed asymptotically tight confidence intervals for learned (optimal) value functions.
\citet{li2023statistical} and \citet{li2023online} considered online statistical inference for value functions in an online reinforcement learning setting.

At the same time, fewer works focus on statistical inferences for other statistical functionals of the return distribution.
\citet{yang2022toward} investigated the asymptotic behaviors of distributionally robust value functions and constructed asymptotically tight confidence bounds.
\citet{chandak2021universal} and \citet{huang2022off} proposed methods to estimate the cumulative distribution function (CDF) and confidence band for the ground truth CDF.
And statistical inference for statistical functionals can be achieved by the plug-in approach.
However, their estimator is based on importance sampling, causing the errors can grow exponentially w.r.t.\ the horizon length.
Also, their confidence intervals are based on non-asymptotic bounds and may thus be too conservative.

\paragraph{Perturbation Theory of Markov Chains}
Early works on the perturbation theory of Markov chains can date back to \cite{schweitzer1968perturbation, kartashov1986inequalities}.
\citet{mitrophanov2005sensitivity} showed perturbation bounds of the $t$-step distributions for uniformly ergodic Markov chains.
\citet{ferre2013regular} analyzed behaviors of perturbed $V$-geometrically ergodic Markov chains.
\citet{rudolf2018perturbation} proved bounds of Wasserstein distance between the $t$-step distributions of a Wasserstein ergodic Markov chain and its perturbed counterpart.
The perturbation theory of Markov chains is widely used in Markov chain Monte-Carlo algorithms and Bayesian statistics \cite{bardenet2014towards, alquier2016noisy, johndrow2017error}.
One can refer to the very recent book chapter \cite{rudolf2024perturbations} for a more thorough review. 

Compared with previous works, the unique feature of our paper is that we study perturbation bounds for a novel object, that is, the return distribution of an MDP (or MRP).
Our results are not simple corollaries of prior perturbation bounds on $t$-step distributions because the rewards obtained at different timesteps are dependent.
See \cite{wiltzer2024distributional} for a more detailed discussion.
To get the new results we developed new analysis techniques.
Our proof techniques are specifically developed based on a thorough understanding of the theoretical properties of the distributional Bellman operator, making them particularly suited for the analysis of return distributions.

The remainder of this paper is organized as follows. 
In Section~\ref{Section_prelim}, we introduce some basic concepts of distributional reinforcement learning.
In Section~\ref{Section_analysis}, we present our statistical analysis of distributional reinforcement learning.
In Section~\ref{Section_inference}, we propose a series of inferential procedures for the return distribution.
Section~\ref{Section_numerical} verifies our theoretical findings and tests our inferential procedures through numerical simulations.

\section{Preliminaries}\label{Section_prelim}
% \paragraph{Markov Decision Processes}
\subsection{Problem Setting and the Certainty-equivalence Estimator}

An Markov decision process (MDP) is represented by a 5-tuple $M=\<\gS,\gA,\gP_R,P,\gamma\>$, where $\gS$ represents a finite state space, $\gA$ a finite action space, ${\gP_R\colon\gS\times\gA\to\Delta\prn{[0,1]}}$  the distribution of rewards, ${P\colon\gS\times\gA\to\Delta\prn{\gS}}$  the transition dynamics, and $\gamma\in(0,1)$ a discounted factor.
Here we use $\Delta(\cdot)$ to represent the set of probability distributions over some set.
Given a policy $\pi\colon\gS\to\Delta\prn{\gA}$ and an initial state $S_0= s\in\gS$, a random trajectory $\brc{\prn{S_t,A_t,R_t}_{t=0}^\infty}$ can be sampled from the MDP using the following procedure: 
\begin{equation*}
    \begin{aligned}
 %       S_0 &=s, \\
        A_t\mid S_t&\sim\pi(\cdot\mid S_t),\\
        R_t\mid (S_t,A_t)&\sim \gP_R({\cdot}\mid S_t,A_t),\\
        {S_{t+1}}\mid{(S_t,A_t)}&\sim P({\cdot}\mid{S_t,A_t}).\\
    \end{aligned}
\end{equation*}
We define the return of such trajectories by
$$
G^\pi(s):=\sum_{t=0}^\infty \gamma^t R_t.
$$

$G^\pi(s)$ is a random variable (see Proposition F.1 in Appendix F.
And we always have $G^\pi(s)\in\brk{0,\frac{1}{1-\gamma}}$.
The expected return $\EB G^\pi(s)$ is called the value function and is denoted by $V^\pi(s)$. 
We also define $\eta^\pi(s)$ as the distribution of $G^\pi(s)$.

% \begin{example}[A Simplest Example of Return Distribution]
%     Consider a very simple MDP with a single state $s_0$ and a single action $a_0$.
%     We set $\gamma=1/2$ and $\gP_R(\cdot\mid s_0,a_0)=\mathsf{Bernoulli}(1/2)$.
%     Let $\pi_0$ denote the trivial policy $\pi_0(a_0\mid s_0)=1$. Then
%     \begin{equation*}
%         G^{\pi_0}(s_0)\stackrel{d}{=}\sum_{i=0}^\infty \gamma^i X_i,\quad X_i\simiid \mathsf{Bernoulli}(1/2).
%     \end{equation*}
%     In other words, we have $\eta^{\pi_0}(s_0)=\mathsf{Uniform}[0,2]$.
% \end{example}

% \paragraph{Problem Setting and the Certainty-equivalence Estimator}
In this paper, we focus on the problem of learning $\eta^\pi$ for some policy $\pi$ when the underlying MDP is unknown and must be estimated from a pre-defined dataset.
This problem is called the distributional policy evaluation problem.
We assume the dataset is generated by a generative model, which is able to return a value of the next state $s^\prime$ following $P(\cdot\mid s,a)$ for any given pair $(s,a)\in\gS\times\gA$.
For each pair $(s,a)\in\gS\times\gA$, the generative model is called $n$ times and produces an array $X_1^{(s,a)}, \ldots, X_n^{(s,a)}\simiid P(\cdot\mid s,a)$.

Given the dataset, we may obtain the estimate of the transition probability as
\begin{equation*}
    \what P_n(s^\prime\mid s,a)=\frac{1}{n}\sum_{i=1}^n \ind\brc{X_i^{(s,a)}=s^\prime}.
\end{equation*}
Thus, $\what{P}_n$ defines an empirical MDP $\what{M}=\<\gS,\gA,\gP_R,\what{P}_n,\gamma\>$ with the corresponding distribution of returns $\hat\eta_n^\pi$.
We call $\hat\eta_n^\pi$ the certainty-equivalence estimator, and the goal of our paper is to explore the statistical properties of $\hat\eta_n^\pi$.
\begin{remark}\label{Remark_alternative_formulation}
    An alternative problem formulation is to first transform the MDP $\<\gS,\gA,\gP_R,P,\gamma\>$ to an Markov reward process (MRP) $\<\gS,\gP^\pi_R,P^\pi,\gamma\>$ induced by the original MDP and the policy $\pi$ to be evaluated. 
    Here $\gP^\pi_R(\cdot\mid s)=\sum_{a\in\gA}\pi(a\mid s)\gP_R(\cdot\mid s,a)$ and $P^\pi(\cdot\mid s)=\sum_{a\in\gA}\pi(a\mid s)P(\cdot\mid s,a)$.
    Then we may use samples from the MRP to form the certainty-equivalence estimator.
    With such problem formulation, we no longer need to care about the action space $\gA$ and the policy $\pi$ explicitly.
    However, the main drawback of this formulation is that the data collection process depends on the policy $\pi$.
    % Specifically, if we need to evaluate multiple different policies, we must form multiple different MRPs and collect multiple different datasets.
    On the other hand, in our problem formulation we can evaluate arbitrary policies with the same predefined dataset, which fits better with real-world applications.
    This formulation also allows our theoretical results to draw further implications on various fields of RL (see Corollary~\ref{Corollary_risk_sensitive_RL} and Example~\ref{Example_inference_advantage}).
\end{remark}

\subsection{Metrics on the Space of Measures}
Suppose $\mu$ and $\nu$ are two probability distributions on $\RB$ with finite $p$-moments ($p\geq 1$). The $p$-Wasserstein metric between $\mu$ and $\nu$ is defined as 
\begin{equation*}
    W_p(\mu, \nu)=\left(\inf _{\kappa \in \Gamma(\mu, \nu)} \int_{\RB^2}\abs{x-y}^p \kappa(dx, dy)\right)^{1 / p}.
\end{equation*}
Elements $\kappa \in \Gamma(\mu, \nu)$ are called couplings of $\mu$ and $\nu$, \ie, joint distributions on $\RB^2$ with prescribed marginals $\mu$ and $\nu$ on each “axis”.
Suppose $\mu$ and $\nu$ have cumulative distribution function $F_\mu$ and $F_\nu$, respectively. In the case of $p=1$ we have
\begin{equation*}
    W_1(\mu, \nu)=\int_\RB |F_\mu(x)-F_\nu(x)| dx.
    % =\norm{F_\mu-F_\nu}_{L_1}.
\end{equation*}
% The $p$-Wasserstein metric also admits the following dual formulations:
% \begin{equation*}\label{problem:Wasserstein_dual}
%     W_p(\mu, \nu)=\prn{\sup_{f(x)+g(y)\leq \abs{x-y}^p}\int_{\RB}f(x)\mu(dx)+\int_{\RB}g(y)\nu(dy)}^{1/p}.
% \end{equation*}
% In the following parts of the paper, we are mainly concerned with the $1$-Wasserstein metric.

The Kolmogorov–Smirnov metric (KS metric) is defined as
\begin{equation*}
    {\KS}(\mu,\nu)=\sup_{t\in\RB}\abs{\mu((-\infty,t])-\nu((-\infty,t])}.
\end{equation*}
We may bound  $\KS(\mu,\nu)$ with  $W_1(\mu,\nu)$ when either of $\mu, \nu$ has bounded densities.
\begin{proposition}\cite[Proposition 1.2]{ross2011fundamentals}\label{Proposition_bound_KS_with_W1} 
Assume that $\mu\in\Delta(\RB)$ has finite moment and $\mu$ has a Lebesgue density $p_\mu$ that is bounded by $C$. Then for any $\nu\in\Delta(\RB)$ with finite moment, $\KS(\mu,\nu)\leq\sqrt{2CW_1(\mu,\nu)}$.
\end{proposition}

The total variation distance (TV distance) is defined as 
\begin{equation*}
    \TV(\mu,\nu)=\sup_{A\in\gB(\RB)}\abs{\mu(A)-\nu(A)},
\end{equation*}
where $\gB(\RB)$ denotes all Borel sets in $\RB$.
$\TV(\mu,\nu)$ can also bounded by $W_1(\mu,\nu)$ when $\mu$ and $\nu$ have smooth densities.
\begin{proposition}\cite[Theorem 2.1]{chae2020wasserstein}\label{Proposition_bound_TV_with_W1} 
Assume that $\mu,\nu\in\Delta(\RB)$ have Lebesgue densities $p_\mu, p_\nu\in H^1_1(\RB)$. 
Specifically, $H^1_1(\RB)$ represents the $L^1$ Sobolev space of order $1$ defined as 
\begin{equation*}
    H^1_1(\RB)\colon=\brc{f\in L^1(\RB)\colon D^1f\in L^1(\RB);~\norm{f}_{H^1_1}=\norm{f}_{1}+\norm{D^1 f}_{1}<\infty}.
\end{equation*}
Here $L^1(\RB)$ is the space of Lebesgue integrable functions
% measurable functions whose absolute value is Lebesgue integrable 
and $\norm{\cdot}_1$ is the associated $L^1$ norm, $D^1f$ represents the weak derivative of $f$.
Then we have
\begin{equation*}
    \TV(\mu,\nu)\leq \sqrt{K\prn{\norm{p_\mu}_{H^1_1}+\norm{p_\nu}_{H^1_1}}W_1(\mu,\nu)}.
\end{equation*}
Here $K$ is a universal constant.
\end{proposition}

The $1$-Wasserstein metric, the KS metric, and the TV distance are all special cases of integral probability metrics.
Specifically, we define
\begin{equation*}
    \norm{\mu-\nu}_{\gH}=\sup_{h\in \gH} \abs{\mu h-\nu h},
\end{equation*}
where $\gH$ denotes some function class and $\mu h$ represents $\EB_{X\sim\mu}\brk{h(X)}$.
If we choose 
\begin{itemize}
    \item $\FW:=\brc{f\mid f \text{ is $1$-Lipschitz}}$, then $\norm{\mu-\nu}_{\FW}=W_1(\mu,\nu)$.
    \item $\FKS:=\brc{\ind\brc{\cdot\leq z}\mid z\in\RB}$, then $\norm{\mu-\nu}_{\FKS}=\KS(\mu,\nu)$.
    \item $\FTV:=\brc{f\mid f \text{ is measuable and }\norm{f}_\infty\leq 1}$, then $\norm{\mu-\nu}_{\FTV}=\TV(\mu,\nu)$.
\end{itemize}

\subsection{Distributional Bellman Operator and Distributional Dynamic Programming}

It is well-known that the expected returns (also called the value functions) satisfy the Bellman equation.
In particular, letting $V^\pi$ denote $\prn{V^\pi(s_1), \ldots,V^\pi(s_{\abs{\gS}})}$, we have for any $s\in\gS$,
\begin{equation}\label{Equation_Bellman_equation}
    \begin{aligned}
            V^\pi(s)&=\brk{T^\pi(V^\pi)}(s)\\
    &:=\EB_{A\sim\pi(\cdot\mid s), R\sim\gP(\cdot\mid s,A)}R+\EB_{A\sim\pi(\cdot\mid s),S^\prime\sim P(\cdot\mid s,A)} V^\pi(S^\prime)\\
    &=\sum_{a\in\gA}\pi(a\mid s)\int_0^1 r \gP_R(dr\mid s,a)+\sum_{a\in\gA,s^\prime\in\gS} \pi(a\mid s)P(s^\prime\mid s,a)V^\pi(s^\prime).
    \end{aligned}
\end{equation}
We call the operator $T^\pi\colon \RB^\gS\to \RB^\gS$ the Bellman operator.
And the Bellman equation suggests that the value function $V^\pi$ is a fixed point of $T^\pi$.

The distributional Bellman equation describes a similar relationship to Equation~\eqref{Equation_Bellman_equation} for the distributions of returns.
Letting $\eta^\pi$ denote $\prn{\eta^\pi(s_1),\dots,\eta^\pi(s_{|\gS|})}$, we have for any $s\in\gS$
\begin{equation}\label{Equation_distributional_Bellman_equation}
\begin{aligned}
        \eta^\pi(s)&=\brk{\gT^\pi(\eta^\pi)}(s)\\
    &:=\EB_{A\sim\pi(\cdot\mid s), R\sim\gP_R(\cdot \mid s,A),S^\prime\sim P(\cdot\mid s,A)}\prn{b_{R,\gamma}}_\#\eta^\pi(S^\prime)\\
    &=\sum_{a\in\gA,s^\prime\in\gS}\pi(a\mid s)P(s^\prime\mid s,a)\int_0^1 \prn{b_{r,\gamma}}_\#\eta^\pi(s^\prime)\gP_R(dr\mid s,a).
\end{aligned}
\end{equation}
Here $b_{r,\gamma}\colon \RB\to\RB$ is an affine function defined by $b_{r,\gamma}(x)=r+\gamma x$, and $g_\#\mu$ is the push forward measure of $\mu$ through function $g$ so that $g_\#\mu(A)=\mu(g^{-1}(A))$.
The integral ${\int_0^1 \prn{b_{r,\gamma}}_\#\eta^\pi(s^\prime)\gP_R(dr\mid s,a)}$ is defined by
\begin{equation*}
    \brk{\int_0^1 \prn{b_{r,\gamma}}_\#\eta^\pi(s^\prime)\gP_R(dr\mid s,a)}(B)=\int_0^1 \brk{\prn{b_{r,\gamma}}_\#\eta^\pi(s^\prime)}(B)\gP_R(dr\mid s,a)
\end{equation*}
for any Borel set $B$ in $\brk{0,\frac{1}{1-\gamma}}$.
We call the operator $\gT^\pi\colon \Delta\prn{\brk{0,\frac{1}{1-\gamma}}}^\gS\to \Delta\prn{\brk{0,\frac{1}{1-\gamma}}}^\gS$ the distributional Bellman operator, and the distribution of returns $\eta^\pi$ is a fixed point of $\gT^\pi$.

Suppose the MDP $M$ is already known. 
For a policy $\pi$, we may compute the value function $V^\pi$ by the dynamic programming algorithm.
Specifically, assuming $V_{k+1}=T^\pi(V_k)$, we have $T^\pi$ is a $\gamma$-contraction w.r.t. norm $\norm{\cdot}_\infty$ on $\RB^\gS$, and thus $\lim_{k\to\infty}\norm{V_k-V^\pi}_\infty=0$.
In analogy to dynamic programming, we may also define distributional dynamic programming, \ie,~$\eta^{(k+1)}=\gT^\pi\eta^{(k)}$.
It can be shown that $\gT^\pi$ is a $\gamma$-contraction in the supreme $p$-Wasserstein metrics. 
Thus distributional dynamic programming exhibits geometric convergence in the supreme $p$-Wasserstein metric.
\begin{proposition} \emph{\cite[Propositions 4.15 and 4.16]{bdr2022}} \label{Proposition_value_iteration_Wp}
    The distributional Bellman operator is a $\gamma$-contraction on $\Delta(\RB)^\gS$ in the supreme $p$-Wasserstein metric.
    More precisely, for $\eta,\eta^\prime\in\Delta(\RB)^\gS$, we have
    \begin{equation*}
        {\sup_{s\in\gS}W_p\prn{\brk{\gT^\pi\eta}(s),\brk{\gT^\pi\eta^\prime}(s)}\leq \gamma\sup_{s\in\gS}W_p(\eta(s),\eta^\prime(s))}.
    \end{equation*}
   Furthermore, we have
   \begin{equation*}
       \sup_{s\in\gS}W_p(\eta^{(k)}(s),\eta^\pi(s))\leq \gamma^k \sup_{s\in\gS} W_p(\eta^{(0)}(s),\eta^\pi(s))
   \end{equation*}
   and
    \begin{equation*}
        \lim_{k\to\infty}\sup_{s\in\gS} W_p(\eta^{(k)}(s),\eta^\pi(s))=0.
    \end{equation*}
\end{proposition}

When measured by other commonly used probability metrics like the supreme KS metric or the supreme TV distance, the distributional Bellman operator may not be a (strict) contraction and distributional dynamic programming may not converge at all \citep{bdr2022}.
This is because, unlike the cases of dynamic programming, now we operate in an infinite-dimensional space and the metrics may not be equivalent.
However, we find that under mild conditions distributional dynamic programming does converge in the supreme KS metric and the supreme TV distance, and the convergences are also geometrically fast.
To the best of our knowledge, we are the first to examine the convergence property of distributional dynamic programming w.r.t. the supreme KS metric and the supreme TV distance.

\begin{assumption}\label{Assumption_reward_bounded_density}
Assume that for any $s\in\gS$, $a\in\gA$, $\gP_R(dr\mid s,a)$ has a Lebesgue density $p_{s,a}^R$ upper-bounded  by a constant $C$.
\end{assumption}

\begin{assumption}\label{Assumption_reward_smooth}
Assume that for any $s\in\gS$, $a\in\gA$, $\gP_R(dr\mid s,a)$ has a Lebesgue density $p_{s,a}^R\in H_1^1(\RB)$ and $\norm{p_{s,a}^R}_{H_1^1(\RB)}\leq M$.
\end{assumption}

Assumption~\ref{Assumption_reward_smooth} is strictly stronger than Assumption~\ref{Assumption_reward_bounded_density} as a consequence of Sobolev's inequality.

\begin{proposition}\label{Proposition_value_iteration_KS}
The distributional Bellman operator is non-expansive on $\Delta(\RB)^\gS$ in the supreme $\KS$ metrics.
If Assumption~\ref{Assumption_reward_bounded_density} holds, then we have
\begin{equation*}
    \sup_{s\in\gS} \KS(\eta^{(k)}(s),\eta^\pi(s))\leq (\sqrt{\gamma})^k\sup_{s\in\gS}\sqrt{CW_1(\eta^{(0)}(s),\eta^\pi(s))}.
\end{equation*}
% and
%    \begin{equation*}
%        \lim_{k\to\infty} \sup_{s\in\gS} \; \KS(\eta^{(k)}(s),\eta^\pi(s))=0.
%    \end{equation*}
\end{proposition}

\begin{proposition}\label{Proposition_value_iteration_TV}
The distributional Bellman operator is non-expansive on $\Delta(\RB)^\gS$ in the supreme TV metrics.
If Assumption~\ref{Assumption_reward_smooth} holds, then we have
% \begin{equation*}
%     \sup_{s\in\gS} \TV(\eta^{(k)}(s),\eta^\pi(s))\leq \prn{\sqrt{\gamma}}^k\sup_{s\in\gS}\sqrt{\frac{2MK}{\gamma}W_1(\eta^{0}(s),\eta^\pi(s))}.
% \end{equation*}
\begin{equation*}
    \sup_{s\in\gS} \TV(\eta^{(k)}(s),\eta^\pi(s))\leq \prn{\sqrt{\gamma}}^k\sup_{s\in\gS}\sqrt{2MKW_1(\eta^{0}(s),\eta^\pi(s))}.
\end{equation*}
% and
%    \begin{equation*}
%        \lim_{k\to\infty} \sup_{s\in\gS} \; \TV(\eta^{(k)}(s),\eta^\pi(s))=0.
%    \end{equation*}
\end{proposition}
To prove Proposition~\ref{Proposition_value_iteration_KS}, we first show that when Assumption~\ref{Assumption_reward_bounded_density} is true the distribution of return $\eta^\pi(s)$ must have a bounded density.
Then by Proposition~\ref{Proposition_bound_KS_with_W1} the KS metric can be controlled by the $1$-Wasserstein metric. 
The proof strategy is similar to that of Proposition~\ref{Proposition_value_iteration_KS}.
We first show that when Assumption~\ref{Assumption_reward_smooth} holds, the $\eta^{\pi}(s)$ and $\eta^{(k)}(s)$ both have densities in $H_1^1(\RB)$.
% Note that here we only require $p^R_{s,a}\in H^1_1((0,1))$ for any $s\in\gS$, $a\in \gA$, which is weaker than the condition $p^R_{s,a}\in H^1_1(\RB)$.
% The latter excludes many commonly used probability distributions such as the uniform distribution or truncated gaussians.
Then we can bound the TV distance with the $1$-Wasserstein metric through Proposition~\ref{Proposition_bound_TV_with_W1}.
The full proof can be found in Appendix F. 

For the certainty-equivalence estimator $\hat\eta_n^\pi$, we also have the following distributional Bellman equation
\begin{equation}\label{Equation_empirical_distributional_Bellman_equation}
\begin{aligned}
        \hat\eta_n^\pi(s)&=\brk{\what\gT_n^\pi(\hat\eta_n^\pi)}(s)\\
    &:=\EB_{A\sim\pi(\cdot\mid s), R\sim\gP(\cdot \mid s,A),S^\prime\sim \what{P}(\cdot\mid s,A)}\prn{b_{R,\gamma}}_\#\hat\eta_n^\pi(S^\prime)\\
    &=\sum_{a\in\gA,s^\prime\in\gS}\pi(a\mid s)\what{P}(s^\prime\mid s,a)\int_0^1 \prn{b_{r,\gamma}}_\#\hat\eta_n^\pi(s^\prime)\gP_R(dr\mid s,a),
\end{aligned}
\end{equation}
where $\what{\gT}_n^\pi$ is called the empirical distributional Bellman operator. 
$\hat\eta_n^\pi$ can be computed via the empirical version of distributional dynamic programming, \ie,~$\hat\eta^{(k+1)}=\what \gT_n^\pi(\hat\eta^{(k)})$.

\section{Statistical Analysis}\label{Section_analysis}
In this section, we analyze distributional reinforcement learning from both the non-asymptotic and asymptotic viewpoints. 
We give the non-asymptotic convergence rates of $\sup_{s\in\gS}W_1\prn{\hat\eta_n^\pi(s),\eta^\pi(s)}$, $\sup_{s\in\gS}\KS\prn{\hat\eta_n^\pi(s),\eta^\pi(s)}$, and $\sup_{s\in\gS}\TV\prn{\hat\eta_n^\pi(s),\eta^\pi(s)}$, which suggest distributional policy evaluation is sample-efficient when a generative model is available.
We also study the asymptotics of $\sqrt n\prn{\hat\eta_n^\pi(s)-\eta^\pi(s)}$ for any $s\in\gS$.
Under mild conditions, we demonstrate that $\sqrt n\prn{\hat\eta_n^\pi(s)-\eta^\pi(s)}$ converges weakly to a Gaussian random element in the spaces $\ell^\infty(\FW)$, $\ell^\infty(\FKS)$ and $\ell^\infty(\FTV)$.
\subsection{Results on Non-asymptotic Analysis}\label{Section_non_asymp}

Our main results of non-asymptotic analysis is given in the following theorems.
\begin{theorem}\label{Theorem_bound_of_Wdist}
For any fixed policy $\pi$, we have  that
\begin{equation*} \EB\sup_{s\in\gS}W_1(\hat\eta_n^\pi(s),\eta^\pi(s))\leq \sqrt{\frac{9\log|\gS|}{n(1-\gamma)^{4}}},
\end{equation*}
and for any $\delta\in(0,1)$, with probability at least $1-\delta$,
\begin{equation*}
\sup_{s\in\gS}W_1(\hat\eta_n^\pi(s),\eta^\pi(s))\leq \frac{\sqrt{9\log|\gS|}+\sqrt{\log(1/\delta)/2}}{\sqrt{n(1-\gamma)^{4}}}.
\end{equation*}

\end{theorem}

To sum up, we show that
% $n=\wtilde{O}\prn{\frac{1}{\varepsilon^{2p}(1-\gamma)^{4p}}}$ suffices to ensure both $\EB\sup_{s\in\gS}W_p(\hat\eta_n^\pi(s),\eta^\pi(s))\leq\varepsilon$ and $\sup_{s\in\gS}W_p(\hat\eta_n^\pi(s),\eta^\pi(s))\leq \varepsilon$ 
$n=\wtilde{O}\prn{\frac{1}{\varepsilon^{2}(1-\gamma)^{4}}}$ suffices to ensure both $\EB\sup_{s\in\gS}W_1(\hat\eta_n^\pi(s),\eta^\pi(s))\leq\varepsilon$ and $\sup_{s\in\gS}W_1(\hat\eta_n^\pi(s),\eta^\pi(s))\leq \varepsilon$ 
with high probability, which implies model-based distributional policy evaluation is sample-efficient.
The key idea of our proof is that we first analyze the concentration behaviors of the infinite-dimensional operator $(\what \gT_n^\pi-\gT^\pi)$.
Then we examine the properties of $\gT^\pi$ in the vector space of signed measures equipped with the $W_1$-metric and give a reasonable definition of $\prn{\gI-\what\gT_n^\pi}^{-1}:=\sum_{i=0}^\infty \prn{\what\gT_n^\pi}^i$ on a product of vector space consisting of signed measures $\mu$ such that $\mu\prn{\brk{0,\frac{1}{1-\gamma}}}=0$.
 This allows us to write $\hat\eta_n-\eta^\pi=\prn{\gI-\what\gT_n^\pi}^{-1}(\what \gT_n^\pi-\gT^\pi)\eta^\pi$. 
We can draw the conclusion noting that the operator norm of $\prn{\gI-\what\gT_n^\pi}^{-1}$ w.r.t. the $W_1$-metric is always bounded by $\frac{1}{1-\gamma}$.
Detailed proof can be found in Appendix B.

Compared with the minimax optimal $\wtilde{O}\prn{\frac{1}{\varepsilon^2(1-\gamma)^3}}$ sample complexity bound for the model-based policy evaluation \citep{li2020breaking}, our sample complexity bound has an additional ${\frac{1}{1-\gamma}}$ factor.
In fact, learning the distribution of returns is harder than the classic policy evaluation problem.
Because we always have $\abs{V(s)-\what V(s)}\leq \varepsilon$ as long as $W_1(\eta^\pi(s),\hat\eta_n^\pi(s))\leq \varepsilon$.
However, we speculate that the additional ${\frac{1}{1-\gamma}}$ factor can be eliminated with more refined analysis techniques specially developed for handling infinite-dimensional cases.

Combining Theorem~\ref{Theorem_bound_of_Wdist} with the elementary inequality $\brk{W_p(\hat\eta_n^\pi(s),\eta^\pi(s))}^p\leq \frac{1}{(1-\gamma)^{p-1}}W_1(\hat\eta_n^\pi(s),\eta^\pi(s))$, we can derive the non-asymptotic results for $W_p$ metric.
\begin{corollary}\label{Corollary_bound_of_Wdist_p}
For any fixed policy $\pi$ and $p>1$, we have that 
\begin{equation*} \EB\sup_{s\in\gS}W_p(\hat\eta_n^\pi(s),\eta^\pi(s))\leq \brk{\frac{9\log|\gS|}{n(1-\gamma)^{2p+2}}}^{\frac{1}{2p}},
\end{equation*}
and for any $\delta\in(0,1)$, with probability at least $1-\delta$,
\begin{equation*}
\sup_{s\in\gS}W_p(\hat\eta_n^\pi(s),\eta^\pi(s))\leq \brk{\frac{\sqrt{9\log|\gS|}+\sqrt{\log(1/\delta)/2}}{\sqrt{n(1-\gamma)^{2p+2}}}}^{\frac{1}{p}}.
\end{equation*}

\end{corollary}

We comment that the slow rate $n^{-\frac{1}{2p}}$ for the $W_p$-metric is inevitable without assuming additional regularity conditions.
Consider an MDP with $\gS=\{s_1,s_2,s_3\}$ and $\gA=\{a_1\}$.
As there is only one single action $a_1$ available, the action variable can be safely omitted.
We start from $s_1$ with $r(s_1)=0$ and $P(s_2\mid s_1)=P(s_3\mid s_1)=\frac{1}{2}$.
And $s_2$ and $s_3$ are absorbing states with $r(s_2)=1$ and $r(s_3)=0$.
Suppose after $n$ calls of the generative model we have gained an estimator of $P(s_2\mid s_1)$ that is denoted as $\hat p$, then $\eta(s_1)=\frac{1}{2}\delta_{0}+\frac{1}{2}\delta_{\frac{1}{1-\gamma}}$, $\hat\eta_n(s_1)=(1-\hat p)\delta_{0}+\hat p\delta_{\frac{1}{1-\gamma}}$.
We have
\begin{equation*}
    W_p(\eta(s_1),\hat\eta_n(s_1))=\frac{1}{1-\gamma}\abs{\hat p - \frac{1}{2}}^{\frac{1}{p}}.
\end{equation*}
Since $\hat p\sim\bindist\prn{n,\frac{1}{2}}$, $\abs{\hat p - \frac{1}{2}}$ is of the order $n^{-\frac{1}{2}}$ by CLT.
Thus $W_p(\eta(s_1),\hat\eta_n(s_1))$ is of the order $n^{-\frac{1}{2p}}$.

Under Assumption~\ref{Assumption_reward_bounded_density}, we also have the following bounds on the KS metric.

\begin{theorem}\label{Theorem_bound_of_KS}
Suppose Assumption~\ref{Assumption_reward_bounded_density} holds true. 
For any fixed policy $\pi$, %we have for 
\begin{equation*} \EB\sup_{s\in\gS}\KS(\hat\eta_n^\pi(s),\eta^\pi(s))\leq C^{\prime\prime}\sqrt{\frac{\log|\gS|}{n(1-\gamma)^4}}.
\end{equation*}
And for any $\delta\in(0,1)$, with probability at least $1-\delta$,
\begin{equation*}
\sup_{s\in\gS}\KS(\hat\eta_n^\pi(s),\eta^\pi(s))\leq \frac{C^\prime\prn{\sqrt{\log|\gS|}+\sqrt{\log(1/\delta)}}}{\sqrt{n(1-\gamma)^4}}.
\end{equation*}
Here $C^\prime$ and $C^{\prime\prime}$ are constants only depending on $C$ in Assumption~\ref{Assumption_reward_bounded_density}.
\end{theorem}

The upper bound of the supreme KS metric between $\hat\eta_n$ and $\eta^\pi$ is of the same order as that of the $W_1$-metric, which indicates that under mild conditions learning a near-optimal return distribution in the sense of the KS metric is not more difficult than learning a near-optimal return distribution in the sense of $W_1$-metric.
This is somewhat a surprise since the distributional Bellman operator exhibits benign behaviors only when measured by Wasserstein metrics, and for $\mu$, $\nu$ with bounded support $W_1(\mu,\nu)$ can be always bounded by $\KS(\mu,\nu)$ multiplying a constant factor.

Simply combining the results in Theorem~\ref{Theorem_bound_of_Wdist} and Proposition~\ref{Proposition_bound_KS_with_W1} can only yield a sub-optimal $n^{-\frac{1}{4}}$ convergence rate for $\sup_{s\in\gS}\KS(\hat\eta_{n}(s),\eta^\pi(s))$.
Instead, we obtain a $n^{-\frac{1}{2}}$ rate using a quite different proof strategy with that of Theorem~\ref{Theorem_bound_of_Wdist}.
The first challenge is that the operator $\prn{\gI-\what\gT_n^\pi}^{-1}$ may be unbounded on its domain measured with the norm induced by the KS distance.
Therefore, although we can still write $(\hat\eta_n^\pi-\eta^\pi)=\prn{\gI-\what\gT_n^\pi}^{-1}\prn{\what\gT_{n}^\pi-\gT^\pi}\eta^\pi$, bounds of $\prn{\what\gT_{n}^\pi-\gT^\pi}\eta^\pi$ does not directly translate to bounds of $(\hat\eta_n^\pi-\eta^\pi)$.
We handle this challenge by an ``expansion trick", which raises yet another technical challenge: we need a stronger notion of concentration of $\prn{\what \gT_n^\pi-\gT^\pi}\eta^\pi$.
Specifically, unlike in the proof of Theorem~\ref{Theorem_bound_of_Wdist} where it suffices to bound $W_1\prn{\what\gT_n^\pi\eta^\pi,\gT^\pi\eta^\pi}$, here we further need to bound $\TV\prn{\what\gT_n^\pi\eta^\pi,\gT^\pi\eta^\pi}$.
And we achieve it with an analysis through the lens of density functions.
Detailed proof can be found in Appendix C.

\begin{theorem}\label{Theorem_bound_of_TV}
Suppose Assumption~\ref{Assumption_reward_smooth} holds.
For any fixed policy $\pi$, %we have for 
\begin{equation*} \EB\sup_{s\in\gS}\TV(\hat\eta_n^\pi(s),\eta^\pi(s))\leq K^{\prime\prime}\sqrt{\frac{\log|\gS|}{n(1-\gamma)^4}}.
\end{equation*}
And for any $\delta\in(0,1)$, with probability at least $1-\delta$,
\begin{equation*}
\sup_{s\in\gS}\TV(\hat\eta_n^\pi(s),\eta^\pi(s))\leq \frac{K^\prime\prn{\sqrt{\log|\gS|}+\sqrt{\log(1/\delta)}}}{\sqrt{n(1-\gamma)^4}}.
\end{equation*}
$K^\prime,K^{\prime\prime}$ are absolute constants that depend only on $M$ in Assumption~\ref{Assumption_reward_smooth}.
\end{theorem}

Based on the theorem mentioned above, we observe that our upper bounds for the supremum TV distance are of the same order as those for the supremum $W_1$ metric or the KS distance.
Directly applying Theorem~\ref{Theorem_bound_of_Wdist} and Proposition~\ref{Proposition_bound_TV_with_W1} only attain a slow $n^{-\frac{1}{4}}$ rate.
We rather employ a similar analytical approach as in the case of the KS distance to establish a standard convergence rate of $n^{-\frac{1}{2}}$.
Detailed proof is given in Appendix D.

\paragraph{Extension I: Less Exploratory Offline Dataset}
In our analysis, we assume the offline dataset is obtained via a generative model.
We can relax such assumption to that the dataset is generated via some probability measure $\xi$ and the transition dynamic $P$ \cite{chen2019information}.
Specifically, we first sample a batch of state-action pairs $\left\{(s_i,a_i)\right\}_{i=1}^m$ with $\xi$.
For state-action pair $(s_i,a_i)$, we sample the next-state $s_i^\prime$ according to $P(\cdot\mid s_i,a_i)$.
Then we have the dataset $\gD=\left\{(s_i,a_i,s_i^\prime)\right\}_{i=1}^m$.
The empirical estimate of the transition probability is constructed as follows.
\begin{equation*}
    \what P_m(s^\prime\mid s,a)=\frac{\sum_{i=1}^m \ind\{(s_i,a_i,s_i^\prime)=(s,a,s^\prime)\}}{1\vee\sum_{i=1}^m \ind\{(s_i,a_i)=(s,a)\}}.
\end{equation*}
We may define the certainty-equivalence estimator $\hat\eta_m^\pi$ with $\what P_m$ in the same manner as before.
\begin{theorem}\label{Theorem_offline_non_asymp}
    Let $\xi_{\min}=\min_{(s,a)\in\gS\times\gA}\xi(s,a)$.
    For any fixed policy $\pi$ and for any $\delta\in(0,1)$, with probability at least $1-\delta$, as long as $m\geq 8\log(2|\gS||\gA|/\delta)/\xi_{\min}$, we have
\begin{enumerate}
        \item[\emph{(a)}] $\sup_{s\in\gS}W_1(\hat\eta_m^\pi(s),\eta^\pi(s))\leq \frac{\sqrt{18\log|\gS|}+\sqrt{\log(2/\delta)}}{\sqrt{\xi_{\min}m(1-\gamma)^{4}}}$.
        \item[\emph{(b)}] $\sup_{s\in\gS}\KS(\hat\eta_m^\pi(s),\eta^\pi(s))\leq \frac{C^\prime\prn{\sqrt{\log|\gS|}+\sqrt{\log(1/\delta)}}}{\sqrt{\xi_{\min}m(1-\gamma)^4}}$ when Assumption~\ref{Assumption_reward_bounded_density} is true.
        Here $C^\prime$ is a constant only depending on $C$ in Assumption~\ref{Assumption_reward_bounded_density}.
        \item[\emph{(c)}] $\sup_{s\in\gS}\TV(\hat\eta_m^\pi(s),\eta^\pi(s))\leq \frac{K^\prime\prn{\sqrt{\log|\gS|}+\sqrt{\log(1/\delta)}}}{\sqrt{\xi_{\min}m(1-\gamma)^4}}$ when Assumption~\ref{Assumption_reward_smooth} is true.
        Here $K^\prime$ is a constant only depending on $M$ in Assumption~\ref{Assumption_reward_smooth}.
\end{enumerate}
\end{theorem}
One may refer to Appendix G for detailed proof.
The dependence on $\xi_{\min}$ is inevitable since our aim is to bound the $l_\infty$-type estimation error (e.g. $\sup_{s\in\gS}W_1(\hat\eta^\pi_m(s),\eta^\pi(s))$).
See Example G.1 in Appendix G for more concrete discussions.
And we may also observe such phenomenon in the asymptotic results (see Theorem~\ref{Theorem_weak_convergence_empirical_process_offline_data}).
If $\xi$ is the uniform distribution and $\xi_{\min}=1/(|\gS||\gA|)$, then the bounds here is equivalent with bounds in Theorem~\ref{Theorem_bound_of_Wdist},~\ref{Theorem_bound_of_KS},~\ref{Theorem_bound_of_TV} up to constant factors.

\paragraph{Extension II: Unknown Reward Distributions}
In our previous analysis, we assume that the reward distribution $\gP_R$ is known.
Here we remove this assumption and extend our analysis to the scenario where the reward distribution is estimated using a finite number of samples.
Specifically, for each state-action pair $(s,a)\in\gS\times\gA$, in addition to the $n$ next-state samples $X_1^{(s,a)}, \ldots, X_n^{(s,a)}\simiid P(\cdot\mid s,a)$ used to obtain the estimator $\what{P}_n(\cdot\mid s,a)$, we also sample $n$ rewards $R_1^{(s,a)}, \ldots, R_n^{(s,a)}\simiid \gP_R(\cdot\mid s,a)$ to estimate $\gP_R(\cdot\mid s,a)$.
We denote the estimated reward distribution as $\what{\gP}_{R,n}$, and the empirical distributional Bellman operator as $\wtilde{\gT}^\pi_n$.
The explicit forms of $\what{\gP}_{R,n}$ and $\wtilde{\gT}^\pi_n$ will be given later, depending on the metric we choose.
Now, we may define the estimator $\tilde{\eta}_n^\pi$ as the solution to the fixed point equation $\eta=\wtilde{\gT}^\pi_n\eta$.
\begin{theorem}\label{Theorem_unknown_reward_non_asymp}
    For any fixed policy $\pi$ and for any $\delta\in(0,1)$, with probability at least $1-\delta$, we have
\begin{enumerate}
        \item[\emph{(a)}] $\sup_{s\in\gS}W_1(\tilde\eta_n^\pi(s),\eta^\pi(s))\leq \frac{\sqrt{9\log|\gS|}+\sqrt{\log(1/\delta)/2}}{\sqrt{n(1-\gamma)^{4}}}$. 
        In this case, we may directly use the empirical reward distributions as estimator of $\gP_R$, and the empirical distributional Bellman operator is given by
        $$\brk{\wtilde{\gT}^\pi_n(\eta)}(s)=\frac{1}{n}\sum_{i=1}^n\sum_{a\in\gA}\pi(a|s)\prn{b_{R^{(s,a)}_i,\gamma}}_\#\eta(X^{(s,a)}_i).$$
        \item[\emph{(b)}] $\sup_{s\in\gS}\KS(\tilde\eta_n^\pi(s),\eta^\pi(s))\leq \frac{C^\prime\prn{\sqrt{\log|\gS|}+\sqrt{\log(1/\delta)}}}{\sqrt{n(1-\gamma)^4}}+\frac{C^\prime}{1-\gamma}\sup_{s\in\gS,a\in\gA}\norm{\what{\gP}_{R,n}(\cdot\mid s,a)-\gP_R(\cdot\mid s,a)}_{\FTV}$ when Assumption~\ref{Assumption_reward_bounded_density} is true.
        Here $C^\prime$ is a constant only depending on $C$ in Assumption~\ref{Assumption_reward_bounded_density}.
        In this case, $\what{\gP}_{R,n}(\cdot\mid s,a)$ can be any density estimator as long as it has a Lebesgue density upper-bounded by $2C$ for each $(s,a)\in\gS\times\gA$, and $\wtilde{\gT}^\pi_n$ is the distributional Bellman operator of the empirical MDP $\wtilde{M}=\<\gS,\gA,\what{\gP}_{R,n},\what{P}_n,\gamma\>$.
        \item[\emph{(c)}] $\sup_{s\in\gS}\TV(\tilde\eta_n^\pi(s),\eta^\pi(s))\leq \frac{K^\prime\prn{\sqrt{\log|\gS|}+\sqrt{\log(1/\delta)}}}{\sqrt{n(1-\gamma)^4}}+\frac{K^\prime}{1-\gamma}\sup_{s\in\gS,a\in\gA}\norm{\what{\gP}_{R,n}(\cdot\mid s,a)-\gP_R(\cdot\mid s,a)}_{\FTV}$ when Assumption~\ref{Assumption_reward_smooth} is true.
        Here $K^\prime$ is a constant only depending on $M$ in Assumption~\ref{Assumption_reward_smooth}.
        In this case, $\what{\gP}_{R,n}(\cdot\mid s,a)$ can be any density estimator as long as it has a Lebesgue density $\hat p_{s,a}^R\in H_1^1(\RB)$ with $\norm{\hat p_{s,a}^R}_{H_1^1(\RB)}\leq 2M$, and $\wtilde{\gT}^\pi_n$ is the distributional Bellman operator of the empirical MDP $\wtilde{M}=\<\gS,\gA,\what{\gP}_{R,n},\what{P}_n,\gamma\>$.
\end{enumerate}
\end{theorem}
See Appendix G for a detailed proof.
Our non-asymptotic bounds here depend on specific choices of the reward estimator $\what{\gP}_{R,n}$.
One may refer to Remark\ref{Remark_density_estimator} for some choices of density estimators and the corresponding non-asymptotic bounds of total variation $\norm{\what{\gP}_{R,n}(\cdot\mid s,a)-\gP_R(\cdot\mid s,a)}_{\FTV}$.

\paragraph{Implications in Risk-sensitive RL} We conclude this section with a brief discussion that highlights the implications of our results in the field of risk-sensitive RL.
The main goal of risk-sensitive RL is to find a policy $\pi^\star$ minimizing a risk functional of the return distribution.
Concretely, we define $L_\rho(\pi):=\rho(\eta^\pi(s))$ and $\pi^\star:=\arg\min_{\pi}L_\rho(\pi)$.
Here $\rho(\cdot)$ is some risk functional such as value-at-risk or conditional value-at-risk, and $s$ the initial state.
When the underlying MDP is not explicitly known, $\pi^*$ can be estimated by $\pi^\star:=\arg\min_{\pi}\what L_\rho(\pi)$, where $\what L_\rho(\pi):=\rho(\hat\eta_n^\pi(s))$.
If $\rho(\cdot)$ is Lipschitz continuous w.r.t. some probability metric, then we may derive the following sample complexity bounds with our results above.
\begin{corollary}\label{Corollary_risk_sensitive_RL}
    $n=\wtilde{O}\left(\frac{|\gS||\gA|}{\varepsilon^2(1-\gamma)^4
    }\right)$ is sufficient to guarantee $L_\rho(\hat\pi)-L_\rho(\pi^\star)\leq \varepsilon$ with probability at least $1-\delta$ as long as one of the following conditions is true:
\begin{enumerate}
        \item[\emph{(a)}] $\rho(\cdot)$ is Lipschitz continuous w.r.t. the $W_1$ metric.
        \item[\emph{(b)}] $\rho(\cdot)$ is Lipschitz continuous w.r.t. the KS metric, and Assumption~\ref{Assumption_reward_bounded_density} holds.
        \item[\emph{(c)}] $\rho(\cdot)$ is Lipschitz continuous w.r.t. the TV metric and Assumption~\ref{Assumption_reward_smooth} holds.
\end{enumerate}
\end{corollary}
The Lipschitz condition is general since it covers a wide range of risk measures, including the class of distortion risk measures, convex and coherent measures. etc. \cite{liang2023regret}.

The idea behind Corollary~\ref{Corollary_risk_sensitive_RL} is straightforward.
As long as we have non-asymptotic bounds for $d(\eta^\pi(s),\eta_n^\pi(s))$ that hold for arbitrary fixed $\pi$ (Theorem~\ref{Theorem_bound_of_Wdist},~\ref{Theorem_bound_of_KS},~\ref{Theorem_bound_of_TV}), then we may use the covering argument bound $\sup_{\pi}d(\eta^\pi(s),\eta_n^\pi(s))$.
And such uniform convergence results can further lead to bounds on $L_\rho(\hat\pi)-L_\rho(\hat\pi)\leq \varepsilon$.

\subsection{Results on Asymptotic Analysis}\label{Section_asymp}

We first give our main results of the asymptotic analysis in Theorem~\ref{Theorem_weak_convergence_empirical_process}.

\begin{theorem}\label{Theorem_weak_convergence_empirical_process}
    For any fixed policy $\pi$, we have for any $s\in\gS$
    \begin{enumerate}
        \item[\emph{(a)}] $\sqrt{n}\prn{\hat\eta_n^\pi(s)-\eta^\pi(s)}$ converge weakly to the process $f\mapsto \brk{\prn{\gI-\gT^\pi}^{-1}\wtilde\GB^\pi}(s)f$ in $\ell^\infty(\FW)$, where $\FW:=\brc{f\mid f \text{ is supported on $\brk{0,\frac{1}{1-\gamma}}$ and $1$-Lipschitz}}$.
        \item[\emph{(b)}] If Assumption~\ref{Assumption_reward_bounded_density} is true, then $\sqrt{n}\prn{\hat\eta_n^\pi(s)-\eta^\pi(s)}$ converges weakly to the process $f\mapsto \brk{\prn{\gI-\gT^\pi}^{-1}\wtilde\GB^\pi}(s)f$ in $\ell^\infty(\FKS)$, where $\FKS:=\brc{\ind_{(-\infty,z]}\mid z\in\brk{0,\frac{1}{1-\gamma}}}$.

        \item[\emph{(c)}] If Assumption~\ref{Assumption_reward_smooth} is true, then $\sqrt{n}\prn{\hat\eta_n^\pi(s)-\eta^\pi(s)}$ converge weakly to the process $f\mapsto \brk{\prn{\gI-\gT^\pi}^{-1}\wtilde\GB^\pi}(s)f$ in $\ell^\infty(\FTV)$, where $\FTV:=\brc{\ind_{A}\mid  A\subseteq \brk{0,\frac{1}{1-\gamma}}\text{ is Borel}}$.
    \end{enumerate}
    
    Here the random element $\wtilde\GB^\pi$ is defined as
    \begin{equation*}
        \wtilde\GB^\pi(s):=\sum_{a\in\gA}\pi(a\mid s)\sum_{s^\prime\in\gS} Z_{s,a,s^\prime}\int_0^1 \prn{b_{r,\gamma}}_\#\eta^\pi(s^\prime)\gP_R(dr\mid s,a),\ \forall s\in\gS,
    \end{equation*}
    where  $\prn{Z_{s,a,s^\prime}}_{(s,a,s^\prime)\in\gS\times\gA\times\gS}$ are zero-mean Gaussians with 
    \begin{equation*}
        \cov(Z_{s_1,a_1,s_1^\prime},Z_{s_2,a_2,s_2^\prime})=\ind\brc{(s_1,a_1)=(s_2,a_2)}P(s_1^\prime\mid s_1,a_1)\prn{\ind\brc{s_1^\prime=s_2^\prime}-P(s_2^\prime\mid s_1,a_1)}.
    \end{equation*}
    And the operator $\prn{\gI-\gT^\pi}^{-1}$ is defined as $\prn{\gI-\gT^\pi}^{-1}:=\sum_{i=0}^\infty \prn{\gT^\pi}^i$.
\end{theorem}

At a high level, we depict the asymptotic behavior of $\sqrt{n}\prn{\hat\eta^\pi_n-\eta^\pi}$ by showing that the ``empirical processes" induced by $\sqrt{n}\prn{\hat\eta^\pi_n-\eta^\pi}$ converge to a Gaussian random element.
Moreover, the limiting random element has a simple structure in the sense that it is a linear transformation of a finite mixture of probability distributions with Gaussian coefficients.
Our asymptotic results are general in the sense that the conclusions are valid in different spaces under different regularity conditions: $\ell^\infty(\FW)$, $\ell^\infty(\FKS)$, and $\ell^\infty(\FTV)$.
Therefore, our findings have the potential to yield numerous valuable inferential procedures for the field of distributional reinforcement learning.
% And the covariance function can be consistently estimated using standard plug-in approaches (use $\what P$ in place of $P$ and $\hat\eta_n^\pi$ in place of $\eta^\pi$).
Our proof of Theorem~\ref{Theorem_weak_convergence_empirical_process} builds on the foundation of our non-asymptotic analysis and can be found in Section~\ref{Section_analysis}.
Detailed proof can be found in Appendix E.

\paragraph{Extension I: Less Exploratory Offline Dataset}
We also present asymptotic results in the setting of less exploratory dataset.

\begin{theorem}\label{Theorem_weak_convergence_empirical_process_offline_data}
    For any fixed policy $\pi$, we have for any $s\in\gS$
    \begin{enumerate}
        \item[\emph{(a)}] $\sqrt{m}\prn{\hat\eta_m^\pi(s)-\eta^\pi(s)}$ converge weakly to the process $f\mapsto \brk{\prn{\gI-\gT^\pi}^{-1}\mathring\GB^\pi}(s)f$ in $\ell^\infty(\FW)$.
        \item[\emph{(b)}] If Assumption~\ref{Assumption_reward_bounded_density} is true, then $\sqrt{m}\prn{\hat\eta_m^\pi(s)-\eta^\pi(s)}$ converges weakly to the process $f\mapsto \brk{\prn{\gI-\gT^\pi}^{-1}\mathring\GB^\pi}(s)f$ in $\ell^\infty(\FKS)$.

        \item[\emph{(c)}] If Assumption~\ref{Assumption_reward_smooth} is true, then $\sqrt{m}\prn{\hat\eta_m^\pi(s)-\eta^\pi(s)}$ converge weakly to the process $f\mapsto \brk{\prn{\gI-\gT^\pi}^{-1}\mathring\GB^\pi}(s)f$ in $\ell^\infty(\FTV)$.
    \end{enumerate}
    
    Here the random element $\mathring\GB^\pi$ is defined as
    \begin{equation*}
        \mathring\GB^\pi(s):=\sum_{a\in\gA}\pi(a\mid s)\sum_{s^\prime\in\gS}\frac{ \mathring Z_{s,a,s^\prime}}{\sqrt{\xi(s,a)}}\int_0^1 \prn{b_{r,\gamma}}_\#\eta^\pi(s^\prime)\gP_R(dr\mid s,a),\ \forall s\in\gS,
    \end{equation*}
    where  $\prn{\mathring Z_{s,a,s^\prime}}_{(s,a,s^\prime)\in\gS\times\gA\times\gS}$ are zero-mean Gaussians with 
    \begin{equation*}
        \cov(\mathring Z_{s_1,a_1,s_1^\prime},\mathring Z_{s_2,a_2,s_2^\prime})=\ind\brc{(s_1,a_1)=(s_2,a_2)}P(s_1^\prime\mid s_1,a_1)\prn{\ind\brc{s_1^\prime=s_2^\prime}-\xi(s_1,a_1)P(s_2^\prime\mid s_1,a_1)}.
    \end{equation*}
\end{theorem}
The main difference between the limiting random elements $\mathring \GB^\pi$ and $\wtilde \GB^\pi$ is that $\mathring \GB^\pi$ has larger ``variance".
This is because we introduce additional randomness in the process of sampling the state-action pairs from distribution $\xi$.
We can also observe the $\frac{1}{\sqrt{\xi(s,a)}}$ factor in $\mathring \gB^\pi(s)$.
This suggests that the $l_\infty$-type error bound would inevitably depend on the factor $\frac{1}{\sqrt{\xi_{\min}}}$ as in Theorem~\ref{Theorem_offline_non_asymp}.

\paragraph{Extension II: Unknown Reward Distributions}
We give the asymptotic results in the setting of unknown reward distributions.
\begin{theorem}\label{Theorem_weak_convergence_empirical_process_unknown_reward}
    For any fixed policy $\pi$, we have for any $s\in\gS$
    \begin{enumerate}
        \item[\emph{(a)}] $\sqrt{n}\prn{\tilde\eta_n^\pi(s)-\eta^\pi(s)}$ converge weakly to the process $f\mapsto \brk{\prn{\gI-\gT^\pi}^{-1}\bar\GB^\pi}(s)f$ in $\ell^\infty(\FW)$.
        Here $\wtilde{\gT}^\pi_n$ is defined as in part (a) of Theorem~\ref{Theorem_unknown_reward_non_asymp}, and the random element $\bar{\GB}^\pi(s)$ is a zero-mean Gaussian process in $\ell^\infty(\FW)$ with covariance function: $\forall f, g\in\FW,$
            \begin{equation*}
            \begin{aligned}
            &\cov\big(\bar\GB^\pi(s)f,\bar\GB^\pi(s)g\big)=\sum_{a\in\gA}\pi(a\mid s)^2\Bigg\{\sum_{s^\prime\in\gS}P(s^\prime\mid s,a)\int_0^1 \brk{\prn{b_{r,\gamma}}_\#\eta^\pi(s^\prime)f}\brk{\prn{b_{r,\gamma}}_\#\eta^\pi(s^\prime)g}\gP_R(dr\mid s,a)\\
            &-\brk{\sum_{s^\prime\in\gS}P(s^\prime\mid s,a)\int_0^1 \prn{b_{r,\gamma}}_\#\eta^\pi(s^\prime)f\gP_R(dr\mid s,a)}\brk{\sum_{s^\prime\in\gS}P(s^\prime\mid s,a)\int_0^1 \prn{b_{r,\gamma}}_\#\eta^\pi(s^\prime)g\gP_R(dr\mid s,a)}\Bigg\}.
            \end{aligned}
    \end{equation*}
        \item[\emph{(b)}] If Assumption~\ref{Assumption_reward_bounded_density} is true and $\sqrt{n}\prn{\what{\gP}_{R,n}(\cdot\mid s,a)-\gP_R(\cdot\mid s,a)}$ converges weakly to a tight random element $\GB^R_{s,a}$ in $\ell^\infty(\FKS)$ for all $(s,a)\in\gS\times\gA$, then $\sqrt{n}\prn{\tilde\eta_n^\pi(s)-\eta^\pi(s)}$ converges weakly to the process $f\mapsto \brk{\prn{\gI-\gT^\pi}^{-1}\prn{\wtilde\GB^\pi+\GB^\pi_R}}(s)f$ in $\ell^\infty(\FKS)$, where
        % $\wtilde{\gT}^\pi_n$ is given in Theorem~\ref{Theorem_unknown_reward_non_asymp} (b), 
        $\wtilde\GB^\pi$ is defined in Theorem~\ref{Theorem_weak_convergence_empirical_process} and $\GB^\pi_R(s)$ is independent of $\wtilde\GB^\pi$ and given by
        \begin{equation*}
            \begin{aligned}
            \GB^\pi_R(s)=\sum_{a\in\gA,s^\prime\in\gS}\pi(a\mid s)P(s^\prime\mid s,a)\int_0^1 \prn{b_{r,\gamma}}_\#\eta^\pi(s^\prime)\GB^R_{s,a}(dr).
            \end{aligned}
        \end{equation*}

        \item[\emph{(c)}] If Assumption~\ref{Assumption_reward_smooth} is true and $\sqrt{n}\prn{\what{\gP}_{R,n}(\cdot\mid s,a)-\gP_R(\cdot\mid s,a)}$ converges weakly to a tight random element $\GB^R_{s,a}$ in $\ell^\infty(\FTV)$ for all $(s,a)\in\gS\times\gA$, then $\sqrt{n}\prn{\tilde\eta_n^\pi(s)-\eta^\pi(s)}$ converge weakly to the process $f\mapsto \brk{\prn{\gI-\gT^\pi}^{-1}\prn{\wtilde\GB^\pi+\GB^\pi_R}}(s)f$ in $\ell^\infty(\FTV)$, where 
        % $\wtilde{\gT}^\pi_n$ is given in Theorem~\ref{Theorem_unknown_reward_non_asymp} (c), 
        $\wtilde\GB^\pi$ is given in Theorem~\ref{Theorem_weak_convergence_empirical_process} and $\GB^\pi_R(s)$ is independent of $\wtilde\GB^\pi$ and given by
        \begin{equation*}
            \begin{aligned}
            \GB^\pi_R(s)=\sum_{a\in\gA,s^\prime\in\gS}\pi(a\mid s)P(s^\prime\mid s,a)\int_0^1 \prn{b_{r,\gamma}}_\#\eta^\pi(s^\prime)\GB^R_{s,a}(dr).
            \end{aligned}
        \end{equation*}
    \end{enumerate}
\end{theorem}

For the weak convergence in $\ell^\infty(\FKS)$ (or $\ell^\infty(\FTV)$) we require the reward estimator satisfies that $\sqrt{n}\prn{\what{\gP}_{R,n}(\cdot\mid s,a)-\gP_R(\cdot\mid s,a)}$ converges weakly to a tight random element in $\ell^\infty(\FKS)$ (or $\ell^\infty(\FTV)$).
Such an assertion is invalid for most non-parametric density estimators like the histogram estimators, kernel density estimators, etc.
Since the typical convergence rates of most density estimators are slower than $O(n^{-1/2})$.
However, such an assertion does hold if we confine $\gP_R(\cdot\mid s,a)$ to some parametric families.
For example, set $\gP_R(\cdot\mid s,a)$ to be truncated normal distributions with known mean and variance.
We comment that we may observe the phenomenon of inflated variance in the limiting distribution as before since estimating the reward distribution induces new randomness.

\section{Statistical Inference}\label{Section_inference}
% \subsection{Statistical Inferences for \texorpdfstring{$W_1(\hat\eta_n^\pi(s),\eta^\pi(s))$}{W1} and \texorpdfstring{$\mathrm{KS}(\hat\eta_n^\pi(s), \eta^\pi(s))$}{KS} }
In this section, we consider the statistical inference of distributional reinforcement learning. 
% First, we present inference on $W_1(\hat\eta_n^\pi(s),\eta^\pi(s))$ and $\mathrm{KS}(\hat\eta_n^\pi(s), \eta^\pi(s))$. Second, we study inference on Hadamard differentiable functionals.  
First, we present non-parametric confidence sets for $\eta^\pi(s)$ in the forms of $W_1$, \KS, and \TV\ balls.
Second, we study inference on Hadamard differentiable functionals, with moments, quantiles, and uniform advantage of policy as special examples. 
\subsection{Inferences for \texorpdfstring{$\eta^\pi(s)$}{eta}}

Our theoretical findings in Theorems~\ref{Theorem_weak_convergence_empirical_process} allow us to construct confidence sets in the space $\Delta\prn{\brk{0,\frac{1}{1-\gamma}}}$ for the true return distribution $\eta^\pi(s)$, given any initial state $s\in\gS$.
Specifically, we can construct three types of confidence sets for $\eta^\pi(s)$: $W_1$, $\KS$ and $\TV$ balls.
% one in the form of a $1$-Wasserstein ball and the other in the form of a $\KS$-metric ball.
\begin{theorem}
    \label{Theorem_inference_W1_ball}
    For some fixed policy $\pi$ and initial state $s\in\gS$, define
    % \begin{enumerate}
        % \item[] 
        $\rho_1(\alpha):=\frac{z_1(1-\alpha)}{\sqrt{n}}$, where $z_1(p)$ is defined as the $p$-quantile of $\sup_{f\in\FW} \brk{(\gI-\gT^\pi)^{-1}\wtilde \GB^\pi} (s)f$.
        % \item[] $\rho_2:=\frac{z_2(1-\alpha)}{\sqrt{n}}$, where $z_2(p)$ is defined as the $p$-quantile of $\int_0^{\frac{1}{1-\gamma}}\abs{\brk{(\gI-\gT^\pi)^{-1}\wtilde \GB^\pi} (s)\ind\brc{\cdot\leq x}}dx$.
    % \end{enumerate}
    % Then we have
    % \begin{enumerate}
    %     \item[\emph{(a)}] $\lim_{n\to\infty}\PB\prn{W_1(\hat\eta_n^\pi(s),\eta^\pi(s))\leq\rho_1}=1-\alpha$;
    %     \item[\emph{(b)}]  $\lim_{n\to\infty}\PB\prn{W_1(\hat\eta_n^\pi(s),\eta^\pi(s))\leq\rho_2}=1-\alpha$ if Assumption~\ref{Assumption_reward_bounded_density} holds.
    % \end{enumerate}
        % Then we have 
        % \begin{equation*}
        %     \lim_{n\to\infty}\PB\prn{W_1(\hat\eta_n^\pi(s),\eta^\pi(s))\leq\rho_1(\alpha)}=1-\alpha. 
        % \end{equation*}
        Let the confidence set
        \begin{equation*}
        C_1(\alpha):=\brc{\eta\in\Delta\prn{\brk{0,\frac{1}{1-\gamma}}}\Big\vert \ W_1(\eta,\hat \eta_n^{\pi}(s))\leq \rho_1(\alpha)},
    \end{equation*}
    % where $i=1,2$.
    then
    \begin{equation*}
    \lim_{n\to\infty}\PB\prn{\eta^\pi(s)\in C_1(\alpha)}=1-\alpha.
    \end{equation*}
        Furthermore, if Assumption~\ref{Assumption_reward_bounded_density} holds, we have 
        \begin{equation*}
         \sup_{f\in\FW} \brk{(\gI-\gT^\pi)^{-1}\wtilde \GB^\pi} (s)f=\int_0^{\frac{1}{1-\gamma}}\abs{\brk{(\gI-\gT^\pi)^{-1}\wtilde \GB^\pi} (s)\ind\brc{\cdot\leq x}}dx.   
        \end{equation*}
\end{theorem}
The proof is in Appendix H.
% Proposition~\ref{Proposition_W1_quantile} depicts a $W_1$-ball type confidence set of $\eta^\pi$.
% $\rho_1$ is a finite-sample upper bound for the $(1-\alpha)$-quantile of $W_1\prn{\hat\eta_n(s),\eta^\pi(s)}$ through the non-asymptotic guarantees in Theorem~\ref{Theorem_bound_of_Wdist}.
Recall that for two probability distributions $\mu_1,\mu_2$ supported on $\brk{0,\frac{1}{1-\gamma}}$ we have 
\begin{equation*}
    W_1(\mu_1,\mu_2)=\sup_{f\in\FW}\abs{\mu_1 f-\mu_2 f}=\int_0^{\frac{1}{1-\gamma}}\abs{F_1(x)-F_2(x)}dx,
\end{equation*}
where $F_1$ and $F_2$ are the cumulative distribution functions of $\mu_1$ and $\mu_2$, respectively.
Hence, the asymptotic distribution of $W_1\prn{\hat\eta_n(s),\eta^\pi(s)}$ can be described in two different ways using asymptotic results in Theorem~\ref{Theorem_weak_convergence_empirical_process} and the continuous mapping theorem.
And $\rho_1(\alpha)$ can be determined accordingly.

% Now we are ready to present our confidence sets for $\eta^\pi(s)$ that take the form of $1-$Wasserstein balls.

% \begin{theorem}
%     \label{Theorem_inference_W1_ball}
%     For some fixed policy $\pi$ and initial state $s\in\gS$, define
%         \begin{equation*}
%         C_1(\alpha):=\brc{\eta\in\Delta\prn{\brk{0,\frac{1}{1-\gamma}}}\Big\vert \ W_1(\eta,\hat \eta_n^{\pi}(s))\leq \rho_1(\alpha)}.
%     \end{equation*}
%     % where $i=1,2$.
%     Then we have 
% $\lim_{n\to\infty}\PB\prn{\eta^\pi(s)\in C_1(\alpha)}=1-\alpha$.
%     % \begin{enumerate}
%     %     \item[\emph{(a)}] $\lim_{n\to\infty}\PB\prn{\eta^\pi(s)\in C_1}=1-\alpha$;
%     %     \item[\emph{(b)}]  $\lim_{n\to\infty}\PB\prn{\eta^\pi(s)\in C_2}=1-\alpha$ if Assumption~\ref{Assumption_reward_bounded_density} holds.
%     % \end{enumerate}
% \end{theorem}
% $C_1$ has the virtue of being a finite-sample confidence set and works no matter how large $n$ is.
% However, it may also be overly conservative for our inferential purpose.
The confidence set $C_1(\alpha)$ is asymptotically valid, but it relies on the quantile, \ie, $z_1(1-\alpha)$, of the unknown limiting distributions that depend on $\gT^\pi$ and $\eta^\pi$.
We may get a consistent estimate of $z_1(1-\alpha)$ by the plug-in approach.
\begin{proposition}\label{Proposition_plug_in_W1}
    For any fixed policy $\pi$ and initial state $s\in\gS$, define
    \begin{equation*}
        \what\GB^\pi(s):=\sum_{a\in\gA}\pi(a\mid s)\sum_{s^\prime\in\gS} \what Z_{s,a,s^\prime}\int_0^1 \prn{b_{r,\gamma}}_\#\hat\eta_n^\pi(s^\prime)d\gP_R(dr\mid s,a),\ \forall s\in\gS,
    \end{equation*}
    where  $\prn{\what Z_{s,a,s^\prime}}_{(s,a,s^\prime)\in\gS\times\gA\times\gS}$ are zero-mean Gaussians with 
    \begin{equation*}
        \cov(\what Z_{s_1,a_1,s_1^\prime},\what Z_{s_2,a_2,s_2^\prime})=\ind_{\brc{(s_1,a_1)=(s_2,a_2)}}\what P(s_1^\prime\mid s_1,a_1)\prn{\ind_{\brc{s_1^\prime=s_2^\prime}}-\what P(s_2^\prime\mid s_1,a_1)},
    \end{equation*}
    and
    \begin{equation*}
    \begin{aligned}
    \hat z_1(p)&:=\inf\brc{t\ \Bigg\vert\ \PB\prn{\sup_{f\in\FW} \brk{(\gI-\what\gT^\pi)^{-1}\what \GB^\pi} (s)f\leq t}\geq p}.\\
    % \hat z_2(p)&:=\inf\brc{t\ \Bigg\vert\ \PB\prn{\int_0^{\frac{1}{1-\gamma}}\abs{\brk{(\gI-\what\gT^\pi)^{-1}\what \GB^\pi} (s)\ind_{(-\infty,x]}}dx\leq t}\geq p}.
    \end{aligned}
    \end{equation*}
    Then $\hat z_1(p)\cp z_1(p)$ if $z_1(\cdot)$ is continuous at $p$. 
    
    Furthermore, if Assumption~\ref{Assumption_reward_bounded_density} holds, we have 
        \begin{equation*}
    \begin{aligned}
\sup_{f\in\FW} \brk{(\gI-\what\gT^\pi)^{-1}\what \GB^\pi} (s)f=\int_0^{\frac{1}{1-\gamma}}\abs{\brk{(\gI-\what\gT^\pi)^{-1}\what \GB^\pi} (s)\ind_{(-\infty,x]}}dx,
    \end{aligned}
    \end{equation*}which can be computed efficiently, and $z_1(\cdot)$ is continuous at any $p\in(0,1)$.
\end{proposition}
The proof is in Appendix H.
% When Assumption~\ref{Assumption_reward_bounded_density} holds, it is easy to verify that $z_1$ and $z_2$ is continuous at any $p\in(0,1)$.

We can also construct confidence sets in the form of $\KS$ balls and $\TV$ balls for $\eta^\pi(s)$ when Assumption~\ref{Assumption_reward_bounded_density} or Assumption~\ref{Assumption_reward_smooth} holds.
\begin{theorem}
    \label{Theorem_inference_KS_ball}
    For some fixed policy $\pi$ and $s\in\gS$, define
    \begin{enumerate}
        \item[] $\rho_2(\alpha):=\frac{z_2(1-\alpha)}{\sqrt{n}}$, where $z_2(p)$ is defined as the $p$-quantile of $\sup_{f\in\FKS} \brk{(\gI-\gT^\pi)^{-1}\wtilde \GB^\pi} (s)f$,
        \item[] $\rho_3(\alpha):=\frac{z_3(1-\alpha)}{\sqrt{n}}$, where $z_3(p)$ is defined as the $p$-quantile of $\sup_{f\in\FTV} \brk{(\gI-\gT^\pi)^{-1}\wtilde \GB^\pi} (s)f$.
    \end{enumerate}
    Let 
    \begin{enumerate}
        \item[] $C_2(\alpha):=\brc{\eta\in\Delta\prn{\brk{0,\frac{1}{1-\gamma}}}\Big\vert\ \KS(\eta,\hat \eta_n^\pi(s))\leq \rho_2(\alpha)},$
        \item[] $C_3(\alpha):=\brc{\eta\in\Delta\prn{\brk{0,\frac{1}{1-\gamma}}}\Big\vert\ \TV(\eta,\hat \eta_n^\pi(s))\leq \rho_3(\alpha)}.$
    \end{enumerate}
    Then we have 
    \begin{enumerate}
        \item[\emph{(a)}] $\lim_{n\to\infty}\PB\prn{\eta^\pi(s)\in C_2(\alpha)}=1-\alpha$ under Assumption~\ref{Assumption_reward_bounded_density};
        \item[\emph{(b)}]  $\lim_{n\to\infty}\PB\prn{\eta^\pi(s)\in C_3(\alpha)}=1-\alpha$ under Assumption~\ref{Assumption_reward_smooth}.
        \end{enumerate}
\end{theorem}

$\rho_2(\alpha)$ and $\rho_3(\alpha)$ asymptotically discribe the quantile of $\KS(\hat\eta_n^\pi(s),\eta^\pi(s))$, $\TV(\hat\eta_n^\pi(s),\eta^\pi(s))$, respectively. They are determined using results in Theorem~\ref{Theorem_weak_convergence_empirical_process} and the continuous mapping theorem.

% \begin{theorem}
%     \label{Theorem_inference_KS_ball}
%     Suppose Assumption~\ref{Assumption_reward_bounded_density} or Assumption~\ref{Assumption_reward_smooth} holds. For some fixed policy $\pi$ and initial state $s\in\gS$, define 
%         \begin{equation*}
%     \begin{aligned}
%     &C_2(\alpha):=\brc{\eta\in\Delta\prn{\brk{0,\frac{1}{1-\gamma}}}\Big\vert\ \KS(\eta,\hat \eta_n^\pi(s))\leq \rho_2(\alpha)},\\
%     &C_3(\alpha):=\brc{\eta\in\Delta\prn{\brk{0,\frac{1}{1-\gamma}}}\Big\vert\ \TV(\eta,\hat \eta_n^\pi(s))\leq \rho_3(\alpha)}.
%     \end{aligned}
%     \end{equation*}
%     Then we have
%     \begin{enumerate}
%         \item[\emph{(a)}]  $\lim_{n\to\infty}\PB\prn{\eta^\pi(s)\in C_2(\alpha)}=1-\alpha$.
%         \item[\emph{(b)}] $\lim_{n\to\infty}\PB\prn{\eta^\pi(s)\in C_3(\alpha)}=1-\alpha$.
%     \end{enumerate}
% \end{theorem}

% Note that $C_4$ is a finite sample confidence set 
Note that $C_2(\alpha)$ and $C_3(\alpha)$ are asymptotically valid confidence sets.
Although they rely on the unknown quantile function $z_2(1-\alpha)$ and $z_3(1-\alpha)$, they can be consistently estimated using a plug-in approach as the case of $z_1(1-\alpha)$.

\begin{proposition}\label{Proposition_plug_in_KS}
    % Let Assumption~\ref{Assumption_reward_bounded_density} hold.
    For any fixed $\pi$ and $s\in\gS$, define
    % \begin{equation*}
    %     \what\GB^\pi(s):=\sum_{a\in\gA}\pi(a\mid s)\sum_{s^\prime\in\gS} Z_{s,a,s^\prime}\int_0^1 \prn{b_{r,\gamma}}_\#\hat\eta_n^\pi(s^\prime)d\gP_R(dr\mid s,a),\ \forall s\in\gS,
    % \end{equation*}
    % and
    \begin{equation*}
    \begin{aligned}
    \hat z_2(p)&:=\inf\brc{t\ \Bigg\vert\ \PB\prn{\sup_{f\in\FKS} \brk{(\gI-\what\gT^\pi)^{-1}\what \GB^\pi} (s)f\leq t}\geq p}.\\
    \hat z_3(p)&:=\inf\brc{t\ \Bigg\vert\ \PB\prn{\sup_{f\in\FTV} \brk{(\gI-\what\gT^\pi)^{-1}\what \GB^\pi} (s)f\leq t}\geq p}.
    \end{aligned}
    \end{equation*}
    Then $\hat z_2(p)\cp z_2(p)$ if Assumption~\ref{Assumption_reward_bounded_density} holds,
    % $z_3(\cdot)$ is continuous at $p$.
    $\hat z_3(p)\cp z_3(p)$ if Assumption~\ref{Assumption_reward_smooth} holds.
    % $z_3(\cdot)$ is continuous at $p$.
\end{proposition}
One may refer to the proof in Appendix H.
% As long as Assumption~\ref{Assumption_reward_bounded_density} holds, $z_3$ is continuous at any $p\in(0,1)$.

\subsection{Inference for Hadamard Differentiable Functionals}

We consider the problem of statistical inference for $\phi(\eta^\pi(s))$, where $\phi(\cdot)\colon \ell^{\infty}(\FW)\to \RB$ represents a statistical functional.
When $\phi(\cdot)$ is Hadamard differentiable, we can determine the limiting distribution of $\sqrt{n}\prn{\phi(\hat\eta_n^\pi(s))-\phi(\eta^\pi(s))}$ using the functional Delta method. 
Subsequently, we can construct asymptotic confidence sets for $\phi(\eta^\pi(s))$ based on this result.

\begin{theorem}\label{Theorem_Hadamard_Inference}
    For fixed policy $\pi$ and $s\in\gS$, let $\phi(\cdot)\colon \ell^{\infty}(\FW)\to \RB$ be Hadamard differentiable at $\eta^\pi(s)$ tangentially to $\DB_0\subset \ell^{\infty}(\FW)$.
    Suppose $\brk{\prn{\gI-\gT^\pi}^{-1}\wtilde \GB^\pi}(s)\in\DB_0$ and define
    \begin{equation*}
        C_\phi(\alpha) := \brk{\phi(\hat\eta_n^\pi(s))+\frac{z_\phi(\alpha/2)}{\sqrt{n}}, \phi(\hat\eta_n^\pi(s))+\frac{z_\phi(1-\alpha/2)}{\sqrt{n}}},
    \end{equation*}
    where $z_\phi$ is the quantile function of $\phi^\prime_{\eta^\pi(s)}\prn{\brk{\prn{\gI-\gT^\pi}^{-1}\wtilde \GB^\pi}(s)}$, which is indeed a one-dimensional gaussian variable with zero means.
    We have
    \begin{equation*}
        \lim_{n\to\infty} \PB\prn{\phi(\eta^\pi(s))\in C_\phi(\alpha)}=1-\alpha.
    \end{equation*}
    Under Assumption~\ref{Assumption_reward_bounded_density} or Assumption~\ref{Assumption_reward_smooth}, we have similar results for $\phi(\cdot)\colon \ell^{\infty}(\FKS)\to \RB$ or $\ell^{\infty}(\FTV)\to \RB$ that is Hadamard differentiable at $\eta^\pi(s)$.
\end{theorem}

Theorem~\ref{Theorem_Hadamard_Inference} follows directly from our asymptotic results described in Theorem~\ref{Theorem_weak_convergence_empirical_process} and functional delta method (Theorem 20.8 in \cite{van2000asymptotic}).
Since the derivative $\phi^\prime$ is continuous, the plug-in approach is still valid for estimating $z_\phi$.
\begin{proposition}\label{Proposition_plug_in_Hadamard}
Whenever the Hadamard derivative $\phi^\prime$ is properly defined, let
    \begin{equation*}
    \begin{aligned}
    \hat z_\phi(p)&:=\inf\brc{t\ \Bigg\vert\ \PB\prn{\phi^\prime_{\hat\eta_n^\pi(s)}\prn{\brk{(\gI-\what\gT^\pi)^{-1}\what \GB^\pi} (s)}\leq t}\geq p},
    % \hat z_2(p)&:=\inf\brc{t\ \Bigg\vert\ \PB\prn{\int_0^{\frac{1}{1-\gamma}}\brk{(\gI-\what\gT^\pi)^{-1}\what \GB^\pi} (s)\ind\brc{\cdot\leq x}dx\leq t}\geq p}.
    \end{aligned}
    \end{equation*}
     Then $\hat z_\phi(p)\cp z_\phi(p)$ if $z_\phi(\cdot)$ is continuous at $p$.
\end{proposition}
The proof of the proposition above is nearly identical to those of Proposition~\ref{Proposition_plug_in_W1} and Proposition~\ref{Proposition_plug_in_KS}.

We demonstrate the use of Theorem~\ref{Theorem_Hadamard_Inference} with three concrete examples.
\begin{example}[The $r$-th moment of returns]\label{Example_inference_moments}
    We first consider a simple example of statistical inference for $r$th moments of returns.
    Let $\phi_r(\mu):=\EB_{X\sim\mu}\prn{X^r}$, where $\mu$ is a signed measure supported on $\brk{0,\frac{1}{1-\gamma}}$.
    It can be easily verified that $\phi(\cdot)\colon \ell^\infty(\FW)\to\RB$ is Hadamard differentiable with the derivative $\phi^\prime_r(h)=\phi_r(h)=\EB_{X\sim h}\prn{X^r}$ for any signed measure $h$ supported on $\brk{0,\frac{1}{1-\gamma}}$.
    Then by Theorem~\ref{Theorem_Hadamard_Inference} we have
    \begin{equation*}
        \sqrt{n}(\phi_r(\hat\eta_n^\pi(s))-\phi_r(\eta^\pi(s)))\cweak \phi_r\prn{\brk{\prn{\gI-\gT^\pi}^{-1} \wtilde\GB^\pi}(s)},  
    \end{equation*}
    according to which we may perform statistical inference for $\phi(\eta^\pi(s))$.
    When $r=1$, we have
    \begin{equation*}
        \sqrt{n}(\what V^\pi(s)-V^\pi(s))\cweak \brk{\prn{I-\gamma P^\pi}^{-1} \wtilde G^\pi}(s).  
    \end{equation*}
    Here $P^\pi\in\RB^{\gS\times\gS}$ is the transition matrix under policy $\pi$, $\what V^\pi$ and $V^\pi$ are the estimated value function and ground-truth value function.
    $\wtilde G^\pi$ is defined as
    % \begin{equation*}
    %     \wtilde G^\pi(s):=\sum_{a\in\gA}\pi(a\mid s) \sum_{s^\prime\in\gS} Z_{s,a,s^\prime}\prn{\EB_{R\sim\gP_R(\cdot\mid s,a)}R+V^\pi(s^\prime)},
    % \end{equation*}
    \begin{equation*}
        \wtilde G^\pi(s):=\sum_{a\in\gA}\pi(a\mid s) \sum_{s^\prime\in\gS} Z_{s,a,s^\prime}V^\pi(s^\prime),
    \end{equation*}
    where $Z$ is a gaussian vector as defined in Theorem~\ref{Theorem_weak_convergence_empirical_process}.
    This recovers the results of limiting distributions of the errors of model-based policy evaluations when a generative model is available.
    Another simple corollary is the limiting distribution of the variance of returns.
    Concretely, we have
    \begin{equation*}
        \sqrt{n}(\var_{X\sim \hat\eta^\pi_n(s)}\prn{X}-\var_{X\sim \eta^\pi(s)}\prn{X})\cweak \phi_2\prn{\brk{\prn{\gI-\gT^\pi}^{-1} \wtilde\GB^\pi}(s)}-2\phi_1\prn{\brk{\prn{\gI-\gT^\pi}^{-1} \wtilde\GB^\pi}(s)}\phi_1\prn{\eta^\pi(s)}.  
    \end{equation*}
\end{example}

\begin{example}[Quantiles of returns]
    We next consider statistical inference for quantiles of returns when Assumption~\ref{Assumption_reward_bounded_density} holds,.
    Let $\phi_p(\mu):=\inf\brc{t\mid \mu\ind_{(-\infty,t]}\geq p}$ be the $p$-quantile of probability distribution $\mu$.
    Lemma~\ref{Lemma_bounded_density_of_return} in the proof of~\ref{Proposition_value_iteration_KS} indicates $\eta^\pi(s)$ must have bounded density.
    Hence we have $\phi_p$ is Hadamard differentiable  tangentially to $C\brk{0,\frac{1}{1-\gamma}}$ by Lemma 21.4 in \cite{van2000asymptotic}.
    And the derivative $\phi^\prime_{p}(\eta^\pi(s))$ is the map ${h\mapsto -\frac{h(\phi_p(\eta^\pi(s)))}{g(\phi_p(\eta^\pi(s)))}}$, where $g$ is the density of $\eta^\pi(s)$.
    Therefore, the cumulative distribution function of $\brk{\prn{\gI-\gT^\pi}^{-1}\wtilde \GB^\pi}(s)$ is in $C\brk{0,\frac{1}{1-\gamma}}$ almost surely, we have
    \begin{equation*}
        \sqrt{n}(\phi_p(\hat\eta_n^\pi(s))-\phi_p(\eta^\pi(s)))\cweak -\frac{\brk{\prn{\gI-\gT^\pi}^{-1}\wtilde \GB^\pi}(s)\ind_{(-\infty,\phi_p(\eta^\pi(s))]}}{g(\phi_p(\eta^\pi(s)))},
    \end{equation*}
     which may lead to asymptotically valid inferential procedures for quantiles of returns.
\end{example}

\begin{example}[Uniform Advantage]\label{Example_inference_advantage}
Policy improvement \cite{sutton2004} is a key ingredient of many reinforcement learning algorithms.
The goal is to find a new policy $\pi$ such that the advantage function $V^\pi(s_0)-V^{\pi_0}(s_0)\geq 0$ where $\pi_0$ is a given baseline policy.
Here we propose a new notion of policy improvement called (near)-uniform policy improvement.
Specifically, the aim is to find a new policy $\pi$ such that the uniform advantage $\phi(\eta^\pi(s_0),\eta^{\pi_{0}}(s_0)):=\PB(G^\pi(s_0)\geq G^{\pi_0}(s_0))$ is above some threshold.

A natural estimator of $\phi(\eta^\pi(s_0),\eta^{\pi_{0}}(s_0))$ is $\phi(\hat\eta_n^\pi(s_0),\hat\eta_n^{\pi_{0}}(s_0))$.
For technical convenience, we assume $\hat\eta_n^\pi(s_0)$
and $\hat\eta_n^{\pi_{0}}(s_0)$ are estimated using data splitting technique.
From Lemma 20.10 in \cite{van2000asymptotic}, $\phi$ is Hadamard differentiable tangentially to $C\brk{0,1/(1-\gamma)}$.
For $h_1,h_2\in C\brk{0,1/(1-\gamma)}$, the derivative $\phi^\prime(\eta^\pi(s_0),\eta^{\pi_{0}}(s_0))$ is $(h_1,h_2)\mapsto h_2-\eta^\pi(s_0)h_{2}+\eta^{\pi_0}(s_0)h_1$.
When Assumption~\ref{Assumption_reward_bounded_density} is true, 
\begin{equation*}
    \sqrt{n}\brk{(\hat\eta_n^\pi(s_0),\hat\eta_n^{\pi_{0}}(s_0))-(\eta^\pi(s_0),\eta^{\pi_{0}}(s_0))}\cweak\prn{\brk{\prn{\gI-\gT^\pi}^{-1}\wtilde \GB^\pi}(s_0), \brk{\prn{\gI-\gT^\pi}^{-1}\wtilde \GB^{\pi_0}}(s)}
\end{equation*}
and the cumulative distribution functions of $\brk{\prn{\gI-\gT^\pi}^{-1}\wtilde \GB^\pi}(s_0)$ and $\brk{\prn{\gI-\gT^{\pi_0}}^{-1}\wtilde \GB^{\pi_0}}(s_0)$ (denoted as $F$ and $F_0$) are in $C\brk{0,1/(1-\gamma)}$ almost surely.
Thus, we have
\begin{equation*}
    \sqrt{n}\brk{\phi(\hat\eta_n^\pi(s_0),\hat\eta_n^{\pi_{0}}(s_0))-\phi(\eta^\pi(s_0),\eta^{\pi_{0}}(s_0))}\cweak\eta^\pi(s_0)F_0-\eta^{\pi_0}(s_0)F.
\end{equation*}

\end{example}

\section{Numerical Simulations}\label{Section_numerical}
In this section we conduct numerical simulations to validate our theoretical findings as well as the proposed inferential procedures.
All of the numerical simulations are conducted on a desktop computer with a single TITAN RTX GPU.
The code is available in \url{https://github.com/zhangliangyu32/EstimationAndInferenceDistributionalRL}.

\subsection{Implementations}
To make computations tractable, we confine the return distributions to the class of categorical distributions.
A vector of categorical distributions is defined as ${\eta=\prn{\eta(s_1),\dots,\eta\prn{s_{|\gS|}}}}$, where ${\eta(s):=\sum_{k=0}^K w_k\delta_{x_k}}$, with weights $\sum_{k=0}^K w_k=1$ and particles $x_k:=\frac{k}{(K+1)(1-\gamma)}$.
We set $K=1000$, which is large enough to make the categorical class rich enough and good approximations of continuous return distributions.

The categorical distributions can be updated with a categorical version of distributional dynamical programming \cite{bellemare2017distributional}, which is also a good approximation of the original version of distributional dynamic programming considered in our paper when $K$ is large.
Throughout our simulation studies, the ground-truth return distributions $\eta^\pi$ are obtained via a sufficiently large number of iterations of distributional dynamic programming with the ground-truth distributional Bellman operator $\gT^\pi$.
The estimated return distributions $\hat\eta^\pi_n$ are obtained by the same procedure except for the ground-truth distributional Bellman operator $\gT^\pi$ is replaced by the estimated distributional Bellman operator $\what\gT_n^\pi$.

Another issue that may cause computational intractability is that in our inferential procedures, we must explicitly form the operator $\prn{\gI-\gT^\pi}^{-1}$.
We instead use a truncated Neumann series to approximate $\prn{\gI-\gT^\pi}^{-1}$, that is, $\prn{\gI-\gT^\pi}^{-1}\approx\sum_{j=0}^J \prn{\gT^\pi}^j$ with $J$ sufficiently large.
In summary, by these approximation techniques, our implementations achieve computational tractability while ensuring that the approximation error is negligible compared to the statistical error that is of primary interest.

\subsection{Linear Convergence of Distributional Dynamic Programming}
We first verify the results in Proposition~\ref{Proposition_value_iteration_Wp}, Proposition~\ref{Proposition_value_iteration_KS} and Proposition~\ref{Proposition_value_iteration_TV}.
We perform the simulations in randomly generated tabular MDPs with $|\gS|=5$, $|\gA|=2$ and $\gamma=0.9$.
WLOG, we always use the first state $s_1$ as the initial state.
The reward distribution is chosen to be truncated Gaussians.
Specifically, $\gP_{R}(\cdot\mid s,a)$ is set to be $\normal(l_{s,a},0.1)$ truncated to $[0,1]$, with the location parameter $l_{s,a}$ randomly determined.
The dataset we use to form the estimator $\hat\eta_n^\pi$ is obtained via a generative model with $n=10000$.
We first perform distributional DP for $N$ iterations, with $N$ sufficiently large, and use $\eta^{(N)}$ as a proxy of the estimator $\hat\eta_n^\pi$.
In Figure~\ref{fig:DPConverge} we depict how $d(\eta^{(t)}, \hat\eta_n^\pi)$ would change as $t$ increases from $0$ to $N/2$.
Here the probability metric $d$ can be $W_1$ metric, the KS metric, or the TV metric.
We may find that distributional DP does exhibit linear convergence when measured by the $W_1$ metric, the KS metric, and the TV metric.

\subsection{Finite-sample Convergence Performance}
We investigate the finite-sample convergence performances of empirical distributional dynamic programming and verify our non-asymptotic results.
The simulation environments remain the same as in the last section.
And we try more choices of $\gamma$.
Specifically, we try $\gamma\in\brc{0.7,0.8,0.9,0.97}$.
And we also try more choices of $n$.
Specifically, we try $n\in\brc{10, 100, 1000, 10000}$.
We repeat the estimation process for $100$ times and report the averaged errors.
The numerical results are displayed in figures listed as follows.
\begin{figure}[!htbp]
     \centering
     \begin{minipage}[t]{0.32\textwidth}
         \centering
         \includegraphics[width=\textwidth]{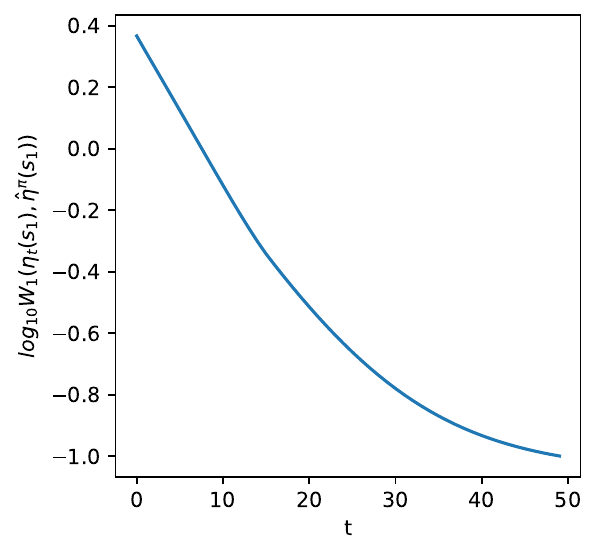}
         % \caption{$\gamma=0.7$}
         % \label{fig:DPProcessWDisGamma=07}
     \end{minipage}
     % \hfill
     \begin{minipage}[t]{0.32\textwidth}
         \centering
         \includegraphics[width=\textwidth]{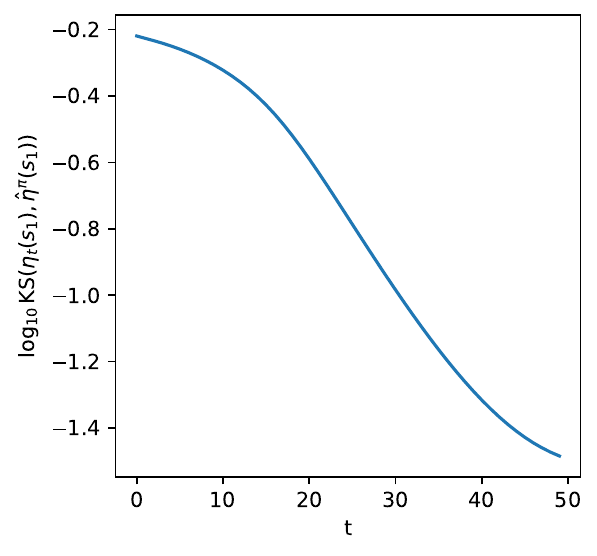}
         % \caption{$\gamma=0.8$}
         % \label{fig:DPProcessWDisGamma=08}
     \end{minipage}
     % \hfill
     \begin{minipage}[t]{0.32\textwidth}
         \centering
         \includegraphics[width=\textwidth]{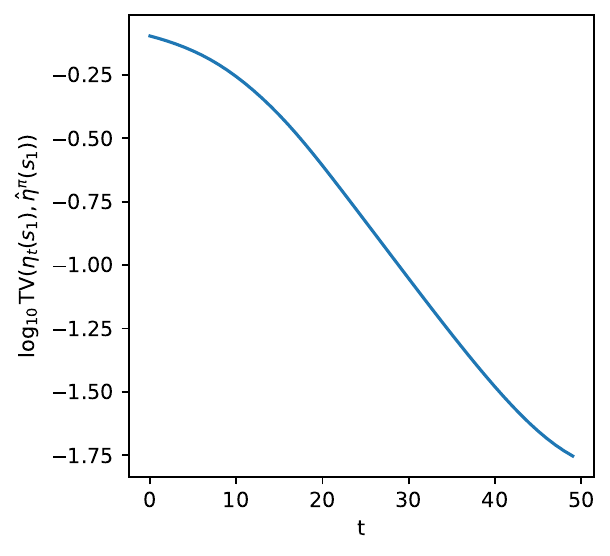}
         % \label{fig:DPProcessWDisGamma=09}
     \end{minipage}
     % \hfill
    \caption{Convergence of $\log W_1(\eta^{(t)}(s_1),\hat\eta^\pi(s_1))$, $\log\KS(\eta^{(t)}(s_1),\hat\eta^\pi(s_1))$, and $\log\TV(\eta^{(t)}(s_1),\hat\eta^\pi(s_1))$ with sample size $n=10000$ and $\gamma=0.9$. $t$ is the iteration number.}
        \label{fig:DPConverge}
\end{figure}

\begin{figure}[!htbp]
     \centering
     \begin{minipage}[t]{0.24\textwidth}
         \centering
         \includegraphics[width=\textwidth]{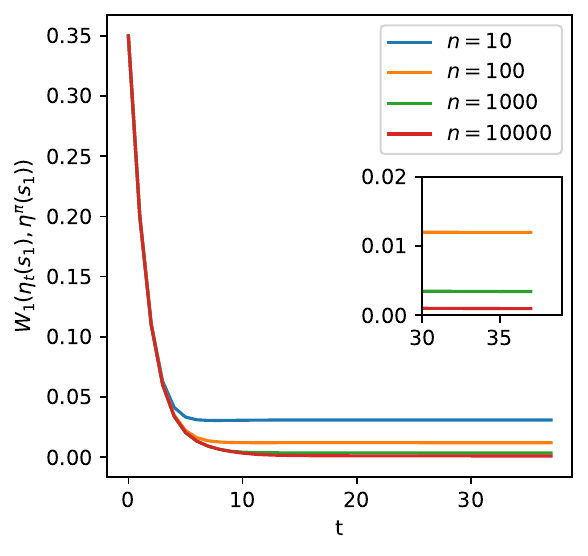}
         % \caption{$\gamma=0.7$}
         % \label{fig:DPProcessWDisGamma=07}
     \end{minipage}
     % \hfill
     \begin{minipage}[t]{0.24\textwidth}
         \centering
         \includegraphics[width=\textwidth]{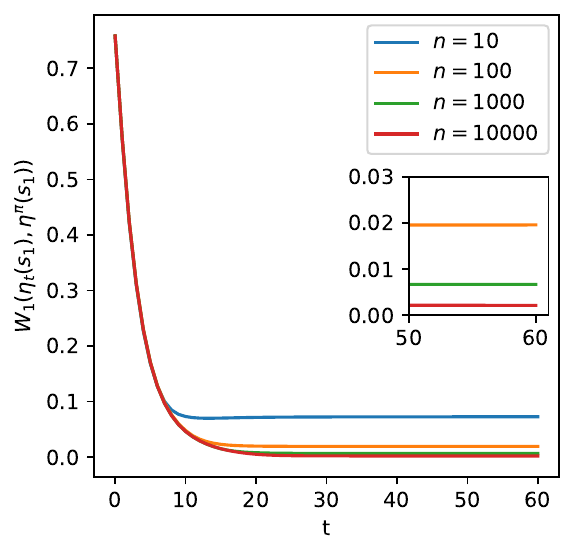}
         % \caption{$\gamma=0.8$}
         % \label{fig:DPProcessWDisGamma=08}
     \end{minipage}
     % \hfill
     \begin{minipage}[t]{0.24\textwidth}
         \centering
         \includegraphics[width=\textwidth]{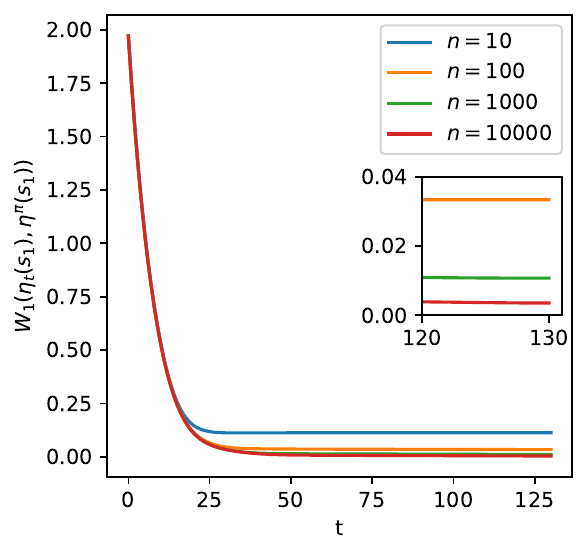}
         % \label{fig:DPProcessWDisGamma=09}
     \end{minipage}
     % \hfill
    \begin{minipage}[t]{0.24\textwidth}
         \centering
         \includegraphics[width=\textwidth]{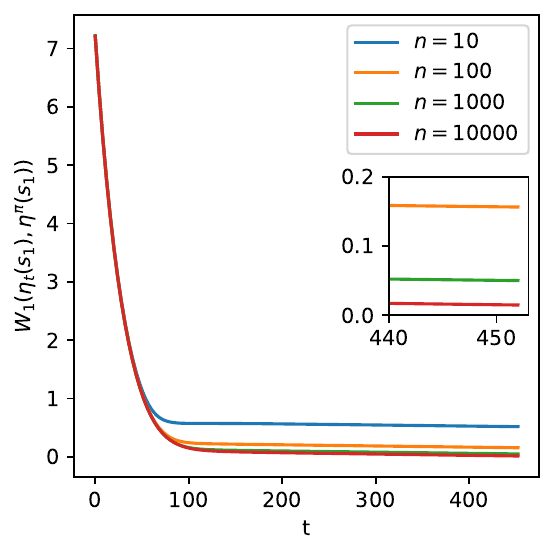}
         % \label{fig:DPProcessWDisGamma=0.97}
     \end{minipage}
        \caption{Two-phase convergence of $W_1(\eta^{(t)}(s_1),\eta^\pi(s_1))$ with different sample sizes. $t$ is the iteration number. From left to right: $\gamma=0.7$; $\gamma=0.8$; $\gamma=0.9$; $\gamma=0.97$.}
        \label{fig:DPProcessWDis}
\end{figure}

\begin{figure}[!htbp]
     \centering
     \begin{minipage}[t]{0.24\textwidth}
         \centering
         \includegraphics[width=\textwidth]{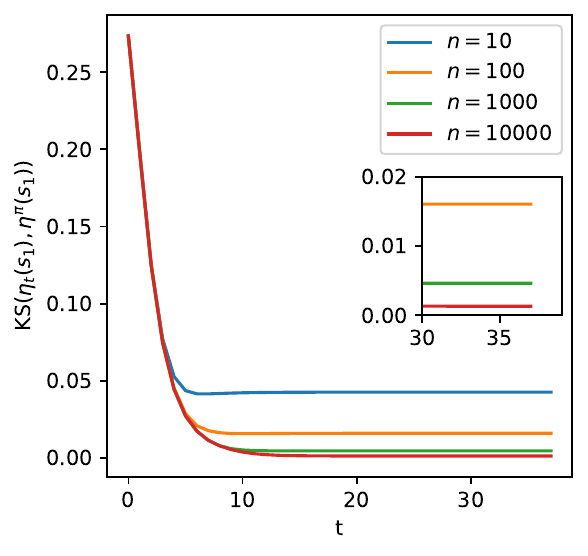}
         % \caption{$\gamma=0.7$}
         % \label{fig:DPProcessWDisGamma=07}
     \end{minipage}
     % \hfill
     \begin{minipage}[t]{0.24\textwidth}
         \centering
         \includegraphics[width=\textwidth]{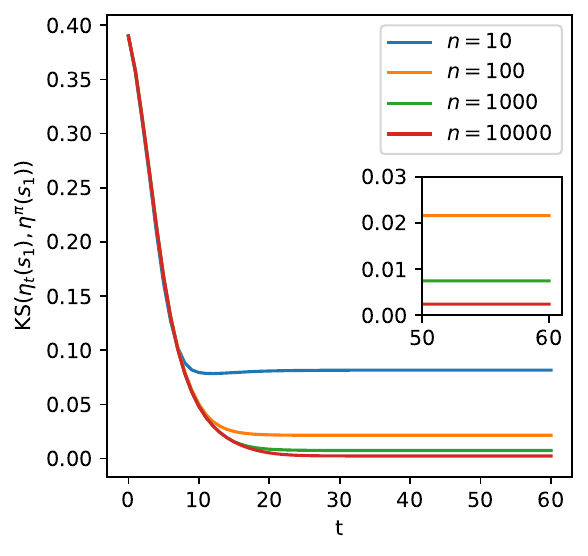}
         % \caption{$\gamma=0.8$}
         % \label{fig:DPProcessWDisGamma=08}
     \end{minipage}
     % \hfill
     \begin{minipage}[t]{0.24\textwidth}
         \centering
         \includegraphics[width=\textwidth]{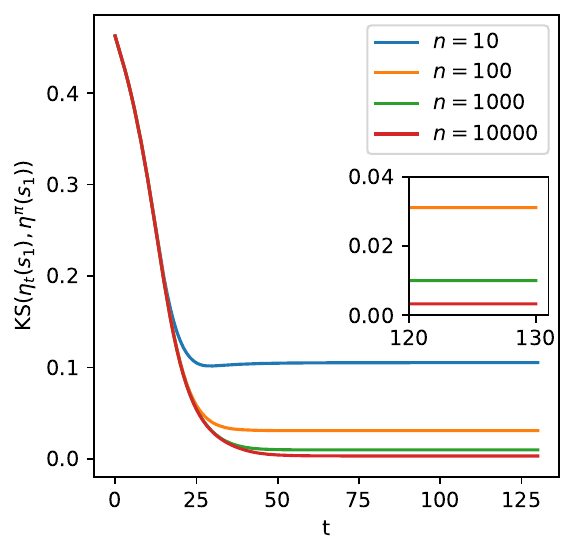}
         % \label{fig:DPProcessWDisGamma=09}
     \end{minipage}
     % \hfill
    \begin{minipage}[t]{0.24\textwidth}
         \centering
         \includegraphics[width=\textwidth]{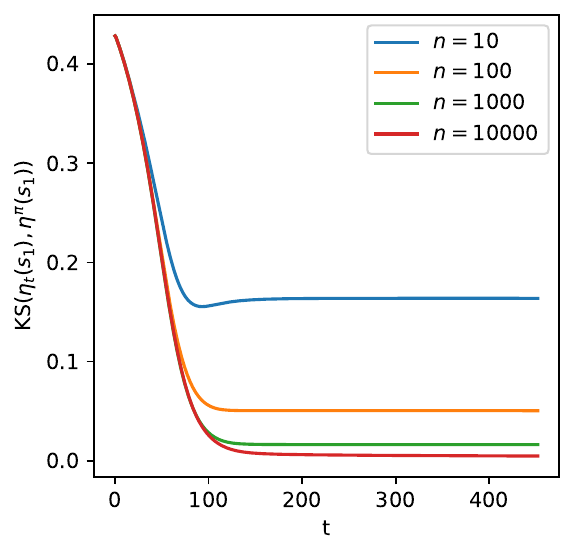}
         % \label{fig:DPProcessWDisGamma=0.97}
     \end{minipage}
        \caption{Two-phase convergence of $\KS(\eta^{(t)}(s_1),\eta^\pi(s_1))$ with different sample seize. $t$ is the iteration number. From left to right: $\gamma=0.7$; $\gamma=0.8$; $\gamma=0.9$; $\gamma=0.97$.}
        \label{fig:DPProcessKSDis}
\end{figure}
\begin{figure}[!htbp]
     \centering
     \begin{minipage}[t]{0.24\textwidth}
         \centering
         \includegraphics[width=\textwidth]{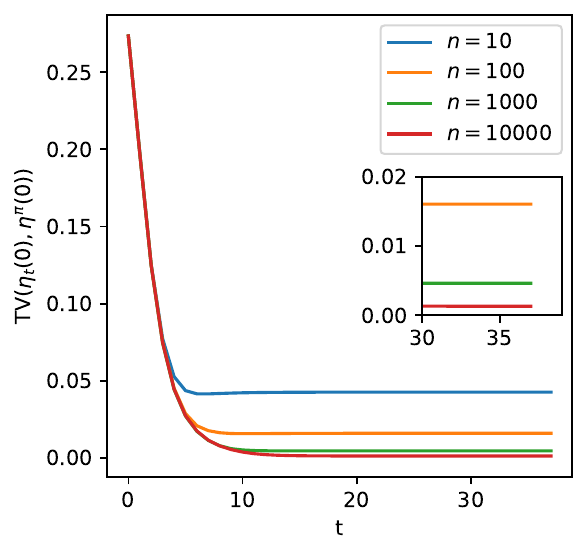}
         % \caption{$\gamma=0.7$}
         % \label{fig:DPProcessWDisGamma=07}
     \end{minipage}
     % \hfill
     \begin{minipage}[t]{0.24\textwidth}
         \centering
         \includegraphics[width=\textwidth]{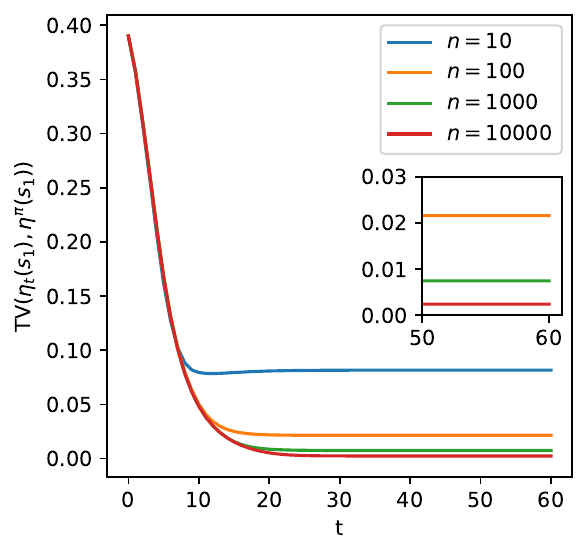}
         % \caption{$\gamma=0.8$}
         % \label{fig:DPProcessWDisGamma=08}
     \end{minipage}
     % \hfill
     \begin{minipage}[t]{0.24\textwidth}
         \centering
         \includegraphics[width=\textwidth]{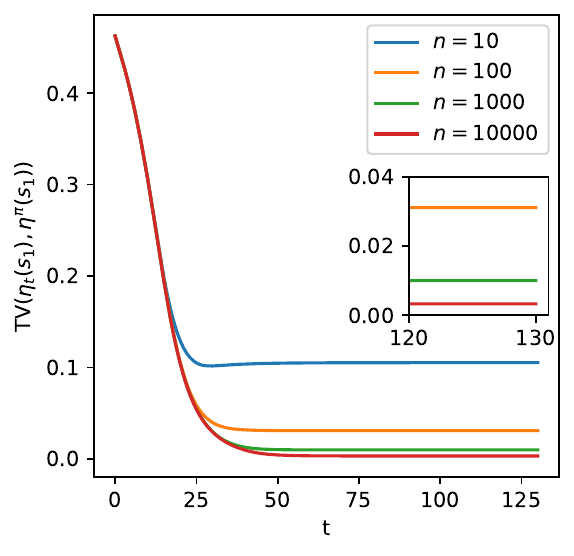}
         % \label{fig:DPProcessWDisGamma=09}
     \end{minipage}
     % \hfill
    \begin{minipage}[t]{0.24\textwidth}
         \centering
         \includegraphics[width=\textwidth]{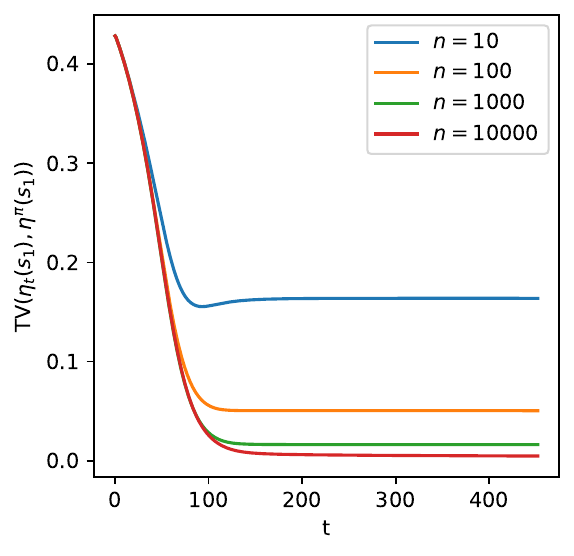}
         % \label{fig:DPProcessWDisGamma=0.97}
     \end{minipage}
        \caption{Two-phase convergence of $\TV(\eta^{(t)}(s_1),\eta^\pi(s_1))$ with different sample seize. $t$ is the iteration number. From left to right: $\gamma=0.7$; $\gamma=0.8$; $\gamma=0.9$; $\gamma=0.97$.}
        \label{fig:DPProcessTVDis}
\end{figure}

In Figure~\ref{fig:DPProcessWDis}-\ref{fig:DPProcessTVDis}, we show the convergence performance of empirical distributional dynamic programming measured by $W_1$ metric, \KS\ distance and \TV\ distance, respectively.
Note that in all cases the convergence consists of two phases.
In the first phase, the dynamic programming algorithm does not converge and we may observe a linear convergence rate.
In the second phase, the error terms are dominated by the statistical error, \ie~$W_1(\hat\eta_n^\pi(s_1),\eta^\pi(s_1))$, $\KS(\hat\eta_n^\pi(s_1),\eta^\pi(s_1))$ or $\TV(\hat\eta_n^\pi(s_1),\eta^\pi(s_1))$, which exhibits strong correlations with $n$ and $1/(1-\gamma)$.

\begin{figure}[!htbp]
     \centering
     \begin{minipage}[t]{0.24\textwidth}
         \centering
         \includegraphics[width=\textwidth]{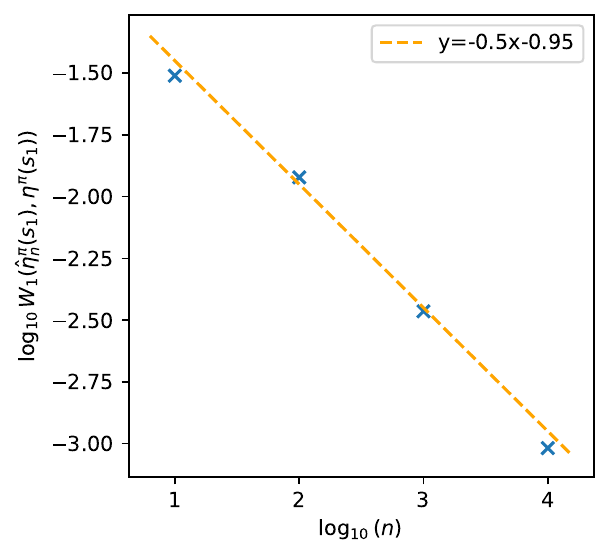}
         % \caption{$\gamma=0.7$}
         % \label{fig:DPProcessWDisGamma=07}
     \end{minipage}
     % \hfill
     \begin{minipage}[t]{0.24\textwidth}
         \centering
         \includegraphics[width=\textwidth]{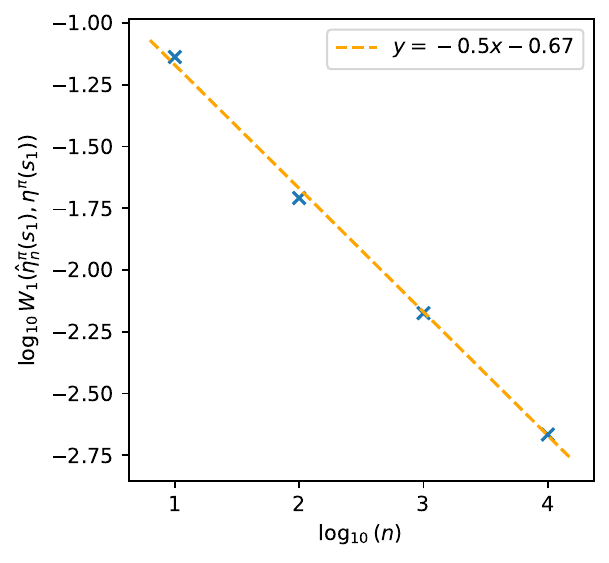}
         % \caption{$\gamma=0.8$}
         % \label{fig:DPProcessWDisGamma=08}
     \end{minipage}
     % \hfill
     \begin{minipage}[t]{0.24\textwidth}
         \centering
         \includegraphics[width=\textwidth]{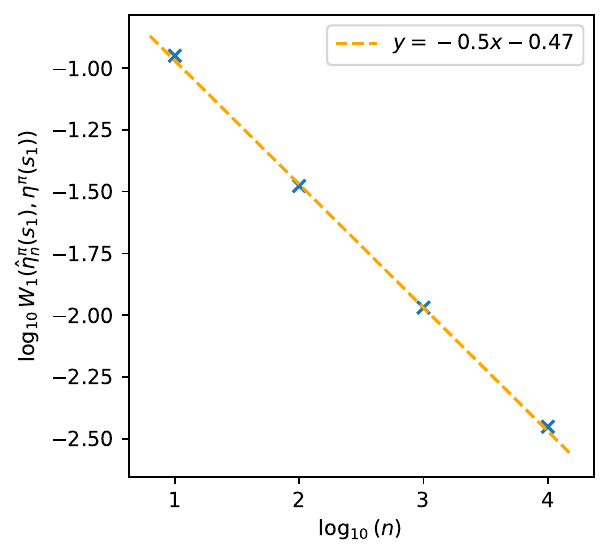}
         % \label{fig:DPProcessWDisGamma=09}
     \end{minipage}
     % \hfill
    \begin{minipage}[t]{0.24\textwidth}
         \centering
         \includegraphics[width=\textwidth]{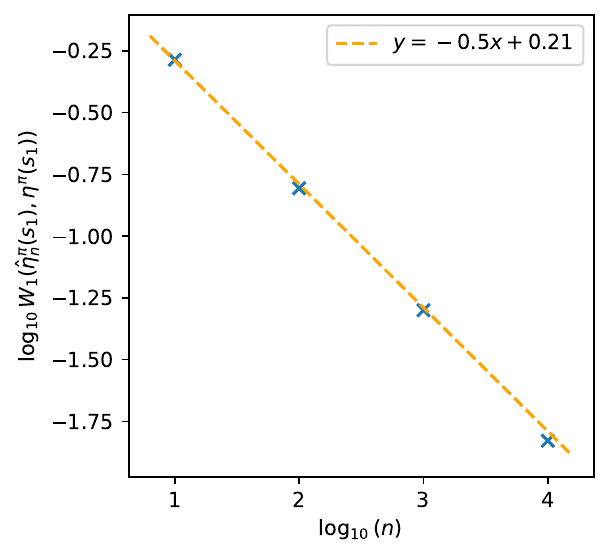}
         % \label{fig:DPProcessWDisGamma=0.97}
     \end{minipage}
        \caption{The statistical error $W_1(\hat\eta_n^\pi(s_1),\eta^\pi(s_1))$ with different sample size. From left to right: $\gamma=0.7$; $\gamma=0.8$; $\gamma=0.9$; $\gamma=0.97$.}
        \label{fig:WDisVSn}
\end{figure}

\begin{figure}[!htbp]
     \centering
     \begin{minipage}[t]{0.24\textwidth}
         \centering
         \includegraphics[width=\textwidth]{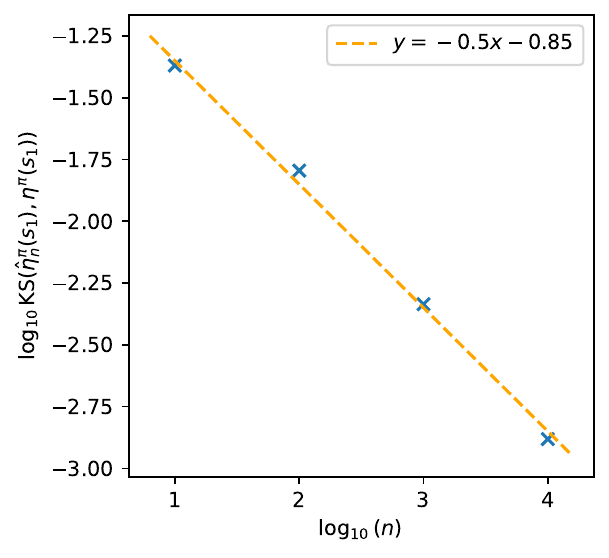}
         % \caption{$\gamma=0.7$}
         % \label{fig:DPProcessWDisGamma=07}
     \end{minipage}
     % \hfill
     \begin{minipage}[t]{0.24\textwidth}
         \centering
         \includegraphics[width=\textwidth]{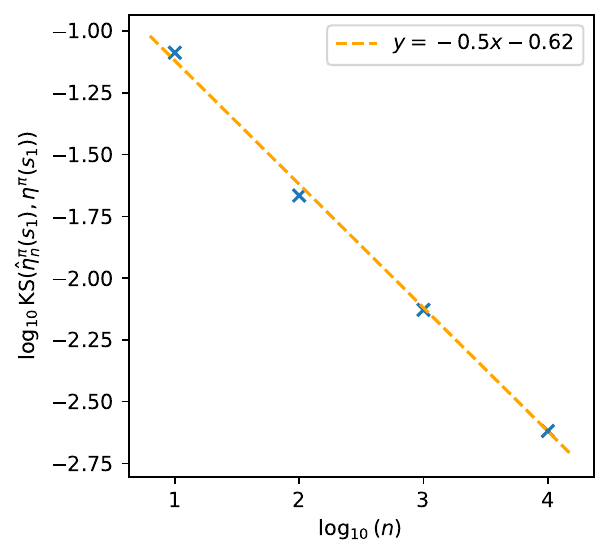}
         % \caption{$\gamma=0.8$}
         % \label{fig:DPProcessWDisGamma=08}
     \end{minipage}
     % \hfill
     \begin{minipage}[t]{0.24\textwidth}
         \centering
         \includegraphics[width=\textwidth]{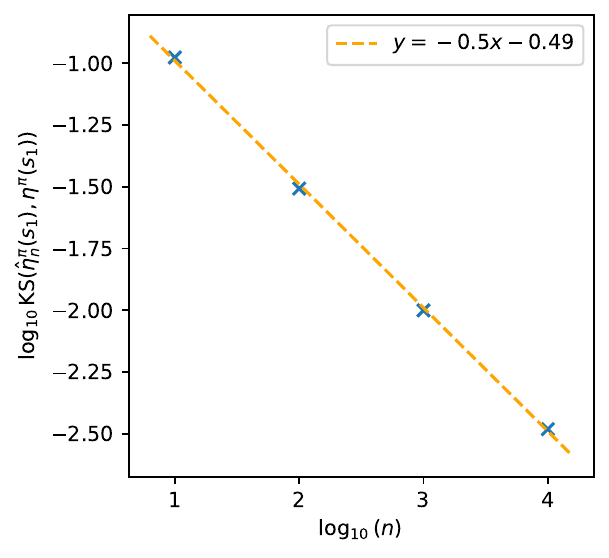}
         % \label{fig:DPProcessWDisGamma=09}
     \end{minipage}
     % \hfill
    \begin{minipage}[t]{0.24\textwidth}
         \centering
         \includegraphics[width=\textwidth]{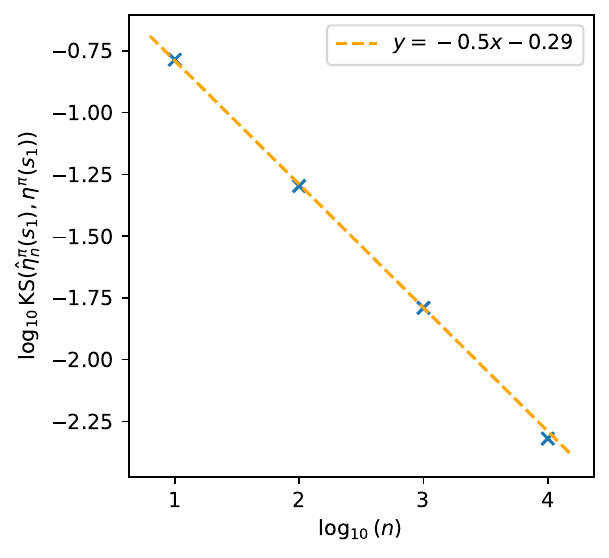}
         % \label{fig:DPProcessWDisGamma=0.97}
     \end{minipage}
        \caption{The statistical error $\KS(\hat\eta_n^\pi(s_1),\eta^\pi(s_1))$ with different sample size. From left to right: $\gamma=0.7$; $\gamma=0.8$; $\gamma=0.9$; $\gamma=0.97$.}
        \label{fig:KSDisVSn}
\end{figure}

\begin{figure}[!htbp]
     \centering
     \begin{minipage}[t]{0.24\textwidth}
         \centering
         \includegraphics[width=\textwidth]{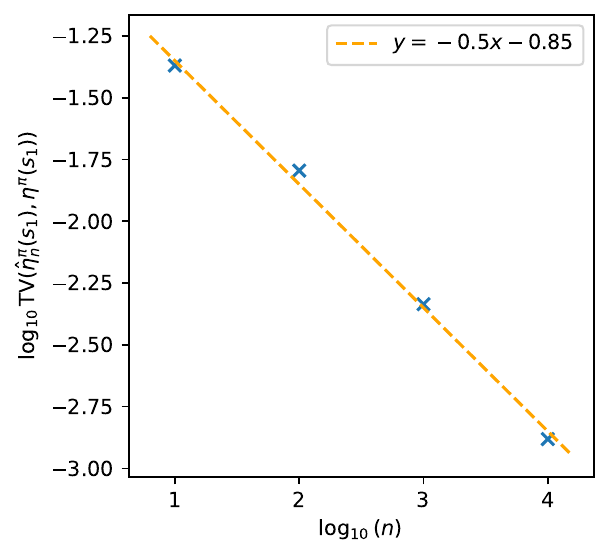}
         % \caption{$\gamma=0.7$}
         % \label{fig:DPProcessWDisGamma=07}
     \end{minipage}
     % \hfill
     \begin{minipage}[t]{0.24\textwidth}
         \centering
         \includegraphics[width=\textwidth]{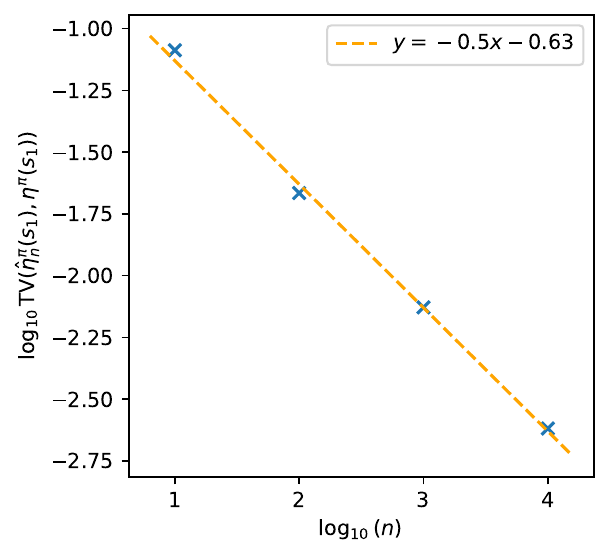}
         % \caption{$\gamma=0.8$}
         % \label{fig:DPProcessWDisGamma=08}
     \end{minipage}
     % \hfill
     \begin{minipage}[t]{0.24\textwidth}
         \centering
         \includegraphics[width=\textwidth]{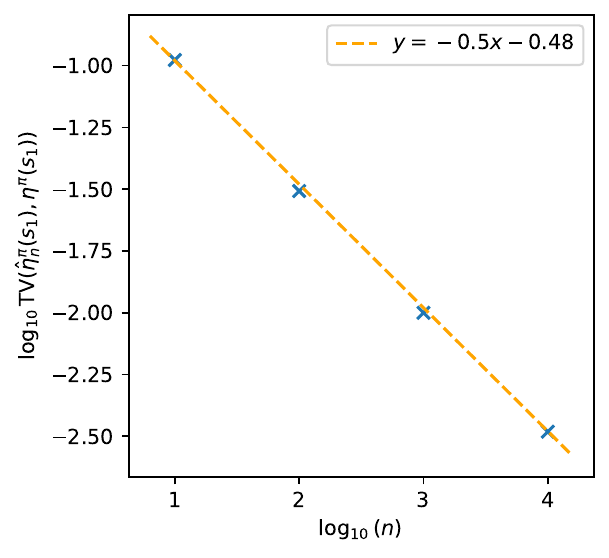}
         % \label{fig:DPProcessWDisGamma=09}
     \end{minipage}
     % \hfill
    \begin{minipage}[t]{0.24\textwidth}
         \centering
         \includegraphics[width=\textwidth]{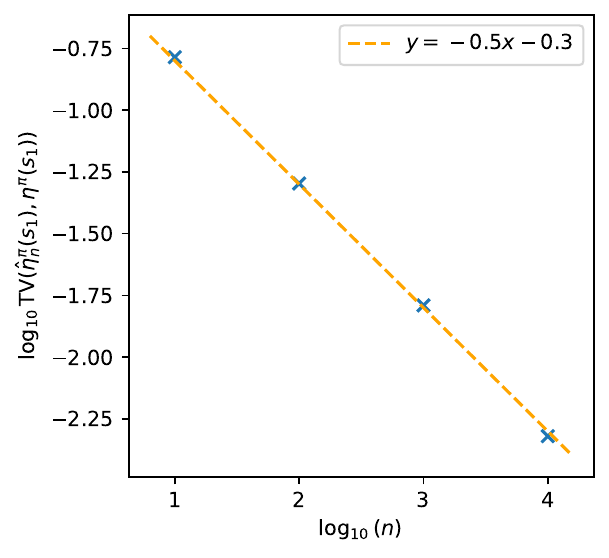}
         % \label{fig:DPProcessWDisGamma=0.97}
     \end{minipage}
        \caption{The statistical error $\TV(\hat\eta_n^\pi(s_1),\eta^\pi(s_1))$ with different sample size. From left to right: $\gamma=0.7$; $\gamma=0.8$; $\gamma=0.9$; $\gamma=0.97$.}
        \label{fig:TVDisVSn}
\end{figure}

We also try to examine how the error terms $W_1(\hat\eta_n^\pi(s_1),\eta^\pi(s_1))$, $\KS(\hat\eta_n^\pi(s_1),\eta^\pi(s_1))$ or $\TV(\hat\eta_n^\pi(s_1),\eta^\pi(s_1))$ change as $n$ scales up.
In Figure~\ref{fig:WDisVSn}-\ref{fig:TVDisVSn}, we may verify that for all cases, the convergence rates are indeed of the typical $n^{-\frac{1}{2}}$ order as described in Theorem~\ref{Theorem_bound_of_Wdist}, $\ref{Theorem_bound_of_KS}$ and \ref{Theorem_bound_of_TV}.

\subsection{Validity of Inferential Procedures}
We also perform numerical simulations to validate the inferential procedures proposed in Section~\ref{Section_inference}.
The environment is exactly the same as that of the previous section with $\gamma$ fixed as $0.9$.
The confidence sets are constructed using plug-in approaches.
The nominal coverage probability is set to be $0.95$, the quantiles of the estimated limiting distributions are computed using Monte Carlo methods.
Here our implementations of Monte Carlos are fully vectorized to further improve computational efficiency.
We repeat our inferential procedures for $1000$ times and report the empirical coverage rate and averaged radius of confidence sets.
The results are presented in the following tables.
We observe the empirical coverage rates approach the nominal confidence level and the radius of confidence sets decreases as $n$ increases in all cases.
\begin{table}[t!]
\centering
    \begin{tabular}{|c|c|c|c|c|c|c|}
    \hline
    Type of Confidence Sets & \multicolumn{2}{|c|}{$W_1$ ball} & \multicolumn{2}{|c|}{$\KS$ ball} & \multicolumn{2}{|c|}{$\TV$ ball}\\
    \hline
    $n$ & CR & CSR$\pm$std & CR & CSR$\pm$std & CR & CSR$\pm$std\\
    \hline
    5 & 0.918 & $0.3467\pm 0.0719$ & 0.939 & $0.3302\pm 0.0538$ & 0.934 & $0.3310\pm0.0530$ \\
    \hline
    10 & 0.945 & $0.2528\pm 0.0366$ & 0.934 & $0.2364\pm 0.0250$ & 0.956 & $0.2367\pm0.0246$ \\
    \hline
    100 & 0.941 & $0.0804\pm 0.0041$ & 0.956 & $0.0759\pm 0.0032$ & 0.944 & $0.0761\pm0.0030$ \\
    \hline
    1000 & 0.945 & $0.0255\pm 0.0008$ & 0.951 & $0.0241\pm 0.0007$ & 0.950 & $0.0241\pm0.0007$ \\
    \hline
    \end{tabular}
    \caption{Coverage rate (CR) and confidence set radius (CSR) of our proposed non-parametric confidence sets for $\eta^\pi(s_1)$ under different choices of $n$.}
    \label{tab:nonparamInference}
\end{table}

\begin{table}[t!]
\centering
    \begin{tabular}{|c|c|c|c|c|c|c|}
    \hline
    Functionals of Interest & \multicolumn{2}{|c|}{variance} & \multicolumn{2}{|c|}{0.1 quantile} & \multicolumn{2}{|c|}{0.9 quantile}\\
    \hline
    $n$ & CR & CSR$\pm$std & CR & CSR$\pm$std & CR & CSR$\pm$std\\
    \hline
    5 & 0.928 & $0.0627\pm 0.0202$ & 0.917 & $0.4020\pm 0.0929$ & 0.916 & $0.3210\pm0.0660$ \\
    \hline
    10 & 0.939 & $0.0442\pm 0.0101$ & 0.945 & $0.2949\pm 0.0436$ & 0.930 & $0.2288\pm0.0308$ \\
    \hline
    100 & 0.959 & $0.0137\pm 0.0010$ & 0.949 & $0.0912\pm 0.0050$ & 0.946 & $0.0733\pm0.0036$ \\
    \hline
    1000 & 0.946 & $0.0043\pm 0.0002$ & 0.935 & $0.0289\pm 0.0010$ & 0.952 & $0.0232\pm0.0008$ \\
    \hline
    \end{tabular}
    \caption{Coverage rate (CR) and confidence set radius (CSR) of our proposed confidence intervals for different Hardamard differentiable statistical functionals of $\eta^\pi(s_1)$ under different choices of $n$.}
    \label{tab:functionalsInference}
\end{table}

\section{Discussions}\label{Section_discussion}
In this paper, we have analyzed the statistical performance of distributional reinforcement learning from both non-asymptotic and asymptotic perspectives.
We present non-asymptotical rates for $\sup_{s\in\gS}W_p(\hat\eta_n^\pi(s),\eta^\pi(s))$, $\sup_{s\in\gS}\KS(\hat\eta_n^\pi(s), \eta^\pi(s))$ and $\sup_{s\in\gS}TV(\hat\eta_n^\pi(s), \eta^\pi(s))$.
We also derive that given initial state $s$, the ``empirical process'' $\sqrt{n}(\hat\eta_n^\pi(s)-\eta^\pi(s))$ converge weakly to a Gaussian random element.
Based on our theoretical findings, we have devised inferential procedures for a wide class of statistical functionals of the return distribution. 
We hope our work can spur further research in the uncertainty quantification of reinforcement learning.

One future direction is whether we can close the gap between our sample complexity bound $\wtilde O\prn{\frac{1}{\varepsilon^2(1-\gamma)^4}}$ and the lower bound $\wtilde O\prn{\frac{1}{\varepsilon^2(1-\gamma)^3}}$.
We speculate that the minimax optimal sample complexity is indeed $O\prn{\frac{1}{\varepsilon^2(1-\gamma)^3}}$, and could be attained through more sophisticated analysis techniques.
Another interesting future direction is to develop non-asymptotic bounds as well as asymptotic results that are uniform for $\pi\in\Pi$, where $\Pi$ denotes a policy class of interest.
This might give rise to a wider range of inferential applications in reinforcement learning.

\bibliography{ref}
\bibliographystyle{abbrvnat}
\newpage

\appendix
\section{Notations and Some Basic Facts}\label{Appendix_notations}
Before presenting our proofs, we first define some notations.
For any signed Borel measure $\mu$, we define $\norm{\mu}_\gF:=\sup_{f\in\gF}\abs{\mu f}$, with $\gF$ being some function class $\gF$ supported on $\brk{0,\frac{1}{1-\gamma}}$.
We use $M_\gF$ to denote the vector space of signed Borel measures with finite $\norm{\cdot}_\gF$ norm and zero measure, formally, define $\gB_0$ as the Borel sets in $\brk{0,\frac{1}{1-\gamma}}$,
\begin{equation*}
    M_\gF^0:=\brc{\mu\text{ signed measure on }\prn{\brk{0,\frac{1}{1-\gamma}},\gB_0}\mid \mu\prn{\brk{0,\frac{1}{1-\gamma}}}=0,\norm{\mu}_\gF<\infty}.
\end{equation*}
Let $\ell^\infty(\gF)$ be the space of bounded real-valued functions on $\gF$.
$\ell^\infty(\gF)$ is a Banach space if we equip it with the supreme norm $\norm{\cdot}_\gF$.
Note that $(M_\gF^0,\norm{\cdot}_\gF)$ can be embedded into $(\ell^\infty(\gF),\norm{\cdot}_\gF)$ by mapping $\mu$ to the process $f\mapsto \mu f$.

Recall that the $W_1$ metric, $\KS$ distance, and the $\TV$ are all integral probability metrics.
Therefore, if $\mu\in M_{\gF}^0$ can be written as the difference of two probability distributions $\mu_+,\mu_-\in\Delta\prn{\brk{0,\frac{1}{1-\gamma}}}$, we have
\begin{itemize}
    \item $\norm{\mu}_{\FW}=W_1(\mu_+,\mu_-)$, where $\FW:=\brc{f\mid f \text{ is $1$-Lipschitz}}$.
    \item $\norm{\mu}_{\FKS}=\KS(\mu_+,\mu_-)$, where $\FKS:=\brc{\ind\brc{\cdot\leq z}\mid z\in\RB}$.
    \item $\norm{\mu}_{\FTV}=\TV(\mu_+,\mu_-)$, where $\FTV:=\brc{f\mid f \text{ is measuable and }\norm{f}_\infty\leq 1}$.
\end{itemize}

The distributional Bellman operator $\gT^\pi$ can be naturally extended to $\prn{M_{\gF}^0}^\gS$ without modifying its original definition.
Here, the product $\left(M_{\mathcal{F}}^0\right)^\mathcal{S}$ also constitutes a normed vector space, where the norm is defined as $\sup_{s\in\mathcal{S}}\norm{\cdot}_{\mathcal{F}}$.

\begin{proposition}\label{Proposition_extention_of_Bellman_operator}
    $\gT^\pi$ is a linear operator on $\prn{M_{\gF}^0}^\gS$.
\end{proposition}

The proof is straightforward by noting the linearity of the push-forward operation.

\section{Analysis of Theorem~\ref{Theorem_bound_of_Wdist}}\label{Appendix_analysis_W1}
Let 
\begin{equation*}
    M^0_{\FW}:=\brc{\mu\text{ signed measure on }\prn{\brk{0,\frac{1}{1-\gamma}},\gB_0}\mid\norm{\mu}_{\FW}<\infty,\mu\prn{\brk{0,\frac{1}{1-\gamma}}}=0},    
\end{equation*}
for any vector of probability measures $\mu,\nu\in\Delta\prn{\brk{0,\frac{1}{1-\gamma}}}^\gS$, we always have $\mu-\nu\in \prn{M_{\FW}^0}^\gS$.
Therefore, we have
\begin{equation*}
\begin{aligned}
    \hat\eta_n^\pi-\eta^\pi&=\what \gT_n^\pi\hat\eta_n^\pi-\gT^\pi\eta^\pi\\
    &=\what \gT_n^\pi\hat\eta_n^\pi-\what \gT_n^\pi\eta^\pi+\what \gT_n^\pi\eta^\pi-\gT^\pi\eta^\pi\\
    &=\what \gT_n^\pi(\hat\eta_n^\pi-\eta^\pi)+\prn{\what \gT_n^\pi-\gT^\pi}\eta^\pi.\\
\end{aligned}
\end{equation*}
Rearranging terms yields
\begin{equation*}
    \prn{\gI-\what\gT_n^\pi}(\hat\eta_n^\pi-\eta^\pi)=\prn{\what \gT_n^\pi-\gT^\pi}\eta^\pi.
\end{equation*}

We first investigate the concentration behavior of $\prn{\what \gT_n^\pi-\gT^\pi}\eta^\pi$.
Formally, we have
\begin{lemma}\label{Lemma_concentration_operator_W1}
    For any fixed policy $\pi$, we have
    \begin{equation*}
        \EB\sup_{s\in\gS}\norm{\brk{\prn{\what \gT_n^\pi-\gT^\pi}\eta^\pi}(s)}_{\FW}\leq \sqrt{\frac{9\log\abs{\gS}}{n(1-\gamma)^2}}.
    \end{equation*}
    And for any $\delta\in(0,1)$,
    \begin{equation*}
        \sup_{s\in\gS}\norm{\brk{\prn{\what \gT_n^\pi-\gT^\pi}\eta^\pi}(s)}_{\FW}\leq \frac{\sqrt{9\log\abs{\gS}}+\sqrt{\log(1/\delta)/2}}{\sqrt{n(1-\gamma)^2}}
    \end{equation*}
    with probability greater than $1-\delta$.
\end{lemma}

The high-level idea of proof is to first study the concentration of $\abs{\what F_s(x)-F_s(x)}$, where $\what F_s(x)$ is defined to be the cumulative distribution function of $\brk{\what\gT_n^\pi \eta^\pi}(s)$ and $F_s(x)$ is defined as the cumulative distribution function of $\brk{\gT^\pi \eta^\pi}(s)$.
And then use the fact that 
\begin{equation*}
\norm{\brk{\prn{\what \gT_n^\pi-\gT^\pi}\eta^\pi}(s)}_{\FW}=\int_0^{\frac{1}{1-\gamma}}\abs{\what F_s(x)-F_s(x)}dx.      
\end{equation*}

\begin{lemma}\label{Lemma_cdf_sub_gaussian_1}
    Let $F_s(\cdot)$, $\what F_s(\cdot)$ denote the cumulative distribution function of $\brk{\gT^\pi\eta^\pi}(s)$, $\brk{\what\gT^\pi_n\eta^\pi}(s)$, respectively.
    Then $\sqrt{n}(\what F_s(x)-F_s(x))$ is $\frac{1}{\sqrt{n}}$-sub-gaussian with zero mean, $\forall x\in\brk{0,\frac{1}{1-\gamma}}$, $\forall s\in\gS$.
\end{lemma}
\begin{proof}[Proof of Lemma~\ref{Lemma_cdf_sub_gaussian_1}.]
Recall that 
\begin{equation*}
\begin{aligned}
    F_s(x)&=\sum_{a\in\gA} \pi(a\mid s)\sum_{s^\prime \in\gS}w_{s,a,s^\prime}P(s^\prime\mid s,a)\\
    \what F_s(x)&=\sum_{a\in\gA} \pi(a\mid s)\sum_{s^\prime\in\gS}w_{s,a,s^\prime}\what P(s^\prime\mid s,a),
\end{aligned}
\end{equation*}
where the weights $w_{s,a,s^\prime}:=\int_0^1 F_{s^\prime}^G\prn{\frac{x-r}{\gamma}}\gP_R(dr\mid s,a)$ and we always have $w_{s,a,s^\prime}\in[0,1]$.
Apparently $\EB\what F_s(x)=F_s(x)$.
By Jensen's inequality, we can get
\begin{equation*}
    \EB\exp{\prn{\lambda\prn{\what F_s(x)- F_s(x)}}}\leq \sum_{a\in\gA} \pi(a\mid s)\EB\exp{\prn{\lambda\sum_{s^\prime\in\gS}w_{s,a,s^\prime}\prn{\what P(s^\prime\mid s,a)- P(s^\prime\mid s,a)}}}.
\end{equation*}
Since 
\begin{equation*}
\begin{aligned}
\sum_{s^\prime\in\gS}w_{s,a,s^\prime}\prn{\what P(s^\prime\mid s,a)-P(s^\prime\mid s,a)}&=\sum_{s^\prime\in\gS}w_{s,a,s^\prime}\brk{\frac{1}{n}\sum_{i=1}^n\ind \brc{X_i^{(s,a)}=s^\prime}-P(s^\prime\mid s,a)}\\
&=\frac{1}{n}\sum_{i=1}^n\sum_{s^\prime\in\gS}w_{s,a,s^\prime}\prn{\ind\brc{X_i^{(s,a)}=s^\prime}-P(s^\prime\mid s,a)}.
\end{aligned}
\end{equation*}
Therefore, 
\begin{equation*}
\begin{aligned}
    &\EB\exp{\prn{\lambda\sum_{s^\prime\in\gS}w_{s,a,s^\prime}\prn{\what P(s^\prime\mid s,a)- P(s^\prime\mid s,a)}}}\\
    &=\EB\exp{\prn{\frac{\lambda}{n}\sum_{i=1}^n\sum_{s^\prime\in\gS}w_{s,a,s^\prime}\prn{\ind\brc{X_i^{(s,a)}=s^\prime}-P(s^\prime\mid s,a)}}}\\
    &=\brk{\EB \exp{\prn{\frac{\lambda}{n}\sum_{s^\prime\in\gS}w_{s,a,s^\prime}\prn{\ind\brc{X_i^{(s,a)}=s^\prime}-P(s^\prime\mid s,a)}}}}^n.
\end{aligned}
\end{equation*}
Because 
\begin{equation*}
\begin{aligned}
    \sum_{s^\prime\in\gS}w_{s,a,s^\prime} \prn{\ind\brc{X_i^{(s,a)}=s^\prime}-P(s^\prime\mid s,a)}&\leq 1-\sum_{s^\prime\in\gS}w_{s,a,s^\prime}P(s^\prime\mid s,a)\leq 1,\\
    \sum_{s^\prime\in\gS}w_{s,a,s^\prime} \prn{\ind\brc{X_i^{(s,a)}=s^\prime}-P(s^\prime\mid s,a)}&\geq -\sum_{s^\prime\in\gS}w_{s,a,s^\prime}P(s^\prime\mid s,a)\geq -1,
\end{aligned}
\end{equation*}
we may obtain 
\begin{equation*}
    \EB \exp{\prn{\frac{\lambda}{n}\sum_{s^\prime\in\gS}w_{s,a,s^\prime}\prn{\ind\brc{X_i^{(s,a)}=s^\prime}-P(s^\prime\mid s,a)}}}\leq \exp{\prn{\frac{\lambda^2}{2n^2}}}
\end{equation*}
through Hoeffding's lemma (Lemma~\ref{Lemma_Hoeffding_lemma}).
Consequently, we may conclude
\begin{equation*}
    \EB\exp{\prn{\lambda\sum_{s^\prime\in\gS}w_{s,a,s^\prime}\prn{\what P(s^\prime\mid s,a)- P(s^\prime\mid s,a)}}}\leq \exp{\prn{\frac{\lambda^2}{2n}}}
\end{equation*}
and further
\begin{equation*}
     \EB\exp{\prn{\lambda\prn{\what F_s(x)- F_s(x)}}}\leq \exp{\prn{\frac{\lambda^2}{2n}}},
\end{equation*}
which completes the proof.
\end{proof}

\begin{proof}[Proof of Lemma~\ref{Lemma_concentration_operator_W1}]
\begin{equation*}
    \begin{aligned}
        \sup_{s\in\gS}\norm{\brk{\prn{\what \gT_n^\pi-\gT^\pi}\eta^\pi}(s)}_{\FW}&=\sup_{s\in\gS}\int_0^{\frac{1}{1-\gamma}}\abs{\what F_{s}(x)-F_{s}(x)}dx.
    \end{aligned}
\end{equation*}
Here $F_s(\cdot)$, $\what F_s(\cdot)$ denotes the cumulative distribution function of $\brk{\gT^\pi\eta^\pi}(s)$, $\brk{\what\gT^\pi_n\eta^\pi}(s)$, respectively.
Specifically, we have
% \begin{equation*}
%     \begin{aligned}
%         F_s(x)&=\sum_{a\in\gA}\pi(a\mid s)\sum_{s^\prime\in\gS} P(s^\prime\mid s,a)\int_0^1 F^G_{s^\prime}\prn{\frac{x-r}{\gamma}}\gP_R(dr\mid s,a),\\
%         \what F_s(x)&=\sum_{a\in\gA}\pi(a\mid s)\sum_{s^\prime\in\gS} \what P(s^\prime\mid s,a)\int_0^1 F^G_{s^\prime}\prn{\frac{x-r}{\gamma}}\gP_R(dr\mid s,a)\\
%         % &=\frac{1}{n}\sum_{a\in\gA}\pi(a\mid s)\sum_{s^\prime\in\gS} \sum_{i=1}^n \ind\brc{X_i^{(s,a)}=s^\prime}\int_0^1 F^G_{s^\prime}\prn{\frac{x-r}{\gamma}}\gP_R(dr\mid s,a).
%     \end{aligned}
% \end{equation*}
% Note that for any $s,s^\prime\in\gS$, $a\in\gA$, we always have $\int_0^1 F^G_{s^\prime}\prn{\frac{x-r}{\gamma}}\gP_R(dr\mid s,a)\leq 1$.
% Therefore, it is straightforward to verify that 
$\what F_s(x)- F_s(x)$ is $\frac{1}{\sqrt{n}}$-sub-gaussian (see Lemma~\ref{Lemma_cdf_sub_gaussian_1}) and thus 
\begin{equation*}
    \EB\sup_{s\in\gS}\abs{\what F_s(x)-F(x)}\leq 3\sqrt{\frac{\log|\gS|}{n}}, \forall x\in\brk{0,\frac{1}{1-\gamma}}.
\end{equation*}

Therefore, we may get
\begin{equation*}
    \begin{aligned}
        \EB  \sup_{s\in\gS}\norm{\brk{\prn{\what \gT_n^\pi-\gT^\pi}\eta^\pi}(s)}_{\FW}&\leq\int_0^{\frac{1}{1-\gamma}}\EB\sup_{s\in\gS}\abs{\what F_{s}(x)-F_{s}(x)}dx\\
        &\leq \sqrt{\frac{9\log|\gS|}{n(1-\gamma)^2}}.
    \end{aligned}
\end{equation*}

The high probability bound for $\sup_{s\in\gS}\norm{\brk{\prn{\what \gT_n^\pi-\gT^\pi}\eta^\pi}(s)}_{\FW}$ can be derived via a combination of the bound on expectation and McDiarmid's inequality (Lemma~\ref{Lemma_McDiarmid_inequality}).
Note that for any fixed $i\in\brc{1,\dots,n}$, substituting data vector $\prn{X_i^{(s,a)}}_{s\in\gS,a\in\gA}$ can change $\sup_{s\in\gS}\norm{\brk{\prn{\what \gT_n^\pi-\gT^\pi}\eta^\pi}(s)}_{\FW}$ by at most $\frac{1}{n(1-\gamma)}$.
Hence for any $\delta>0$, with probability at least $1-\delta$,
\begin{equation*}
    \begin{aligned}
        \sup_{s\in\gS}\norm{\brk{\prn{\what \gT_n^\pi-\gT^\pi}\eta^\pi}(s)}_{\FW}&\leq \EB \sup_{s\in\gS}\norm{\brk{\prn{\what \gT_n^\pi-\gT^\pi}\eta^\pi}(s)}_{\FW} + \sqrt{\frac{\log(1/
        \delta)}{2n(1-\gamma)^2}}\\
        &\leq \frac{\sqrt{9\log |\gS|}+\sqrt{\log(1/\delta)/2}}{\sqrt{n(1-\gamma)^2}}.
    \end{aligned}
\end{equation*}
\end{proof}

The next step is to relate the error term $\hat\eta_n^\pi-\eta^\pi$ with $\prn{\what \gT_n^\pi-\gT^\pi}\eta^\pi$.
Since
\begin{equation*}
    \prn{\gI-\what\gT_n^\pi}(\hat\eta_n^\pi-\eta^\pi)=\prn{\what \gT_n^\pi-\gT^\pi}\eta^\pi,
\end{equation*}
this can be immediately accomplished if $\prn{\gI-\what\gT_n^\pi}$ is invertible on $\prn{M_{\FW}^0}^\gS$.
The invertibility does not hold in general.
% since $(\gI-\what\gT_n^\pi)$ is a linear operator and has a non-trivial kernel ( $0\not=\hat\eta_n^\pi\in\textup{Ker}\prn{\gI-\what\gT_n^\pi}$).
However, we find $\prn{\gI-\what\gT_n^\pi}$ is invertible on the closure $\prn{\overline{M^0_{\FW}}}^\gS$, 
% with
% \begin{equation*}
%     M^0_{\FW}:=\brc{\mu\text{ signed measure on }\prn{\brk{0,\frac{1}{1-\gamma}},\gB_0},\norm{\mu}_{\FW}<\infty,\mu\prn{\brk{0,\frac{1}{1-\gamma}}}=0},    
% \end{equation*}
which suffices for our analysis.
\begin{lemma}\label{Lemma_operator_invertible}
The operator $(\gI-\gT^\pi)$ is invertible on $\prn{\overline{M^0_{\FW}}}^\gS$ and $\prn{\gI-\gT^\pi}^{-1}=\sum_{i=0}^\infty \prn{\gT^\pi}^i$.
Also, the operator norm of $\prn{\gI-\gT^\pi}^{-1}$ is upper bounded by $\frac{1}{1-\gamma}$.
Here $\gT^\pi$ can be replaced by any valid distributional Bellman operator.
\end{lemma}

Lemma~\ref{Lemma_operator_invertible} not only shows $(\gI-\gT^\pi)$ is invertible but constructs the inverse explicitly, thereby facilitating computational convenience.
Note that for $\mu\in\prn{M^0_{\FW}}^\gS$,  it is not necessarily true that 
$\prn{\gI-\gT^\pi}^{-1}\mu\in\prn{M^0_{\FW}}^\gS$ because the space $\prn{M^0_{\FW}}^\gS$ is not complete.
Precisely, we have $\prn{\gI-\gT^\pi}^{-1}\mu\in\prn{\overline{M^0_{\FW}}}^\gS$ that is a closed subspace of $\prn{\ell^\infty(\FW)}^\gS$ for any $\mu\in \prn{M^0_{\FW}}^\gS$.

To prove Lemma~\ref{Lemma_operator_invertible}, the main idea is to show the convergence of the Neumann series $\sum_{i=1}^\infty \prn{\gT^\pi}^i$.
% on $\prn{M^0_{\FW}}^\gS$. 
% to show $\norm{\gT^\pi}_{\textup{op}}<1$ utilizing the contraction property of the distributional Bellman operator $\gT^\pi$.
% Given that the space $\prn{M^0_{\FW}}^\gS$ is complete, the Neumann series is guaranteed to converge in $\prn{M^0_{\FW}}^\gS$.
For any $\mu\in M^0_{\FW}$, one has the Jordan decomposition $\mu=a_\mu(\mu_+-\mu_-)$ such that $a_\mu$ is a positive constant and $\mu_+,\mu_-$ are probability measures.
Thus $\norm{\mu}_{\FW}=a_\mu W_1(\mu_+,\mu_-)$ and the convergence of Neumann series can be shown utilizing the contraction property of the distributional Bellman operator $\gT^\pi$.
\begin{proof}[Proof of Lemma~\ref{Lemma_operator_invertible}]
    We first verify that the Neumann series $\sum_{i=0}^\infty\prn{\gT^\pi}^i$ converges in $\prn{\overline{M^0_{\FW}}}^\gS$.
    First, we claim for any $\nu\in \prn{M^0_{\FW}}^\gS$, $\brc{\sum_{i=1}^k\prn{\gT^\pi}^i\nu,k=1,2,\dots}$ is a Cauchy sequence.
    WLOG, for any $\nu\in \prn{M^0_{\FW}}^\gS$, $s\in\gS$, we may write $\nu=a_\nu(s)\prn{\nu_+(s)-\nu_-(s)}$, where $a_\nu(s)$ is a positive constant and $\nu_+(s),\nu_-(s)\in\Delta\prn{[0,\frac{1}{1-\gamma}]}$.
    Note that $\norm{\nu(s)}_{\FW}=a_\nu(s) W_1(\nu_+(s),\nu_-(s))$ for any $s\in\gS$.
    For $k_1<k_2$, we have for any $s\in\gS$
    \begin{equation*}
        \begin{aligned}
            \norm{\sum_{i=k_1}^{k_2}\brk{\prn{\gT^\pi}^i\nu}(s)}_{\FW}&\leq\sum_{i=k_1}^{k_2}\norm{\brk{\prn{\gT^\pi}^i\nu}(s)}_{\FW}\\
            &=\sup_{s\in\gS}a_\nu(s)\sum_{i=k_1}^{k_2}W_1\prn{\brk{\prn{\gT^\pi}^i\nu_+}(s),\brk{\prn{\gT^\pi}^i\nu_-}(s)}\\
            &\leq \sup_{s\in\gS}a_\nu(s)\sum_{i=k_1}^{k_2}\gamma^i\sup_{s\in\gS}W_1\prn{\nu_+(s),\nu_-(s)}\\
            &\leq \frac{\gamma^{k_1}\sup_{s\in\gS}a_\nu(s)\sup_{s\in\gS}W_1\prn{\nu_+(s),\nu_-(s)}}{1-\gamma}.
        \end{aligned}
    \end{equation*}
    The second last inequality is due to Proposition~\ref{Proposition_value_iteration_Wp}.
    Then we define $\nu_*$ such that for any $f\in\FW$, $\nu_*(s)f:=\lim_{k\to\infty}\sum_{i=0}^k\brk{\prn{\gT^\pi}^i\nu}(s)f$. 
    $\nu_*$ is well-defined because $\brc{\sum_{i=1}^k\prn{\gT^\pi}^i\nu(s) f,k=1,2,\dots}$ is a Cauchy sequence in $\RB$ by similar arguments as above.
    As 
    \begin{equation*}
        \begin{aligned}
            \norm{\nu_*(s)-\sum_{i=0}^{k}\brk{\prn{\gT^\pi}^i\nu}(s)}_{\FW}
            &=\norm{\sum_{i=k}^{\infty}\brk{\prn{\gT^\pi}^i\nu}(s)}_{\FW}\\
            &\leq \sum_{i=k}^\infty \norm{\brk{\prn{\gT^\pi}^i\nu}(s)}_{\FW}\\
            &\leq \frac{\gamma^{k}\sup_{s\in\gS}a_\nu(s)\sup_{s\in\gS}W_1\prn{\nu_+(s),\nu_-(s)}}{1-\gamma},
        \end{aligned}
    \end{equation*}
    $\nu_*=
    \lim_{k\to\infty}\sum_{i=1}^k\prn{\gT^\pi}^i\nu$.
    % And our conclusion follows since $\nu_*\in \prn{M^0_{\FW}}^\gS$ by that $\forall s\in\gS$,
    % \begin{equation*}
    % \begin{aligned}
    %     \norm{\nu_*(s)}_{\FW}&\leq \sum_{i=0}^{\infty}\norm{\brk{\prn{\gT^\pi}^i\nu}(s)}_{\FW}\\
    %     &\leq \frac{a_\nu\sup_{s\in\gS}W_1\prn{\nu_+(s),\nu_-(s)}}{1-\gamma}\\
    %     &<\infty.
    % \end{aligned}
    % \end{equation*}
    Apparently $\nu_*\in \prn{\overline{M^0_{\FW}}}^\gS$.
    Since for any $\nu\in \prn{M^0_{\FW}}^\gS$,
        \begin{equation*}
        \begin{aligned}
            &\norm{\brk{\prn{\gI-\gT^\pi}^{-1}\nu}(s)}_{\FW}\\&=
            \norm{\sum_{i=0}^{\infty}\brk{\prn{\gT^\pi}^i\nu}(s)}_{\FW}\\
            &\leq\sum_{i=0}^{\infty}\norm{\brk{\prn{\gT^\pi}^i\nu}(s)}_{\FW}\\
            &\leq \sum_{i=0}^{\infty}\gamma^i\sup_{s\in\gS}a_\nu(s)W_1\prn{\nu_+(s),\nu_-(s)}\\
            &\leq \frac{\sup_{s\in\gS}a_\nu(s)W_1\prn{\nu_+(s),\nu_-(s)}}{1-\gamma}\\
            &=\frac{\sup_{s\in\gS}\norm{\nu(s)}_{\FW}}{1-\gamma}.
        \end{aligned}
    \end{equation*}
    We complete the proof. 
    We comment that the proof is inspired by the B.L.T. theorem (continuous linear extension theorem) in functional analysis.
\end{proof}

Applying Lemma~\ref{Lemma_operator_invertible} to $\what\gT_n^\pi$ leads to
\begin{equation*}
    \begin{aligned}
        \sup_{s\in\gS}W_1(\hat\eta^\pi_n(s),\eta^\pi(s))&=\sup_{s\in\gS} \norm{\hat\eta^\pi_n(s)-\eta^\pi(s)}_{\FW}\\
        &=\sup_{s\in\gS}\norm{\brk{\prn{\gI-\what \gT_n^\pi}^{-1}\prn{\what\gT_n^\pi-\gT^\pi}\eta^\pi}(s)}_{\FW}\\
        &\leq \frac{1}{1-\gamma}\sup_{s\in\gS}\norm{\brk{\prn{\what\gT_n^\pi-\gT^\pi}\eta^\pi}(s)}_{\FW},
    \end{aligned}
\end{equation*}
which finishes our analysis of the $W_1$-metric case.

\section{Analysis of Theorem~\ref{Theorem_bound_of_KS}}\label{Appendix_analysis_KS}
We first give a stronger notion of concentration of $\prn{\what\gT_n^\pi-\gT^\pi}\eta^\pi$.
\begin{lemma}\label{Lemma_concentration_operator_TV}
    Suppose Assumption~\ref{Assumption_reward_bounded_density} is true.
    For any fixed policy $\pi$, we have
    \begin{equation*}
        \EB\sup_{s\in\gS}\norm{\brk{\prn{\what \gT_n^\pi-\gT^\pi}\eta^\pi}(s)}_{\FTV}\leq \frac{3C}{2}\sqrt{\frac{\log\abs{\gS}}{n(1-\gamma)^2}}.
    \end{equation*}
    And for any $\delta\in(0,1)$,
    \begin{equation*}
        \sup_{s\in\gS}\norm{\brk{\prn{\what \gT_n^\pi-\gT^\pi}\eta^\pi}(s)}_{\FTV}\leq \frac{C\prn{\sqrt{9\log\abs{\gS}}+\sqrt{\log(1/\delta)/2}}}{2\sqrt{n(1-\gamma)^2}}
    \end{equation*}
    with probability greater than $1-\delta$.
\end{lemma}
Different from the proof of Lemma~\ref{Lemma_concentration_operator_W1}, this time we first bound the term $\abs{\hat p_s(x)-p_s(x)}$, where $\hat p_s(x)$, $p_s(x)$ are defined to be the density functions of $\brk{\what\gT_n^\pi\eta^\pi}(s)$, $\brk{\gT^\pi\eta^\pi}$ respectively.
Then we may draw the conclusion noting that 
\begin{equation*}
    \norm{\brk{\prn{\what \gT_n^\pi-\gT^\pi}\eta^\pi}(s)}_{\FTV}=\frac{1}{2}\int_0^{\frac{1}{1-\gamma}} \abs{\hat p_s(x)-p_s(x)}dx.
\end{equation*}

\begin{lemma}\label{Lemma_pdf_sub_gaussian_1}
    Suppose Assumption~\ref{Assumption_reward_bounded_density} holds.
    Let $p_s(\cdot)$, $\hat p_s(\cdot)$ denote the density function of $\brk{\gT^\pi\eta^\pi}(s)$, $\brk{\what\gT^\pi_n\eta^\pi}(s)$, respectively.
    Then $\sqrt{n}(\hat p_s(x)-p_s(x))$ is $\frac{C}{\sqrt{n}}$-sub-gaussian with zero mean, $\forall x\in\brk{0,\frac{1}{1-\gamma}}$, $\forall s\in\gS$.
\end{lemma}
\begin{proof}[Proof of Lemma~\ref{Lemma_pdf_sub_gaussian_1}.]
Recall that 
\begin{equation*}
\begin{aligned}
    p_s(x)&=\sum_{a\in\gA} \pi(a\mid s)\sum_{s^\prime \in\gS}w_{s,a,s^\prime}P(s^\prime\mid s,a)\\
    \hat p_s(x)&=\sum_{a\in\gA} \pi(a\mid s)\sum_{s^\prime\in\gS}w_{s,a,s^\prime}\what P(s^\prime\mid s,a),
\end{aligned}
\end{equation*}
where the weights $w_{s,a,s^\prime}:=\prn{p_{s^\prime}^{\gamma G}\ast p^R_{s,a}}(x)$ and we always have $w_{s,a,s^\prime}\in[0,C]$ due to $\sup_{x\in[0,1/(1-\gamma)]}p_{s,a}^R\leq C$.
Apparently $\EB\hat p_s(x)=p_s(x)$.
By Jensen's inequality, we can get
\begin{equation*}
    \EB\exp{\prn{\lambda\prn{\hat p_s(x)- p_s(x)}}}\leq \sum_{a\in\gA} \pi(a\mid s)\EB\exp{\prn{\lambda\sum_{s^\prime\in\gS}w_{s,a,s^\prime}\prn{\what P(s^\prime\mid s,a)- P(s^\prime\mid s,a)}}}.
\end{equation*}
Since 
\begin{equation*}
\begin{aligned}
\sum_{s^\prime\in\gS}w_{s,a,s^\prime}\prn{\what P(s^\prime\mid s,a)-P(s^\prime\mid s,a)}&=\sum_{s^\prime\in\gS}w_{s,a,s^\prime}\brk{\frac{1}{n}\sum_{i=1}^n\ind \brc{X_i^{(s,a)}=s^\prime}-P(s^\prime\mid s,a)}\\
&=\frac{1}{n}\sum_{i=1}^n\sum_{s^\prime\in\gS}w_{s,a,s^\prime}\prn{\ind\brc{X_i^{(s,a)}=s^\prime}-P(s^\prime\mid s,a)}.
\end{aligned}
\end{equation*}
Therefore, 
\begin{equation*}
\begin{aligned}
    &\EB\exp{\prn{\lambda\sum_{s^\prime\in\gS}w_{s,a,s^\prime}\prn{\what P(s^\prime\mid s,a)- P(s^\prime\mid s,a)}}}\\
    &=\EB\exp{\prn{\frac{\lambda}{n}\sum_{i=1}^n\sum_{s^\prime\in\gS}w_{s,a,s^\prime}\prn{\ind\brc{X_i^{(s,a)}=s^\prime}-P(s^\prime\mid s,a)}}}\\
    &=\brk{\EB \exp{\prn{\frac{\lambda}{n}\sum_{s^\prime\in\gS}w_{s,a,s^\prime}\prn{\ind\brc{X_i^{(s,a)}=s^\prime}-P(s^\prime\mid s,a)}}}}^n.
\end{aligned}
\end{equation*}
Because 
\begin{equation*}
\begin{aligned}
    -C\leq\sum_{s^\prime\in\gS}w_{s,a,s^\prime} \prn{\ind\brc{X_i^{(s,a)}=s^\prime}-P(s^\prime\mid s,a)}\leq C,
\end{aligned}
\end{equation*}
we may obtain 
\begin{equation*}
    \EB \exp{\prn{\frac{\lambda}{n}\sum_{s^\prime\in\gS}w_{s,a,s^\prime}\prn{\ind\brc{X_i^{(s,a)}=s^\prime}-P(s^\prime\mid s,a)}}}\leq \exp{\prn{\frac{C^2\lambda^2}{2n^2}}}
\end{equation*}
through Hoeffding's lemma (Lemma~\ref{Lemma_Hoeffding_lemma}).
Consequently, we may conclude
\begin{equation*}
    \EB\exp{\prn{\lambda\sum_{s^\prime\in\gS}w_{s,a,s^\prime}\prn{\what P(s^\prime\mid s,a)- P(s^\prime\mid s,a)}}}\leq \exp{\prn{\frac{C^2\lambda^2}{2n}}}
\end{equation*}
and further
\begin{equation*}
     \EB\exp{\prn{\lambda\prn{\hat g_s(x)- g_s(x)}}}\leq \exp{\prn{\frac{C^2\lambda^2}{2n}}},
\end{equation*}
which completes the proof.
\end{proof}

\begin{proof}[Proof of Lemma~\ref{Lemma_concentration_operator_TV}]
When Assumption~\ref{Assumption_reward_bounded_density} holds,
\begin{equation*}
    \begin{aligned}
        \sup_{s\in\gS}\norm{\brk{\prn{\what \gT_n^\pi-\gT^\pi}\eta^\pi}(s)}_{\FTV}&=\sup_{s\in\gS}\frac{1}{2}\int_0^{\frac{1}{1-\gamma}}\abs{\hat p_{s}(x)-p_{s}(x)}dx.
    \end{aligned}
\end{equation*}
Here $p_s(\cdot)$, $\hat p_s(\cdot)$ denotes the density function of $\brk{\gT^\pi\eta^\pi}(s)$, $\brk{\what\gT^\pi_n\eta^\pi}(s)$, respectively.
Specifically, we have $\hat p_s(x)- p_s(x)$ is $\frac{C}{\sqrt{n}}$-sub-gaussian (see Lemma~\ref{Lemma_pdf_sub_gaussian_1}) and thus 
\begin{equation*}
    \EB\sup_{s\in\gS}\abs{\hat p_s(x)-p(x)}\leq 3C  \sqrt{\frac{\log|\gS|}{n}}, \forall x\in\brk{0,\frac{1}{1-\gamma}}.
\end{equation*}

Therefore, we may get
\begin{equation*}
    \begin{aligned}
        \EB  \sup_{s\in\gS}\norm{\brk{\prn{\what \gT_n^\pi-\gT^\pi}\eta^\pi}(s)}_{\FTV}&\leq\frac{1}{2}\int_0^{\frac{1}{1-\gamma}}\EB\sup_{s\in\gS}\abs{\hat p_{s}(x)-p_{s}(x)}dx\\
        &\leq \frac{3C}{2}\sqrt{\frac{\log|\gS|}{n(1-\gamma)^2}}.
    \end{aligned}
\end{equation*}

The high probability bound for $\sup_{s\in\gS}\norm{\brk{\prn{\what \gT_n^\pi-\gT^\pi}\eta^\pi}(s)}_{\FTV}$ can be derived via a combination of bound on expectations and McDiarmid's inequality (Lemma~\ref{Lemma_McDiarmid_inequality}).
Note that for any fixed $i\in\brc{1,\dots,n}$, substituting data vector $\prn{X_i^{(s,a)}}_{s\in\gS,a\in\gA}$ can change $\sup_{s\in\gS}\norm{\brk{\prn{\what \gT_n^\pi-\gT^\pi}\eta^\pi}(s)}_{\FTV}$ by at most $\frac{C}{2n(1-\gamma)}$.
Hence for any $\delta>0$, with probability at least $1-\delta$,
\begin{equation*}
    \begin{aligned}
        \sup_{s\in\gS}\norm{\brk{\prn{\what \gT_n^\pi-\gT^\pi}\eta^\pi}(s)}_{\FTV}&\leq \EB \sup_{s\in\gS}\norm{\brk{\prn{\what \gT_n^\pi-\gT^\pi}\eta^\pi}(s)}_{\FTV} + \frac{C}{2}\sqrt{\frac{\log(1/
        \delta)}{2n(1-\gamma)^2}}\\
        &\leq \frac{C\prn{\sqrt{9\log |\gS|}+\sqrt{\log(1/\delta)/2}}}{2\sqrt{n(1-\gamma)^2}}.
    \end{aligned}
\end{equation*}
\end{proof}

As before, the next step is to show $(\gI-\what\gT_n^\pi)$ is invertible on some space containing the signed measures of interest.
Specifically, let
\begin{equation*}
    M_{\FKS}^0:=\brc{\mu\text{ signed measure on }\prn{\brk{0,\frac{1}{1-\gamma}},\gB_0}\mid\norm{\mu}_{\FKS}<\infty, \norm{\mu}_{\loc}<\infty,\mu\prn{\brk{0,\frac{1}{1-\gamma}}}=0.}
\end{equation*}
Here $\norm{\mu}_{\loc}<\infty$ represents that $\mu$ has a density $f$ such that for any $B\in\gB_0$, $\mu(B)=\int_B f(x)dx$ and $\norm{\mu}_{\loc}\colon=\sup_{x\in[0,1/(1-\gamma)]} |f(x)|$.
When $\mu\in M^0_{\FKS}$, we can control $\norm{\mu}_{\FKS}$ with $\norm{\mu}_{\FW}$ and $\norm{\mu}_{\loc}$.
Formally, we have the following proposition as a generalization of Proposition~\ref{Proposition_bound_KS_with_W1}.

\begin{proposition}\label{Proposition_bound_KS_with_W1_signed}
Suppose $\mu\in M_{\FKS}^0$, then $\norm{\mu}_{\FKS}\leq \sqrt{2\norm{\mu}_{\loc}\norm{\mu}_{\FW}}$.
\end{proposition}

\begin{proof}[Proof of Proposition~\ref{Proposition_bound_KS_with_W1_signed}]
 For any $x\in\brk{0,\frac{1}{1-\gamma}}$, we define $h^x_\varepsilon(\cdot)$ to be a smoothed version of $\ind_{(-\infty,x]}$.
 Specifically, 
 \begin{equation*}
     h_\varepsilon^x(t)=\begin{cases}
         1 & t\leq x\\
         1-\frac{t-x}{\varepsilon} & x<t\leq x+\varepsilon\\
         0 & t>x+\varepsilon
     \end{cases}
 \end{equation*}
We have $h_\varepsilon^x(t)$ is Lipschitz with Lipschitz coefficient $\frac{1}{\varepsilon}$ and $h_\varepsilon^x(t)\geq \ind_{(-\infty,x]}$.
 For any $\mu\in M^0_{\FKS}$, we have the Jordan decomposition $\mu=\mu_+-\mu_-$.
 Then we have
 \begin{equation*}
     \begin{aligned}
         \mu\ind_{(-\infty,x]}&=\mu_+\ind_{(-\infty,x]}- \mu_-\ind_{(-\infty,x]}\\
         &=\mu_+\ind_{(-\infty,x]}-\mu_-h_\varepsilon^x(\cdot)+\mu_-h_\varepsilon^x(\cdot)-\mu_-\ind_{(-\infty,x]}\\
         &\leq (\mu_+-\mu_-)h_\varepsilon^x(\cdot)+\mu_-(h_\varepsilon^x(\cdot)-\ind_{(\-\infty,x]})\\
         &\leq \frac{1}{\varepsilon}\mu(\varepsilon h_\varepsilon^x(\cdot))+\mu_-(h_\varepsilon^x(\cdot)-\ind_{(\-\infty,x]})\\
         &\leq \frac{\norm{\mu}_{\FW}}{\varepsilon}+\varepsilon\norm{\mu}_{\loc}.
     \end{aligned}
 \end{equation*}
 Setting $\varepsilon=\frac{\norm{\mu}_{\FW}}{\norm{\mu}_{\loc}}$ yields
 \begin{equation*}
      \mu\ind_{(-\infty,x]}\leq \sqrt{2\norm{\mu}_{\loc}\norm{\mu}_{\FW}}.
 \end{equation*}
 We may also prove $-\mu\ind_{(-\infty,x]}\leq \sqrt{2\norm{\mu}_{\loc}\norm{\mu}_{\FW}}$ by similar arguments.
\end{proof}

When Assumption~\ref{Assumption_reward_bounded_density} is true, we always have $(\eta^\pi-\hat\eta_n^\pi)\in\prn{M_{\FKS}^0}^\gS$ as $\sup_{s\in\gS}\norm{\eta^\pi(s)-\hat\eta_n^\pi(s)}_{\loc}\leq C$.
Also, if $\mu\in\prn{M^0_{\FKS}}^\gS$, then $\gT^\pi\mu\in\prn{M^0_{\FKS}}^\gS$ (Lemma~\ref{Lemma_bounded_density_close}).
Here $\gT^\pi$ can be replaced by any valid distributional Bellman operator, for example, $\what \gT^\pi_n$.
% It is easy to verify $\sup_{s\in\gS}\norm{\brk{\gT^\pi\mu}(s)}_{\loc}\leq \sup_{s\in\gS}\norm{\mu(s)}_{\loc}$ for any $\mu\in\prn{M_{\FKS}^0}^\gS$ and any valid distributional operator $\gT^\pi$.

\begin{lemma}\label{Lemma_operator_invertible_KS}
For any valid distributional Bellman operator $\gT^\pi$, the operator $(\gI-\gT^\pi)$ is invertible on $\prn{\overline{M^0_{\FKS}}}^\gS$ and $\prn{\gI-\gT^\pi}^{-1}=\sum_{i=0}^\infty \prn{\gT^\pi}^i$.
\end{lemma}
\begin{proof}[Proof of Lemma~\ref{Lemma_operator_invertible_KS}]
    It suffices to verify that the Neumann series $\sum_{i=0}^\infty\prn{\gT^\pi}^i$ converges in $\prn{\overline{M^0_{\FKS}}}^\gS$.
    First, we claim for any $\nu\in \prn{M^0_{\FKS}}^\gS$, $\brc{\sum_{i=1}^k\prn{\gT^\pi}^i\nu,k=1,2,\dots}$ is a Cauchy sequence.
    WLOG, for any $\nu\in \prn{M^0_{\FKS}}^\gS$, $s\in\gS$, we may write $\nu=a_\nu(s)\prn{\nu_+(s)-\nu_-(s)}$, where $a_\nu(s)$ is a positive constant and $\nu_+(s),\nu_-(s)\in\Delta\prn{[0,\frac{1}{1-\gamma}]}$.
    Note that $\norm{\nu(s)}_{\FKS}=a_\nu(s) \KS(\nu_+(s),\nu_-(s))$ for any $s\in\gS$.
    There also exists a positive constant $C_\nu$ such that suppose for any $s\in\gS$, both $\nu_+(s)$ and $\nu_-(s)$ have a density bounded by $C_\nu$.
    For $k_1<k_2$, we have for any $s\in\gS$
    \begin{equation*}
        \begin{aligned}
            \norm{\sum_{i=k_1}^{k_2}\brk{\prn{\gT^\pi}^i\nu}(s)}_{\FKS}&\leq\sum_{i=k_1}^{k_2}\norm{\brk{\prn{\gT^\pi}^i\nu}(s)}_{\FKS}\\
            &=\sup_{s\in\gS}a_\nu(s)\sum_{i=k_1}^{k_2}\KS\prn{\brk{\prn{\gT^\pi}^i\nu_+}(s),\brk{\prn{\gT^\pi}^i\nu_-}(s)}\\
            &\leq \sup_{s\in\gS}a_\nu(s)\sum_{i=k_1}^{k_2}\sqrt{2C_\nu W_1\prn{\prn{\gT^\pi}^i\nu_+(s),\prn{\gT^\pi}^i\nu_-(s)}}\\
            &\leq \sup_{s\in\gS}a_\nu(s)\sum_{i=k_1}^{k_2}\sqrt{2C_\nu\gamma^i\sup_{s\in\gS}W_1\prn{\nu_+(s),\nu_-(s)}}\\
            &\leq \frac{\gamma^{k_1/2}\sup_{s\in\gS}a_\nu(s)\sqrt{2C_\nu\sup_{s\in\gS}W_1\prn{\nu_+(s),\nu_-(s)}}}{1-\sqrt{\gamma}}.
        \end{aligned}
    \end{equation*}
    The second inequality is by Proposition~\ref{Proposition_bound_KS_with_W1} and
    the third inequality is due to Proposition~\ref{Proposition_value_iteration_Wp}.
    Then we define $\nu_*$ such that for any $f\in\FKS$, $\nu_*(s)f:=\lim_{k\to\infty}\sum_{i=0}^k\brk{\prn{\gT^\pi}^i\nu}(s)f$. 
    $\nu_*$ is well-defined because $\brc{\sum_{i=1}^k\prn{\gT^\pi}^i\nu(s) f,k=1,2,\dots}$ is a Cauchy sequence in $\RB$ by similar arguments as above.
    As 
    \begin{equation*}
        \begin{aligned}
            \norm{\nu_*(s)-\sum_{i=0}^{k}\brk{\prn{\gT^\pi}^i\nu}(s)}_{\FKS}
            &=\norm{\sum_{i=k}^{\infty}\brk{\prn{\gT^\pi}^i\nu}(s)}_{\FKS}\\
            &\leq \sum_{i=k}^\infty \norm{\brk{\prn{\gT^\pi}^i\nu}(s)}_{\FKS}\\
            &\leq \frac{\gamma^{k/2}\sup_{s\in\gS}a_\nu(s)\sqrt{2C_\nu\sup_{s\in\gS}W_1\prn{\nu_+(s),\nu_-(s)}}}{1-\sqrt{\gamma}},
        \end{aligned}
    \end{equation*}
    $\nu_*=
    \lim_{k\to\infty}\sum_{i=1}^k\prn{\gT^\pi}^i\nu$.
    Apparently $\nu_*\in \prn{\overline{M^0_{\FKS}}}^\gS$.
\end{proof}

Like the case of Lemma~\ref{Lemma_operator_invertible}, for $\mu\in\prn{M^0_{\FKS}}^\gS$, $\prn{\gI-\gT^\pi}^{-1}\mu$ does not necessarily lie in $\prn{M^0_{\FKS}}^\gS$.
Instead, we have $\prn{\gI-\gT^\pi}^{-1}\mu\in\prn{\overline{M^0_{\FKS}}}^\gS$ that is a closed subspace of $\prn{\ell^\infty(\FKS)}^\gS$ for any $\mu\in \prn{M^0_{\FKS}}^\gS$.

Since the inverse $\prn{\gI-\what\gT_n^\pi}^{-1}$ can be unbounded in $M^0_{\FKS}$, the analysis is more involved.
Here we detour the technical problem with an ``expansion trick".
For any $s\in\gS$, we have
\begin{equation*}
    \begin{aligned}
        &\norm{\hat\eta^\pi_n(s)-\eta^\pi(s)}_{\FKS}\\
        &=\norm{\brk{\prn{\gI-\what\gT_n^\pi}^{-1}\prn{\what\gT_n^\pi-\gT^\pi}\eta^\pi}(s)}_{\FKS}\\
        &=\norm{\brk{\sum_{i=0}^\infty \prn{\what\gT_n^\pi}^i\prn{\what\gT_n^\pi-\gT^\pi}\eta^\pi}(s)}_{\FKS}\\
        &\leq \norm{\brk{\prn{\what\gT_n^\pi-\gT^\pi}\eta^\pi}(s)}_{\FKS} +\sum_{i=1}^\infty\norm{\brk{ \prn{\what\gT_n^\pi}^i\prn{\what\gT_n^\pi-\gT^\pi}\eta^\pi}(s)}_{\FKS}\\
        &\leq \norm{\brk{\prn{\what\gT_n^\pi-\gT^\pi}\eta^\pi}(s)}_{\FKS} +\sum_{i=1}^\infty\sqrt{2\norm{\brk{ \prn{\what\gT_n^\pi}^i\prn{\what\gT_n^\pi-\gT^\pi}\eta^\pi}(s)}_{\FW}\norm{\brk{ \prn{\what\gT_n^\pi}^i\prn{\what\gT_n^\pi-\gT^\pi}\eta^\pi}(s)}_{\loc}}\\
        &\leq\norm{\brk{\prn{\what\gT_n^\pi-\gT^\pi}\eta^\pi}(s)}_{\FKS} +\sum_{i=1}^\infty\sqrt{2\gamma^i\sup_{s^\prime\in\gS}\norm{\brk{ \prn{\what\gT_n^\pi-\gT^\pi}\eta^\pi}(s^\prime)}_{\FW}\norm{\brk{ \prn{\what\gT_n^\pi}^i\prn{\what\gT_n^\pi-\gT^\pi}\eta^\pi}(s)}_{\loc}}\\
        &\leq \norm{\brk{\prn{\what\gT_n^\pi-\gT^\pi}\eta^\pi}(s)}_{\FTV} +\sqrt{2\sup_{s^\prime\in\gS}\norm{\brk{ \prn{\what\gT_n^\pi-\gT^\pi}\eta^\pi}(s^\prime)}_{\FW}}\sum_{i=1}^\infty\sqrt{\gamma^i\norm{\brk{ \prn{\what\gT_n^\pi}^i\prn{\what\gT_n^\pi-\gT^\pi}\eta^\pi}(s)}_{\loc}}.
    \end{aligned}
\end{equation*}
Here the second inequality is due to Proposition~\ref{Proposition_bound_KS_with_W1_signed}, the third inequality is by the contraction property of $\what\gT_n^\pi$, and the last inequality holds from the fact $\norm{\cdot}_{\FKS}\leq\norm{\cdot}_{\FTV}$.

Now we need an upper bound for $\norm{\brk{ \prn{\what\gT_n^\pi}^i\prn{\what\gT_n^\pi-\gT^\pi}\eta^\pi}(s)}_{\loc}$.
Assumption~\ref{Assumption_reward_bounded_density} implies $\norm{\brk{\prn{\what\gT_n^\pi}^i\prn{\what\gT_n^\pi-\gT^\pi}\eta^\pi}(s)}_{\loc}\leq C$, but we can do better here.
\begin{lemma}\label{Lemma_bound_loc_by_TV}
    Suppose Assumption~\ref{Assumption_reward_bounded_density} holds true.
    For any $i\geq1$, $s\in\gS$, 
    \begin{equation*}
        \norm{\brk{ \prn{\what\gT_n^\pi}^i\prn{\what\gT_n^\pi-\gT^\pi}\eta^\pi}(s)}_{\loc}\leq C\sup_{s^\prime\in\gS}\norm{\brk{\prn{\what\gT_n^\pi-\gT^\pi}\eta^\pi}(s^\prime)}_{\FTV}.
    \end{equation*}
\end{lemma}

The main idea is we first normalize $\prn{\what\gT_n^\pi-\gT^\pi}\eta^\pi$ to have a proper Jordan decomposition and then use the fact that under Assumption~\ref{Assumption_reward_bounded_density}, for any $\nu\in\Delta\prn{\brk{0,\frac{1}{1-\gamma}}}^\gS$, $\brk{\gT^\pi\nu}(s)$ must have a density function bounded by $C$ as long as $\nu(s)$ has a density, $\forall s\in\gS$.
Hence the condition $i\geq 1$ is necessary and that is the reason why we break the summation into two parts: $i=0$ and $i\geq 1$.

\begin{proof}[Proof of Lemma~\ref{Lemma_bound_loc_by_TV}]
    We write the Jordan decomposition of $\prn{\what\gT_n^\pi-\gT^\pi}\eta^\pi$ as 
    \begin{equation*}
        \brk{\prn{\what\gT_n^\pi-\gT^\pi}\eta^\pi}(s)=a(s)(\mu_+(s)-\mu_-(s))
    \end{equation*}
    Here for any $s\in\gS$, $\mu_+(s),\mu_-(s)\in\Delta\prn{\brk{0,\frac{1}{1-\gamma}}}$ have probability density functions $p_+(s),p_-(s)$. 
    And it is easy to verify $a(s)=\norm{ \brk{\prn{\what\gT_n^\pi-\gT^\pi}\eta^\pi}(s)}_{\FTV}$.
    Recall that for any $\nu\in\Delta\prn{\brk{0,\frac{1}{1-\gamma}}}^\gS$ with $p_s^\nu(\cdot)$ as the probability density function of $\nu(s)$, we have for any valid distributional Bellman operator $\gT^\pi$, $\brk{\gT^\pi\nu}(s)$ also has probability density $\tilde p_s(\cdot)$ and
    \begin{equation*}
        \tilde p_s(x)=\sum_{a\in\gA}\pi(a\mid s)\sum_{s^\prime\in\gS}P(s^\prime\mid s,a) \prn{p_s^\nu(\cdot/\gamma)/\gamma\ast p^R_{s,a}(\cdot)}(x).
    \end{equation*}
    As for any $x\in\brk{0,\frac{1}{1-\gamma}}$, $\abs{\prn{p_s^\nu(\cdot/\gamma)/\gamma\ast p^R_{s,a}(\cdot)}(x)}\leq \sup_{x\in[0,1/(1-\gamma)]}\abs{p_{s,a}^R(x)}$, we can get $\norm{\brk{\gT^\pi\nu}(s)}_{\loc}\leq C$ under Assumption~\ref{Assumption_reward_bounded_density}.
    Therefore, for $i\geq 1$, we have
    \begin{equation*}
        \begin{aligned}
            \norm{\brk{\prn{\what \gT_n^\pi}^i\prn{\what\gT_n^\pi-\gT^\pi}\eta^\pi}(s)}_{\loc}&\leq\sup_{s\in\gS}a(s)\norm{\brk{\prn{\what \gT_n^\pi}^i(\mu_+-\mu_-)}(s)}_{\loc}\\
            &\leq C\sup_{s\in\gS}a(s)
        \end{aligned}
    \end{equation*}
We complete the proof.
    
\end{proof}

To sum up,
\begin{equation*}
\begin{aligned}
     &\norm{\hat\eta^\pi_n(s)-\eta^\pi(s)}_{\FKS}\\
     &\leq\frac{\sqrt{\gamma}}{1-\sqrt{\gamma}}\sqrt{2C\sup_{s^\prime\in\gS}\norm{\brk{\prn{\what\gT_n^\pi-\gT^\pi}\eta^\pi}(s^\prime)}_{\FTV}\sup_{s^\prime\in\gS}\norm{\brk{\prn{\what\gT_n^\pi-\gT^\pi}\eta^\pi}(s^\prime)}_{\FW}}+\norm{\brk{\prn{\what\gT_n^\pi-\gT^\pi}\eta^\pi}(s)}_{\FTV}.
\end{aligned}
\end{equation*}
Combining Lemma~\ref{Lemma_concentration_operator_W1} and Lemma~\ref{Lemma_concentration_operator_TV}, we have for any $\delta\in(0,1)$, with probability at least $1-\delta$,
% \begin{equation*}
% \begin{aligned}
%     \norm{\hat\eta^\pi_n(s)-\eta^\pi(s)}_{\FKS}&\leq \frac{\sqrt{\gamma}}{1-\sqrt{\gamma}}\frac{C\prn{\sqrt{9\log|\gS|}+\sqrt{\log(2/\delta)/2}}}{\sqrt{2n(1-\gamma)^2}}+\frac{C\prn{\sqrt{9\log|\gS|}+\sqrt{\log(2/\delta)/2}}}{2\sqrt{n(1-\gamma)^2}}
% \end{aligned}
% \end{equation*}
\begin{equation*}
\begin{aligned}
    \norm{\hat\eta^\pi_n(s)-\eta^\pi(s)}_{\FKS}&\leq \frac{C}{2}\prn{\frac{\sqrt{2\gamma}}{1-\sqrt{\gamma}}+1}\frac{\sqrt{9\log|\gS|}+\sqrt{\log(2/\delta)/2}}{\sqrt{n(1-\gamma)^2}}
\end{aligned}
\end{equation*}
Note that $\frac{1}{1-\sqrt{\gamma}}=\frac{1+\gamma}{1-\gamma}\leq \frac{2}{1-\gamma}$ when $\gamma\in(0,1)$, thus
\begin{equation*}
    \sup_{s\in\gS}\norm{\hat\eta^\pi_n(s)-\eta^\pi(s)}_{\FKS}\leq \frac{C^\prime\prn{\sqrt{\log|\gS|}+\sqrt{\log(1/\delta)}}}{\sqrt{n(1-\gamma)^4}},
\end{equation*}
where $C^\prime$ is some constant depending on $C$ in Assumption~\ref{Assumption_reward_bounded_density}.
We also have
\begin{equation*}
    \begin{aligned}
        \EB\sup_{s\in\gS}\norm{\hat\eta^\pi_n(s)-\eta^\pi(s)}_{\FKS}&\leq C^\prime\sqrt{\frac{\log|\gS|}{n(1-\gamma)^4}}+\int_{0}^\infty \PB\prn{\sup_{s\in\gS}\norm{\hat\eta^\pi_n(s)-\eta^\pi(s)}_{\FKS}>C^\prime\sqrt{\frac{\log|\gS|}{n(1-\gamma)^4}}+t}dt\\
        &\leq C^\prime\sqrt{\frac{\log|\gS|}{n(1-\gamma)^4}}+\int_{0}^\infty \exp{\prn{-\frac{n(1-\gamma)^4t^2}{{C^\prime}^2}}}dt\\
        &\leq C^{\prime\prime}\sqrt{\frac{\log|\gS|}{n(1-\gamma)^4}},
    \end{aligned}
\end{equation*}
where $C^{\prime\prime}$ is some constant depending on $C$ in Assumption~\ref{Assumption_reward_bounded_density}.

\section{Analysis of Theorem~\ref{Theorem_bound_of_TV}}\label{Appendix_analysis_TV}
By Sobolev's inequality, we also have the following concentration results of $\prn{\what\gT_n^\pi-\gT^\pi}\eta^\pi$ as a corollary of Lemma~\ref{Lemma_concentration_operator_TV}.
\begin{corollary}\label{Corollary_concentration_operator_TV_Sobolev}
    Suppose Assumption~\ref{Assumption_reward_smooth} is true.
    For any fixed policy $\pi$, we have
    \begin{equation*}
        \EB\sup_{s\in\gS}\norm{\brk{\prn{\what \gT_n^\pi-\gT^\pi}\eta^\pi}(s)}_{\FTV}\leq 3M\sqrt{\frac{\log\abs{\gS}}{n(1-\gamma)^2}}.
    \end{equation*}
    And for any $\delta\in(0,1)$,
    \begin{equation*}
        \sup_{s\in\gS}\norm{\brk{\prn{\what \gT_n^\pi-\gT^\pi}\eta^\pi}(s)}_{\FTV}\leq \frac{M\prn{\sqrt{9\log\abs{\gS}}+\sqrt{\log(1/\delta)/2}}}{\sqrt{n(1-\gamma)^2}}
    \end{equation*}
    with probability greater than $1-\delta$.
\end{corollary}

We consider the following space containing the signed measures of interest.
Specifically, let
\begin{equation*}
    M_{\FTV}^0:=\brc{\mu\text{ signed measure on }\prn{\brk{0,\frac{1}{1-\gamma}},\gB_0}\mid\norm{\mu}_{\FTV}<\infty, \norm{\mu}_{H_1^1}<\infty,\mu\prn{\brk{0,\frac{1}{1-\gamma}}}=0.}
\end{equation*}
Here we abuse the notation to define $\norm{\mu}_{H_1^1}=\norm{p}_{H_1^1}$, where $f$ is the density of $\mu$. 
When $\mu\in M^0_{\FTV}$, we can control $\norm{\mu}_{\FTV}$ with $\norm{\mu}_{\FW}$ and $\norm{\mu}_{H_1^1}$.
Formally, we have the following proposition as a generalization of Proposition~\ref{Proposition_bound_TV_with_W1}.
\begin{proposition}\label{Proposition_bound_TV_with_W1_signed}
    Suppose $\mu\in M_{\FTV}^0$, then $\norm{\mu}_{\FTV}\leq\sqrt{K\norm{\mu}_{H_1^1}\norm{\mu}_{\FW}}$.
\end{proposition}
\begin{proof}[Proof of Proposition~\ref{Proposition_bound_TV_with_W1_signed}]
For any $\mu\in \prn{M_{\FTV}^0}^\gS$, we may write $\mu(s)=a_\mu(s)(\mu_+(s)-\mu_-(s))$, where $\mu_+(s),\mu_-(s)\in\Delta\prn{\brk{0,\frac{1}{1-\gamma}}}$ with density $p_s^+(\cdot), p_s^-(\cdot)\in H_1^1(\RB)$, and $a_\mu(s)$ is a normalization factor.
Let $\{\phi_m\}$ be the orthogonal system in $L^2([-1,1])$ of Legendre polynomials, and kernel 
$K_h(x)=h^{-1}\prn{\phi_0(0)\phi_0(x/h)+\phi_1(0)\phi_1(x/h)}$.
From Lemma 2.1 in \cite{chae2020wasserstein},
\begin{equation*}
    \begin{aligned}
        \norm{\mu(s)}_{\FTV}&=\frac{1}{2}a_\mu(s)\norm{p^+_s-p^-_s}_1\\
        &=\frac{1}{2}a_\mu(s)\prn{\norm{p^+_s-K_h\ast p^+_s}_1+\norm{K_h\ast p^-_s-K_h\ast p^+_s}_1+\norm{p^-_s-K_h\ast p^-_s}_1}\\
        &\leq K^\prime a_\mu(s)(h\norm{p_s^+}_{H_1^1}+h\norm{p_s^-}_{H_1^1}+W_1(p_s^+,p^-_s)/h)\\
        &=K^\prime h\brk{a_\mu(s)\prn{\norm{p_s^+}_{H_1^1}+\norm{p_s^-}_{H_1^1}}}+\frac{K^\prime}{h}a_\mu(s)W_1(p_s^+,p^-_s)\\
        &=K^\prime h\norm{\mu(s)}_{H_1^1}+\frac{K^\prime}{h}\norm{\mu(s)}_{\FW}.
    \end{aligned}
\end{equation*}
The conclusion follows by choosing $h=\sqrt{\norm{\mu(s)}_{\FW}/\prn{2K^\prime\norm{\mu(s)}_{H_1^1}}}$.
\end{proof}

When Assumption~\ref{Assumption_reward_smooth} is true, we always have $(\eta^\pi-\hat\eta_n^\pi)\in\prn{M_{\FTV}^0}^\gS$ (Lemma~\ref{Lemma_smooth_density_of_return}).
Also, if $\mu\in\prn{M^0_{\FTV}}^\gS$, then $\gT^\pi\mu\in\prn{M^0_{\FTV}}^\gS$ (\ref{Lemma_smooth_density_close}).
Here $\gT^\pi$ can be replaced by any valid distributional Bellman operator, for example, $\what\gT_n^\pi$.

\begin{lemma}\label{Lemma_operator_invertible_TV}
For any valid distributional Bellman operator $\gT^\pi$, the operator $(\gI-\gT^\pi)$ is invertible on $\prn{\overline{M^0_{\FTV}}}^\gS$ and $\prn{\gI-\gT^\pi}^{-1}=\sum_{i=0}^\infty \prn{\gT^\pi}^i$.
\end{lemma}

\begin{proof}[Lemma~\ref{Lemma_operator_invertible_TV}]
     It suffices to verify that the Neumann series $\sum_{i=0}^\infty\prn{\gT^\pi}^i$ converges in $\prn{\overline{M^0_{\FTV}}}^\gS$.
    First, we claim for any $\nu\in \prn{M^0_{\FTV}}^\gS$, $\brc{\sum_{i=1}^k\prn{\gT^\pi}^i\nu,k=1,2,\dots}$ is a Cauchy sequence.
    WLOG, for any $\nu\in \prn{M^0_{\FTV}}^\gS$, $s\in\gS$, we may write $\nu=a_\nu(s)\prn{\nu_+(s)-\nu_-(s)}$, where $a_\nu(s)$ is a positive constant and $\nu_+(s),\nu_-(s)\in\Delta\prn{[0,\frac{1}{1-\gamma}]}$.
    Note that $\norm{\nu(s)}_{\FTV}=a_\nu(s) \TV(\nu_+(s),\nu_-(s))$ for any $s\in\gS$.
    There also exists a positive constant $M_\nu$ such that for any $s\in\gS$, both $\nu_+(s)$ and $\nu_-(s)$ have a density in $H_1^1(\RB)$ and $\norm{\nu_+(s)}_{H_1^1}\leq M_\nu$, $\norm{\nu_-(s)}_{H_1^1}\leq M_\nu$.
    For $k_1<k_2$, we have for any $s\in\gS$
    \begin{equation*}
        \begin{aligned}
            \norm{\sum_{i=k_1}^{k_2}\brk{\prn{\gT^\pi}^i\nu}(s)}_{\FTV}&\leq\sum_{i=k_1}^{k_2}\norm{\brk{\prn{\gT^\pi}^i\nu}(s)}_{\FTV}\\
            &=\sup_{s\in\gS}a_\nu(s)\sum_{i=k_1}^{k_2}\TV\prn{\brk{\prn{\gT^\pi}^i\nu_+}(s),\brk{\prn{\gT^\pi}^i\nu_-}(s)}\\
            &\leq \sup_{s\in\gS}a_\nu(s)\sum_{i=k_1}^{k_2}\sqrt{2KM_\nu W_1\prn{\prn{\gT^\pi}^i\nu_+(s),\prn{\gT^\pi}^i\nu_-(s)}}\\
            &\leq \sup_{s\in\gS}a_\nu(s)\sum_{i=k_1}^{k_2}\sqrt{2KM_\nu\gamma^i\sup_{s\in\gS}W_1\prn{\nu_+(s),\nu_-(s)}}\\
            &\leq \frac{\gamma^{k_1/2}\sup_{s\in\gS}a_\nu(s)\sqrt{2KM_\nu\sup_{s\in\gS}W_1\prn{\nu_+(s),\nu_-(s)}}}{1-\sqrt{\gamma}}.
        \end{aligned}
    \end{equation*}
    The second inequality is by Proposition~\ref{Proposition_bound_TV_with_W1} and
    the third inequality is due to Proposition~\ref{Proposition_value_iteration_Wp}.
    Then we define $\nu_*$ such that for any $f\in\FTV$, $\nu_*(s)f:=\lim_{k\to\infty}\sum_{i=0}^k\brk{\prn{\gT^\pi}^i\nu}(s)f$. 
    $\nu_*$ is well-defined because $\brc{\sum_{i=1}^k\prn{\gT^\pi}^i\nu(s) f,k=1,2,\dots}$ is a Cauchy sequence in $\RB$ by similar arguments as above.
    As 
    \begin{equation*}
        \begin{aligned}
            \norm{\nu_*(s)-\sum_{i=0}^{k}\brk{\prn{\gT^\pi}^i\nu}(s)}_{\FTV}
            &=\norm{\sum_{i=k}^{\infty}\brk{\prn{\gT^\pi}^i\nu}(s)}_{\FTV}\\
            &\leq \sum_{i=k}^\infty \norm{\brk{\prn{\gT^\pi}^i\nu}(s)}_{\FTV}\\
            &\leq \frac{\gamma^{k/2}\sup_{s\in\gS}a_\nu(s)\sqrt{2KM_\nu\sup_{s\in\gS}W_1\prn{\nu_+(s),\nu_-(s)}}}{1-\sqrt{\gamma}},
        \end{aligned}
    \end{equation*}
    $\nu_*=
    \lim_{k\to\infty}\sum_{i=1}^k\prn{\gT^\pi}^i\nu$.
    Apparently $\nu_*\in \prn{\overline{M^0_{\FTV}}}^\gS$.
\end{proof}

We have $\prn{\gI-\gT^\pi}^{-1}\mu\in\prn{\overline{M^0_{\FTV}}}^\gS$ that is a closed subspace of $\prn{\ell^\infty(\FTV)}^\gS$ for any $\mu\in \prn{M^0_{\FTV}}^\gS$.

Similar to the analysis of Theorem~\ref{Theorem_bound_of_KS}, we also use an ``expansion trick" to tackle the unboundedness issue of $\prn{\gI-\what\gT_n^\pi}^{-1}$.
For any $s\in\gS$,
\begin{equation*}
    \begin{aligned}
        &\norm{\hat\eta^\pi_n(s)-\eta^\pi(s)}_{\FTV}\\
        &=\norm{\brk{\prn{\gI-\what\gT_n^\pi}^{-1}\prn{\what\gT_n^\pi-\gT^\pi}\eta^\pi}(s)}_{\FTV}\\
        &=\norm{\brk{\sum_{i=0}^\infty \prn{\what\gT_n^\pi}^i\prn{\what\gT_n^\pi-\gT^\pi}\eta^\pi}(s)}_{\FTV}\\
        &\leq \norm{\brk{\prn{\what\gT_n^\pi-\gT^\pi}\eta^\pi}(s)}_{\FTV} +\sum_{i=1}^\infty\norm{\brk{ \prn{\what\gT_n^\pi}^i\prn{\what\gT_n^\pi-\gT^\pi}\eta^\pi}(s)}_{\FTV}\\
        &\leq \norm{\brk{\prn{\what\gT_n^\pi-\gT^\pi}\eta^\pi}(s)}_{\FTV} +\sum_{i=1}^\infty\sqrt{K\norm{\brk{ \prn{\what\gT_n^\pi}^i\prn{\what\gT_n^\pi-\gT^\pi}\eta^\pi}(s)}_{\FW}\norm{\brk{ \prn{\what\gT_n^\pi}^i\prn{\what\gT_n^\pi-\gT^\pi}\eta^\pi}(s)}_{H_1^1}}\\
        &\leq \norm{\brk{\prn{\what\gT_n^\pi-\gT^\pi}\eta^\pi}(s)}_{\FTV} +\sum_{i=1}^\infty\sqrt{K\gamma^i\sup_{s^\prime\in\gS}\norm{\brk{ \prn{\what\gT_n^\pi-\gT^\pi}\eta^\pi}(s^\prime)}_{\FW}\norm{\brk{ \prn{\what\gT_n^\pi}^i\prn{\what\gT_n^\pi-\gT^\pi}\eta^\pi}(s)}_{H_1^1}}\\
        &\leq \norm{\brk{\prn{\what\gT_n^\pi-\gT^\pi}\eta^\pi}(s)}_{\FTV} +\sqrt{K\sup_{s^\prime\in\gS}\norm{\brk{ \prn{\what\gT_n^\pi-\gT^\pi}\eta^\pi}(s^\prime)}_{\FW}}\sum_{i=1}^\infty\sqrt{\gamma^i\norm{\brk{ \prn{\what\gT_n^\pi}^i\prn{\what\gT_n^\pi-\gT^\pi}\eta^\pi}(s)}_{H_1^1}}.
    \end{aligned}
\end{equation*}

The second inequality is due to Proposition~\ref{Proposition_bound_TV_with_W1_signed}, and the third inequality is by the contraction property of $\what\gT_n^\pi$.

Next, we bound the term $\norm{\brk{ \prn{\what\gT_n^\pi}^i\prn{\what\gT_n^\pi-\gT^\pi}\eta^\pi}(s)}_{H_1^1}$.

\begin{lemma}\label{Lemma_bound_Sobolev_by_TV}
    Suppose Assumption~\ref{Assumption_reward_smooth} holds true.
    For any $i\geq1$, $s\in\gS$, 
    \begin{equation*}
        \norm{\brk{ \prn{\what\gT_n^\pi}^i\prn{\what\gT_n^\pi-\gT^\pi}\eta^\pi}(s)}_{H_1^1}\leq 2M\sup_{s\in\gS}\norm{\brk{\prn{\what\gT_n^\pi-\gT^\pi}\eta^\pi}(s)}_{\FTV}.
    \end{equation*}
\end{lemma}

Like the proof of Lemma~\ref{Lemma_bound_loc_by_TV}, we also deploy the normalization technique.
And the condition $i\geq 1$ is also necessary.

\begin{proof}[Proof of Lemma~\ref{Lemma_bound_Sobolev_by_TV}]
    We write the Jordan decomposition of $\prn{\what\gT_n^\pi-\gT^\pi}\eta^\pi$ as 
    \begin{equation*}
        \brk{\prn{\what\gT_n^\pi-\gT^\pi}\eta^\pi}(s)=a(s)(\mu_+(s)-\mu_-(s))
    \end{equation*}
    Here for any $s\in\gS$, $\mu_+(s),\mu_-(s)\in\Delta\prn{\brk{0,\frac{1}{1-\gamma}}}$ has probability density function $p_+(s),p_-(s)\in H_1^1(\RB)$. 
    And it is easy to verify $a(s)=\norm{ \brk{\prn{\what\gT_n^\pi-\gT^\pi}\eta^\pi}(s)}_{\FTV}$.
    For $i\geq 1$, we have
    \begin{equation*}
        \begin{aligned}
            \norm{\brk{\prn{\what \gT_n^\pi}^i\prn{\what\gT_n^\pi-\gT^\pi}\eta^\pi}(s)}_{H_1^1}&\leq\sup_{s\in\gS}a(s)\norm{\brk{\prn{\what \gT_n^\pi}^i(\mu_+-\mu_-)}(s)}_{H_1^1}\\
            &=\sup_{s\in\gS}a(s)\prn{\norm{\brk{\prn{\what \gT_n^\pi}^i\mu_+}(s)}_{H_1^1}+\norm{\brk{\prn{\what \gT_n^\pi}^i\mu_-}(s)}_{H_1^1}}\\
            &\leq 2\sup_{s\in\gS}a(s)M.
        \end{aligned}
    \end{equation*}
    The last inequality holds by Lemma~\ref{Lemma_smooth_density_close}.
% We complete the proof noting that $a(s)\leq \norm{ \brk{\prn{\what\gT_n^\pi-\gT^\pi}\eta^\pi}(s)}_{\FTV}$ as long as $i\geq 1$ due to the non-expansive property of distributional Bellman operators.
    
\end{proof}

Putting the pieces together, we have
\begin{equation*}
\begin{aligned}
     &\norm{\hat\eta^\pi_n(s)-\eta^\pi(s)}_{\FTV}\\
     &\leq\frac{\sqrt{2KM\gamma
     }}{1-\sqrt{\gamma}}\sqrt{\sup_{s^\prime\in\gS}\norm{\brk{\prn{\what\gT_n^\pi-\gT^\pi}\eta^\pi}(s^\prime)}_{\FTV}\sup_{s^\prime\in\gS}\norm{\brk{\prn{\what\gT_n^\pi-\gT^\pi}\eta^\pi}(s^\prime)}_{\FW}}+\norm{\brk{\prn{\what\gT_n^\pi-\gT^\pi}\eta^\pi}(s)}_{\FTV}.
\end{aligned}
\end{equation*}
According to Lemma~\ref{Lemma_concentration_operator_W1} and Corollary~\ref{Corollary_concentration_operator_TV_Sobolev},
\begin{equation*}
\begin{aligned}
        \norm{\hat\eta^\pi_n(s)-\eta^\pi(s)}_{\FTV}
        &\leq M\prn{\frac{\sqrt{2K\gamma}}{1-\sqrt{\gamma}}+1}\frac{\sqrt{9\log|\gS|}+\sqrt{\log(2/\delta)/2}}{\sqrt{n(1-\gamma)^2}}\\
        % &\leq \frac{\sqrt{KM}}{1-\sqrt{\gamma}} \frac{(1+M)\prn{\sqrt{9\log|\gS|}+\sqrt{\log(2/\delta)/2}}}{\sqrt{n(1-\gamma)^2}}+\frac{M\prn{\sqrt{9\log|\gS|}+\sqrt{\log(2/\delta)/2}}}{\sqrt{n(1-\gamma)^2}}\\
    &\leq \frac{K^\prime\prn{\sqrt{\log|\gS|}+\sqrt{\log(1/\delta)}}}{\sqrt{n(1-\gamma)^4}}
\end{aligned}
\end{equation*}
with probability at least $1-\delta$, $\forall \delta\in(0,1)$.
$K^\prime$ is an absolute constant depending only on $M$ in Assumption~\ref{Assumption_reward_smooth}.
Besides, 
\begin{equation*}
    \begin{aligned}
        \EB\sup_{s\in\gS}\norm{\hat\eta^\pi_n(s)-\eta^\pi(s)}_{\FTV}&\leq K^\prime\sqrt{\frac{\log|\gS|}{n(1-\gamma)^4}}+\int_{0}^\infty \PB\prn{\sup_{s\in\gS}\norm{\hat\eta^\pi_n(s)-\eta^\pi(s)}_{\FTV}>K^\prime\sqrt{\frac{\log|\gS|}{n(1-\gamma)^4}}+t}dt\\
        &\leq K^{\prime\prime}\sqrt{\frac{\log|\gS|}{n(1-\gamma)^4}},
    \end{aligned}
\end{equation*}
where $K^{\prime\prime}$ is some constant depending on $M$ in Assumption~\ref{Assumption_reward_smooth}.

\section{Analysis of Theorem~\ref{Theorem_weak_convergence_empirical_process}}
\label{Appendix_analysis_asymptotic}

\paragraph{Weak Convergence in \texorpdfstring{$\ell^\infty(\gF_{W_1})$}{l(FW1)}}
We have
\begin{equation*}
\begin{aligned}
    \sqrt{n}(\hat\eta_n^\pi-\eta^\pi)&=\sqrt{n}\prn{\what \gT_n^\pi\hat\eta_n^\pi-\gT^\pi\eta^\pi}\\
    &=\sqrt{n}\prn{\what \gT_n^\pi\hat\eta_n^\pi-\what \gT_n^\pi\eta^\pi+\what \gT_n^\pi\eta^\pi-\gT^\pi\eta^\pi}\\
    &=\sqrt{n}\what \gT_n^\pi(\hat\eta_n^\pi-\eta^\pi)+\sqrt{n}\prn{\what \gT_n^\pi-\gT^\pi}\eta^\pi\\
    &=\sqrt{n}\gT^\pi(\hat\eta_n^\pi-\eta^\pi)+\sqrt{n}\prn{\what \gT_n^\pi-\gT^\pi}\eta^\pi+\sqrt{n}\prn{\what \gT_n^\pi-\gT^\pi}(\hat\eta_n^\pi-\eta^\pi).\\
\end{aligned}
\end{equation*}
Rearranging terms yields
\begin{equation*}
    \sqrt{n}\prn{\gI-\gT^\pi}(\hat\eta_n^\pi-\eta^\pi)=\underbrace{\sqrt{n}\prn{\what \gT_n^\pi-\gT^\pi}\eta^\pi}_{(1)}+\underbrace{\sqrt{n}\prn{\what \gT_n^\pi-\gT^\pi}(\hat\eta_n^\pi-\eta^\pi)}_{(2)}.
\end{equation*}
Both term $(1)$ and term $(2)$ are in $\prn{M^0_{\FW}}^\gS$.
Next, we can show that term $(1)$ converges weakly to a mixture of probability distributions and term $(2)$ is negligible.
\begin{lemma}\label{Lemma_convergence_difference_operator}
For any $s\in\gS$, $\sqrt{n}\brk{\prn{\what \gT_n^\pi-\gT^\pi}\eta^\pi}(s)$ converge weakly to the process $f\mapsto \wtilde\GB^\pi(s)f$ in $\ell^\infty(\FW)$.
Here the random element $\wtilde\GB^\pi$ is defined as
    \begin{equation*}
        \wtilde\GB^\pi(s):=\sum_{a\in\gA}\pi(a\mid s)\sum_{s^\prime\in\gS} Z_{s,a,s^\prime}\int_0^1 \prn{b_{r,\gamma}}_\#\eta^\pi(s^\prime)\gP_R(dr\mid s,a),\ \forall s\in\gS,
    \end{equation*}
    where  $\prn{Z_{s,a,s^\prime}}_{s\in\gS,a\in\gA,s^\prime\in\gS}$ are zero-mean gaussians with 
    \begin{equation*}
    \cov(Z_{s_1,a_1,s_1^\prime},Z_{s_2,a_2,s_2^\prime})=\ind\brc{(s_1,a_1)=(s_2,a_2)}P(s_1^\prime\mid s_1,a_1)\prn{\ind\brc{s_1^\prime=s_2^\prime}-P(s_2^\prime\mid s_1,a_1)}.
    \end{equation*}
\end{lemma}

\begin{proof}[Proof of Lemma~\ref{Lemma_convergence_difference_operator}]
For any $s\in\gS$, we have
\begin{equation*}
    \begin{aligned}
        \sqrt{n}\brk{\prn{\what \gT_n^\pi-\gT^\pi}\eta^\pi}(s)&=\sum_{a\in\gA}\pi(a\mid s)\sum_{s^\prime\in\gS} \sqrt{n}\prn{\what P(s^\prime\mid s,a)-P(s^\prime\mid s,a)}\int_0^1 \prn{b_{r,\gamma}}_\#\eta^\pi(s^\prime)\gP_R(dr\mid s,a)\\
        &=\sum_{a\in\gA,s^\prime\in\gS}\brk{\sqrt{n}\prn{\what P(s^\prime\mid s,a)-P(s^\prime\mid s,a)}}\brk{\pi(a\mid s)\int_0^1 \prn{b_{r,\gamma}}_\#\eta^\pi(s^\prime)\gP_R(dr\mid s,a)}
    \end{aligned}
\end{equation*}
    Thus, the conclusion follows through the multivariate CLT and Lemma~\ref{Lemma_simple_weak_convergence}.
\end{proof}

\begin{lemma}\label{Lemma_convergence_res_op1}
For any $s\in\gS$, we have $\norm{\sqrt{n}\brk{\prn{\what \gT_n^\pi-\gT^\pi}(\hat\eta_n^\pi-\eta^\pi)}(s)}_{\FW}=o_P\prn{1}$.
\end{lemma}

\begin{proof}[Proof of Lemma~\ref{Lemma_convergence_res_op1}]
    For any $s\in\gS$, we have
    \begin{equation*}
        \begin{aligned}
            &\norm{\sqrt{n}\brk{\prn{\what \gT_n^\pi-\gT^\pi}(\hat\eta_n^\pi-\eta^\pi)}(s)}_{\FW}\\
            &=\norm{\sum_{a\in\gA,s^\prime\in\gS}\brk{\sqrt{n}\prn{\what P(s^\prime\mid s,a)-P(s^\prime\mid s,a)}}\brk{\pi(a\mid s)\int_0^1 \brk{\prn{b_{r,\gamma}}_\#\eta^\pi(s^\prime)-\prn{b_{r,\gamma}}_\#\hat\eta_n^\pi(s^\prime)}\gP_R(dr\mid s,a)}}_{\FW}\\
            &\leq \norm{\sqrt{n}\prn{\what P(s^\prime\mid s,a)-P(s^\prime\mid s,a)}}_1\sup_{a\in\gA,s^\prime\in\gS}\norm{\int_0^1 \brk{\prn{b_{r,\gamma}}_\#\eta^\pi(s^\prime)-\prn{b_{r,\gamma}}_\#\hat\eta_n^\pi(s^\prime)}\gP_R(dr\mid s,a)}_{\FW}.
        \end{aligned}
    \end{equation*}
    Since $\norm{\sqrt{n}\prn{\what P(s^\prime\mid s,a)-P(s^\prime\mid s,a)}}_1$ is of the order $O_P(1)$, it suffices to show for any $s^\prime\in\gS, a\in\gA$, $\norm{\int_0^1 \brk{\prn{b_{r,\gamma}}_\#\eta^\pi(s^\prime)-\prn{b_{r,\gamma}}_\#\hat\eta_n^\pi(s^\prime)}\gP_R(dr\mid s,a)}_{\FW}=o_P(1)$.
    Noting that 
    \begin{equation*}
    \begin{aligned}
        &\norm{\int_0^1 \brk{\prn{b_{r,\gamma}}_\#\eta^\pi(s^\prime)-\prn{b_{r,\gamma}}_\#\hat\eta_n^\pi(s^\prime)}\gP_R(dr\mid s,a)}_{\FW}\\
        &=W_1\prn{\int_0^1 \prn{b_{r,\gamma}}_\#\eta^\pi(s^\prime)\gP_R(dr\mid s,a),\int_0^1 \prn{b_{r,\gamma}}_\#\hat\eta_n^\pi(s^\prime)\gP_R(dr\mid s,a)}.
    \end{aligned}
    \end{equation*}
    We claim that
    \begin{equation*}
    \begin{aligned}
        W_1\prn{\int_0^1 \prn{b_{r,\gamma}}_\#\eta^\pi(s^\prime)\gP_R(dr\mid s,a),\int_0^1 \prn{b_{r,\gamma}}_\#\hat\eta_n^\pi(s^\prime)\gP_R(dr\mid s,a)}\leq W_1\prn{\hat\eta_n^\pi(s^\prime),\eta^\pi(s^\prime)}.
    \end{aligned}
    \end{equation*}
    For simplicity of notations, we use $\hat\nu$ in short of $\int_0^1 \prn{b_{r,\gamma}}_\#\hat\eta_n^\pi(s^\prime)\gP_R(dr\mid s,a)$ and $\nu$ in short of $\int_0^1 \prn{b_{r,\gamma}}_\#\eta^\pi(s^\prime)\gP_R(dr\mid s,a)$.
    In fact, suppose two random variables $X\sim \eta^\pi(s^\prime)$, $Y\sim\hat\eta_n^\pi(s^\prime)$, and an independent random variable $R\sim \gP_R(dr\mid s,a)$, then $R+\gamma Y\sim \hat\nu$ and $R+\gamma X\sim \nu$.
    Then we have
    \begin{equation*}
    \begin{aligned}
        &W_1\prn{\int_0^1 \prn{b_{r,\gamma}}_\#\eta^\pi(s^\prime)\gP_R(dr\mid s,a),\int_0^1 \prn{b_{r,\gamma}}_\#\hat\eta_n^\pi(s^\prime)\gP_R(dr\mid s,a)}\\
        &=\inf_{W\sim \nu, Z\sim \hat \nu}\EB\abs{W-Z}\\
        &\leq \EB\abs{(R+\gamma X)-(R+\gamma Y)}\\
        &=\gamma\EB\abs{X-Y}.
    \end{aligned}
    \end{equation*}
    Our claim is true since $X$ and $Y$ are chosen arbitrarily.
    Our conclusion follows since by Theorem~\ref{Theorem_bound_of_Wdist} $W_1(\eta^\pi(s^\prime),\hat\eta^\pi_n(s^\prime))=o_P(1)$.
\end{proof}

Recall that we have previously demonstrated $\sqrt{n}\prn{\gI-\gT^\pi}(\hat\eta_n^\pi-\eta^\pi)\cweak \wtilde \GB^\pi$.
Thus our final step is to establish $\sqrt{n}(\hat\eta_n^\pi-\eta^\pi)\cweak \prn{\gI-\gT^\pi}^{-1}\wtilde \GB^\pi$.
This step can be accomplished by continuous mapping theorem as $\prn{\gI-\gT^\pi}^{-1}$ is a bounded operator.
Note that we always have $\sum_{s^\prime} Z_{s,a.s^\prime}=0$ for any state-action pair $(s,a)$, which implies $\wtilde\GB^\pi$ is also in $\prn{M^0_{\FW}}^\gS$.

\paragraph{Weak Convergence in \texorpdfstring{$\ell^\infty(\gF_{\mathsf{KS}})$}{l(FKS)}}

In the following decomposition
\begin{equation*}
    \sqrt{n}\prn{\gI-\gT^\pi}(\hat\eta_n^\pi-\eta^\pi)=\underbrace{\sqrt{n}\prn{\what \gT_n^\pi-\gT^\pi}\eta^\pi}_{(1)}+\underbrace{\sqrt{n}\prn{\what \gT_n^\pi-\gT^\pi}(\hat\eta_n^\pi-\eta^\pi)}_{(2)},
\end{equation*}
we also have the two relevant terms in $\prn{ M_{\FKS}^0}^\gS$ under Assumption~\ref{Assumption_reward_bounded_density}.
\begin{lemma}\label{Lemma_convergence_difference_operator_Ind}
Let Assumption~\ref{Assumption_reward_bounded_density} hold.
For any $s\in\gS$, $\sqrt{n}\brk{\prn{\what \gT_n^\pi-\gT^\pi}\eta^\pi}(s)$ converge weakly to the process $f\mapsto \wtilde\GB^\pi(s)f$ in $\ell^\infty(\FKS)$.
% Here the random element $\wtilde\GB^\pi$ is defined as
%     \begin{equation*}
%         \wtilde\GB^\pi(s):=\sum_{a\in\gA}\pi(a\mid s)\sum_{s^\prime\in\gS} Z_{s,a,s^\prime}\int_0^1 \prn{b_{r,\gamma}}_\#\eta^\pi(s^\prime)\gP_R(dr\mid s,a),\ \forall s\in\gS,
%     \end{equation*}
%     where  $\prn{Z_{s,a,s^\prime}}_{s\in\gS,a\in\gA,s^\prime\in\gS}$ are zero-mean gaussians with 
%     \begin{equation*}
%     \cov(Z_{s_1,a_1,s_1^\prime},Z_{s_2,a_2,s_2^\prime})=\ind\brc{(s_1,a_1)=(s_2,a_2)}P(s_1^\prime\mid s_1,a_1)\prn{\ind\brc{s_1^\prime=s_2^\prime}-P(s_2^\prime\mid s_1,a_1)}.
%     \end{equation*}
\end{lemma}

\begin{proof}[Proof of Lemma~\ref{Lemma_convergence_difference_operator_Ind}]
The proof is identical to that of Lemma~\ref{Lemma_convergence_difference_operator}.
For any $s\in\gS$, we have
\begin{equation*}
    \begin{aligned}
        \sqrt{n}\brk{\prn{\what \gT_n^\pi-\gT^\pi}\eta^\pi}(s)=\sum_{a\in\gA,s^\prime\in\gS}\brk{\sqrt{n}\prn{\what P(s^\prime\mid s,a)-P(s^\prime\mid s,a)}}\brk{\pi(a\mid s)\int_0^1 \prn{b_{r,\gamma}}_\#\eta^\pi(s^\prime)\gP_R(dr\mid s,a)}
    \end{aligned}
\end{equation*}
    Thus the conclusion follows via the multivariate CLT and Lemma~\ref{Lemma_simple_weak_convergence}.
\end{proof}

\begin{lemma}\label{Lemma_convergence_res_op1_Ind}
Let Assumption~\ref{Assumption_reward_bounded_density} hold.
For any $s\in\gS$, we have ${\norm{\sqrt{n}\brk{\prn{\what \gT_n^\pi-\gT^\pi}(\hat\eta_n^\pi-\eta^\pi)}(s)}_{\FKS}=o_P\prn{1}}$.
\end{lemma}

\begin{proof}[Proof of Lemma~\ref{Lemma_convergence_res_op1_Ind}]
    For any $s\in\gS$, we have
    \begin{equation*}
        \begin{aligned}
            &\norm{\sqrt{n}\brk{\prn{\what \gT_n^\pi-\gT^\pi}(\hat\eta_n^\pi-\eta^\pi)}(s)}_{\FKS}\\
            &=\norm{\sum_{a\in\gA,s^\prime\in\gS}\brk{\sqrt{n}\prn{\what P(s^\prime\mid s,a)-P(s^\prime\mid s,a)}}\brk{\pi(a\mid s)\int_0^1 \brk{\prn{b_{r,\gamma}}_\#\eta^\pi(s^\prime)-\prn{b_{r,\gamma}}_\#\hat\eta_n^\pi(s^\prime)}\gP_R(dr\mid s,a)}}_{\FKS}\\
            &\leq \norm{\sqrt{n}\prn{\what P(s^\prime\mid s,a)-P(s^\prime\mid s,a)}}_1\sup_{a\in\gA,s^\prime\in\gS}\norm{\int_0^1 \brk{\prn{b_{r,\gamma}}_\#\eta^\pi(s^\prime)-\prn{b_{r,\gamma}}_\#\hat\eta_n^\pi(s^\prime)}\gP_R(dr\mid s,a)}_{\FKS}.
        \end{aligned}
    \end{equation*}
    Since $\norm{\sqrt{n}\prn{\what P(s^\prime\mid s,a)-P(s^\prime\mid s,a)}}_1$ is of the order $O_P(1)$, it suffices to show for any $s^\prime\in\gS, a\in\gA$, $\norm{\int_0^1 \brk{\prn{b_{r,\gamma}}_\#\eta^\pi(s^\prime)-\prn{b_{r,\gamma}}_\#\hat\eta_n^\pi(s^\prime)}\gP_R(dr\mid s,a)}_{\FKS}=o_P(1)$.
    Noting that 
    \begin{equation*}
    \begin{aligned}
        &\norm{\int_0^1 \brk{\prn{b_{r,\gamma}}_\#\eta^\pi(s^\prime)-\prn{b_{r,\gamma}}_\#\hat\eta_n^\pi(s^\prime)}\gP_R(dr\mid s,a)}_{\FKS}\\
        &=\KS\prn{\int_0^1 \prn{b_{r,\gamma}}_\#\eta^\pi(s^\prime)\gP_R(dr\mid s,a),\int_0^1 \prn{b_{r,\gamma}}_\#\hat\eta_n^\pi(s^\prime)\gP_R(dr\mid s,a)}.
    \end{aligned}
    \end{equation*}
    We claim that
    \begin{equation*}
    \begin{aligned}
        \KS\prn{\int_0^1 \prn{b_{r,\gamma}}_\#\eta^\pi(s^\prime)\gP_R(dr\mid s,a),\int_0^1 \prn{b_{r,\gamma}}_\#\hat\eta_n^\pi(s^\prime)\gP_R(dr\mid s,a)}\leq \KS\prn{\hat\eta_n^\pi(s^\prime),\eta^\pi(s^\prime)}.
    \end{aligned}
    \end{equation*}
    % For simplicity of notations, we use $\hat\nu$ in short of $\int_0^1 \prn{b_{r,\gamma}}_\#\hat\eta_n^\pi(s^\prime)\gP_R(dr\mid s,a)$ and $\nu$ in short of $\int_0^1 \prn{b_{r,\gamma}}_\#\eta^\pi(s^\prime)\gP_R(dr\mid s,a)$.
    % In fact, for any $x\in\brk{0,\frac{1}{1-\gamma}}$,
    % \begin{equation*}
    % \begin{aligned}
    %     \abs{\hat\nu([0,x])-\nu([0,x])}&=\abs{\int_0^1 \prn{b_{r,\gamma}}_\#\hat\eta_n^\pi(s^\prime)([0,x])\gP_R(dr\mid s,a)-\int_0^1 \prn{b_{r,\gamma}}_\#\eta^\pi(s^\prime)([0,x])\gP_R(dr\mid s,a)}\\
    %     &=\abs{\int_0^1 \hat\eta_n^\pi(s^\prime)\prn{\brk{0,\frac{x-r}{\gamma}}}\gP_R(dr\mid s,a)-\int_0^1 \eta^\pi(s^\prime)\prn{\brk{0,\frac{x-r}{\gamma}}}\gP_R(dr\mid s,a)}\\
    %     &\leq \KS\prn{\hat\eta_n^\pi(s^\prime),\eta^\pi(s^\prime)}.
    % \end{aligned}
    % \end{equation*}
    The claim can be verified using the same argument as in the proof of Proposition~\ref{Proposition_value_iteration_KS} where we show the operator $\gT^\pi$ is non-expansive in supreme {\KS} distance.
    Our conclusion follows since by Theorem~\ref{Theorem_bound_of_KS} $\KS(\eta^\pi(s^\prime),\hat\eta^\pi_n(s^\prime))=o_P(1)$.
\end{proof}

Because $\prn{\gI-\gT^\pi}^{-1}$ is not bounded on the space $\prn{M^0_{\FKS}}^\gS$, we may not directly apply continuous mapping theorem here.
However, the limiting random element indeed lies in a finite-dimensional subspace of $\prn{M^0_{\FKS}}^\gS$.
To ensure technical rigor, we define $ \nu(s,a,s^\prime)=\int_0^1 \prn{b_{r,\gamma}}_\#\eta^\pi(s^\prime)\gP_R(dr\mid s,a)$, and
\begin{equation*}
   C_s\colon=\brc{\sum_{a\in\gA,s^\prime\in\gS}c_{a,s^\prime}\nu(s,a,s^\prime)\ \Bigg\vert\  c\in\RB^{\gS\times\gA},\sum_{s^\prime\in\gS} c_{a,s^\prime}=0, \forall a\in\gA},
\end{equation*}
and it is straightforward to verify $\wtilde \GB^\pi$ lies in $\bigtimes\limits_{s\in\gS}C_s$.
Since $\bigtimes\limits_{s\in\gS}C_s$ is finite-dimensional, the linear operator $\prn{\gI-\gT^\pi}^{-1}$ is continuous when confined on $\bigtimes\limits_{s\in\gS}C_s$.
Finally, we can obtain $\sqrt{n}(\hat\eta_n^\pi-\eta^\pi)\cweak \prn{\gI-\gT^\pi}^{-1}\wtilde \GB^\pi$.

\paragraph{Weak Convergence in \texorpdfstring{$\ell^\infty(\gF_{\mathsf{TV}})$}{l(FTV)}}

We again consider the following decomposition
\begin{equation*}
    \sqrt{n}\prn{\gI-\gT^\pi}(\hat\eta_n^\pi-\eta^\pi)=\underbrace{\sqrt{n}\prn{\what \gT_n^\pi-\gT^\pi}\eta^\pi}_{(1)}+\underbrace{\sqrt{n}\prn{\what \gT_n^\pi-\gT^\pi}(\hat\eta_n^\pi-\eta^\pi)}_{(2)}.
\end{equation*}
When Assumption~\ref{Assumption_reward_smooth} holds, we have the two relevant terms in $\prn{ M_{\FTV}^0}^\gS$.
The proof idea is nearly identical as before, we first demonstrate term $(1)$ converges weakly to a gaussian random element, then show term $(2)$ is asymptotically negligible.
\begin{lemma}\label{Lemma_convergence_difference_operator_TV}
Let Assumption~\ref{Assumption_reward_smooth} hold.
For any $s\in\gS$, $\sqrt{n}\brk{\prn{\what \gT_n^\pi-\gT^\pi}\eta^\pi}(s)$ converge weakly to the process $f\mapsto \wtilde\GB^\pi(s)f$ in $\ell^\infty(\FTV)$.
% Here the random element $\wtilde\GB^\pi$ is defined as
%     \begin{equation*}
%         \wtilde\GB^\pi(s):=\sum_{a\in\gA}\pi(a\mid s)\sum_{s^\prime\in\gS} Z_{s,a,s^\prime}\int_0^1 \prn{b_{r,\gamma}}_\#\eta^\pi(s^\prime)\gP_R(dr\mid s,a),\ \forall s\in\gS,
%     \end{equation*}
%     where  $\prn{Z_{s,a,s^\prime}}_{s\in\gS,a\in\gA,s^\prime\in\gS}$ are zero-mean gaussians with 
%     \begin{equation*}
%     \cov(Z_{s_1,a_1,s_1^\prime},Z_{s_2,a_2,s_2^\prime})=\ind\brc{(s_1,a_1)=(s_2,a_2)}P(s_1^\prime\mid s_1,a_1)\prn{\ind\brc{s_1^\prime=s_2^\prime}-P(s_2^\prime\mid s_1,a_1)}.
%     \end{equation*}
\end{lemma}
\begin{proof}[Proof of Lemma~\ref{Lemma_convergence_difference_operator_TV}]
The proof is identical to that of Lemma~\ref{Lemma_convergence_difference_operator}.
For any $s\in\gS$, we have
\begin{equation*}
    \begin{aligned}
        \sqrt{n}\brk{\prn{\what \gT_n^\pi-\gT^\pi}\eta^\pi}(s)=\sum_{a\in\gA,s^\prime\in\gS}\brk{\sqrt{n}\prn{\what P(s^\prime\mid s,a)-P(s^\prime\mid s,a)}}\brk{\pi(a\mid s)\int_0^1 \prn{b_{r,\gamma}}_\#\eta^\pi(s^\prime)\gP_R(dr\mid s,a)}
    \end{aligned}
\end{equation*}
    Thus the conclusion follows via the multivariate CLT and Lemma~\ref{Lemma_simple_weak_convergence}.
\end{proof}

\begin{lemma}\label{Lemma_convergence_res_op1_TV}
Let Assumption~\ref{Assumption_reward_smooth} hold.
For any $s\in\gS$, we have ${\norm{\sqrt{n}\brk{\prn{\what \gT_n^\pi-\gT^\pi}(\hat\eta_n^\pi-\eta^\pi)}(s)}_{\FTV}=o_P\prn{1}}$.
\end{lemma}
\begin{proof}[Proof of Lemma~\ref{Lemma_convergence_res_op1_TV}]
    For any $s\in\gS$, we have
    \begin{equation*}
        \begin{aligned}
            &\norm{\sqrt{n}\brk{\prn{\what \gT_n^\pi-\gT^\pi}(\hat\eta_n^\pi-\eta^\pi)}(s)}_{\FTV}\\
            &=\norm{\sum_{a\in\gA,s^\prime\in\gS}\brk{\sqrt{n}\prn{\what P(s^\prime\mid s,a)-P(s^\prime\mid s,a)}}\brk{\pi(a\mid s)\int_0^1 \brk{\prn{b_{r,\gamma}}_\#\eta^\pi(s^\prime)-\prn{b_{r,\gamma}}_\#\hat\eta_n^\pi(s^\prime)}\gP_R(dr\mid s,a)}}_{\FTV}\\
            &\leq \norm{\sqrt{n}\prn{\what P(s^\prime\mid s,a)-P(s^\prime\mid s,a)}}_1\sup_{a\in\gA,s^\prime\in\gS}\norm{\int_0^1 \brk{\prn{b_{r,\gamma}}_\#\eta^\pi(s^\prime)-\prn{b_{r,\gamma}}_\#\hat\eta_n^\pi(s^\prime)}\gP_R(dr\mid s,a)}_{\FTV}.
        \end{aligned}
    \end{equation*}
    Since $\norm{\sqrt{n}\prn{\what P(s^\prime\mid s,a)-P(s^\prime\mid s,a)}}_1$ is of the order $O_P(1)$, it suffices to show for any $s^\prime\in\gS, a\in\gA$, $\norm{\int_0^1 \brk{\prn{b_{r,\gamma}}_\#\eta^\pi(s^\prime)-\prn{b_{r,\gamma}}_\#\hat\eta_n^\pi(s^\prime)}\gP_R(dr\mid s,a)}_{\FTV}=o_P(1)$.
    Noting that 
    \begin{equation*}
    \begin{aligned}
        &\norm{\int_0^1 \brk{\prn{b_{r,\gamma}}_\#\eta^\pi(s^\prime)-\prn{b_{r,\gamma}}_\#\hat\eta_n^\pi(s^\prime)}\gP_R(dr\mid s,a)}_{\FTV}\\
        &=\TV\prn{\int_0^1 \prn{b_{r,\gamma}}_\#\eta^\pi(s^\prime)\gP_R(dr\mid s,a),\int_0^1 \prn{b_{r,\gamma}}_\#\hat\eta_n^\pi(s^\prime)\gP_R(dr\mid s,a)}.
    \end{aligned}
    \end{equation*}
    We claim that
    \begin{equation*}
    \begin{aligned}
        \TV\prn{\int_0^1 \prn{b_{r,\gamma}}_\#\eta^\pi(s^\prime)\gP_R(dr\mid s,a),\int_0^1 \prn{b_{r,\gamma}}_\#\hat\eta_n^\pi(s^\prime)\gP_R(dr\mid s,a)}\leq \TV\prn{\hat\eta_n^\pi(s^\prime),\eta^\pi(s^\prime)}.
    \end{aligned}
    \end{equation*}
    
    The claim can be verified using the same argument as in the proof of Proposition~\ref{Proposition_value_iteration_TV} where we show the operator $\gT^\pi$ is non-expansive in supreme \TV distance.
    Our conclusion follows by Theorem~\ref{Theorem_bound_of_TV} $\TV (\eta^\pi(s^\prime),\hat\eta^\pi_n(s^\prime))=o_P(1)$.
\end{proof}

The final step is also nearly identical to the $\ell^\infty(\FKS)$ case.
Since the limiting random element $\wtilde\GB^\pi$ lies in a finite-dimensional subspace of $\prn{M^0_{\FTV}}^\gS$, we can obtain $\sqrt{n}(\hat\eta_n^\pi-\eta^\pi)\cweak \prn{\gI-\gT^\pi}^{-1}\wtilde \GB^\pi$ via continuous mapping theorem.

\section{Omitted Proofs in Section~\ref{Section_prelim}}\label{Appendix_proof_prelim}
\begin{proposition}\label{Proposition_return_is_rv}
    For any policy $\pi$ and any initial state $s\in\gS$, $G^\pi(s)\colon\prn{\Omega,\gF,Q}\to\RB$ is a random variable.
    Here  $\Omega=\brc{\gS\times\gA\times[0,1]}^\NB$ is the sample space, $\gF=\brc{2^\gS\times 2^\gA\times\gB[0,1]}^\NB$ is the product $\sigma$-field where $\gB[0,1]$ denotes all Borel sets in $[0,1]$, and $Q$ is a probability measure induced by $\pi,\gP_R$ and $P$.
\end{proposition}
\begin{proof}
    First, we may note that the existence of $\gF$ and $Q$ is guaranteed by Kolmogorov's extension theorem.
    We simply need to verify that $\forall x\in\RB$, $\brc{G^\pi(s)\leq x}$ is an element of the product $\sigma$-algebra $\gF$.
    Let $G^\pi_H(s)=\sum_{t=0}^H \gamma^t R_t$, then we have
    \begin{equation*}
        \brc{G^\pi(s)\leq x}=\cap_{H=1}^\infty\brc{G^\pi_H(s)\leq x}.
    \end{equation*}
    Since $\forall H$, $\brc{G^\pi_H(s)\leq x}$ is an element of the product $\sigma$-algebra $\gF$, we also have $\brc{G^\pi(s)\leq x}\in\gF$.
\end{proof}

\begin{lemma}\label{Lemma_bounded_density_close}
    Suppose $\mu\in\Delta\prn{\brk{0,\frac{1}{1-\gamma}}}^\gS$ is a vector of distributions and $\mu(s)$ has density $p_s(\cdot)$.
    If Assumption~\ref{Assumption_reward_bounded_density} holds, for any $s\in\gS$, $\brk{\gT^\pi\mu}(s)$ has density $\tilde p_s(\cdot)$ such that $\sup_{x\in[0,1/(1-\gamma)]} \tilde p_s(x)\leq C$.
    Here $\gT^\pi$ can be replaced by any valid distributional Bellman operator.
\end{lemma}
\begin{proof}
    % For any $\mu\in\Delta([0,1/(1-\gamma)])^\gS$ with $p_s^\mu(\cdot)$ as the probability density function of $\mu(s)$, 
    By definition of $\gT^\pi$, for any $s\in\gS$,
    $\brk{\gT^\pi\mu}(s)$ also has probability density $\tilde p_s(\cdot)$ and
    \begin{equation*}
        \tilde p_s(x)=\sum_{a\in\gA}\pi(a\mid s)\sum_{s^\prime\in\gS}P(s^\prime\mid s,a) \prn{p_s(\cdot/\gamma)/\gamma\ast p^R_{s,a}(\cdot)}(x).
    \end{equation*}
    As for any $x\in[0,1/(1-\gamma)]$, $\abs{\prn{p_s(\cdot/\gamma)/\gamma\ast p^R_{s,a}(\cdot)}(x)}\leq \sup_{x\in[0,1/(1-\gamma)]}\abs{p_{s,a}^R(x)}$, we can get $ \sup_{x\in[0,1/(1-\gamma)]} \tilde{p_s}(x)\leq C$ under Assumption~\ref{Assumption_reward_bounded_density}.
\end{proof}

\begin{lemma}\label{Lemma_bounded_density_of_return}
    If Assumption~\ref{Assumption_reward_bounded_density} is true, $\forall s\in\gS$, $\eta^\pi(s)$ has density $p^G_{s}(x)$ being bounded from above by constant $C$.
\end{lemma}
\begin{proof}
    Define 
    \begin{equation*}
        G_H^\pi(s)=\sum_{t=0}^{H} \gamma^tR_t,
    \end{equation*}
    where $S_0=s$, $A_t|S_t\sim\pi(\cdot\mid S_t)$, $R_t\sim\gP_R(S_t,A_t)$ and ${{S_{t+1}}\mid{(S_t,A_t)}\sim P({\cdot}\mid{S_t,A_t})}$.
    The density of $G_H^\pi(s)$ can be written as
    \begin{equation*}
        p^{G_H}_s(x)=\sum_{\text{all possible length-H path $\tau$}}q(\tau)p^R_\tau(x)
    \end{equation*}
    with $x\in\left[0,\frac{1}{1-\gamma}\right]$.
    Suppose $\tau=(s,a_0,s_1,...,s_H,a_H)$, then 
    \begin{equation*}
        q(\tau)=\pi(a_0\mid s)P(s_1\mid s,a_0)\pi(a_1\mid s_1)P(s_2\mid s_1,a_1)\cdots P(s_H\mid s_{H-1},a_{H-1})\pi(a_H\mid s_H)
    \end{equation*}
    is the probability of sampling path $\tau$ and 
    \begin{equation*}
        p_\tau^R(x)= \brk{\prn{\prn{\prn{p^R_{s,a_0}(\cdot)\ast \frac{p^R_{s_1,a_1}(\gamma \cdot)}{\gamma}}\ast \frac{p^R_{s_2,a_2}(\gamma^2 \cdot)}{\gamma^2}}\cdots}\ast \frac{p^R_{s_H,a_H}(\gamma^H \cdot)}{\gamma^H}}(x)
    \end{equation*}
    % is the conditional density of r.v. $\sum_{t=0}^H\gamma^t R_t$ given $\tau$.
    is the density of r.v. $\sum_{t=0}^H\gamma^t \prn{R_t \mid \tau}$.
    Here $p^R_{s,a}(x)$ is the density of $\gP_R(dr\mid s,a)$ and $\prn{R_t \mid \tau}\sim \gP_R(\cdot \mid s_t,a_t)$.
    According to Lemma~\ref{Lemma_bounded_density_after_convolution}, $\sup_x\abs{p_\tau^R(x)}\leq C$.
    Thus $\sup_x\abs{p^{G_H}_s(x)}\leq C$, \ie~$G_H(s)$ has bounded density.
    This also implies 
    \begin{equation*}
        \abs{F_s^{G_H}(x)-F_s^{G_H}(y)}\leq C\abs{x-y},
    \end{equation*}
    \ie~the distribution function of $G_H(s)$ is $C$-Lipschitz continuous.
    Let $F^G_s(x)$ be the distribution function of $G^\pi(s)$, for any $x\geq y$, we have
    \begin{equation*}
        \begin{aligned}
        \abs{F^G_s(x)-F^G_s(y)}&=\PB(y<G^\pi(s)\leq x)\\
        &\leq \PB\prn{y-\frac{\gamma^H}{1-\gamma}<G^\pi_H(x)\leq x}\\
        &=\abs{F_s^{G_H}(x)-F_s^{G_H}\prn{y-\frac{\gamma^H}{1-\gamma}}}\\
        &\leq C\abs{x-y+\frac{\gamma^H}{1-\gamma}}\\
        &=C\prn{\abs{x-y}+\frac{\gamma^H}{1-\gamma}}.
        \end{aligned}
    \end{equation*}
    The first inequality is due to the definition of $G_H^\pi(s)$, the second inequality is due to the Lipschitz property of $F_s^{G_H}(x)$.
    Since $H$ is arbitrarily chosen, we have
    \begin{equation*}
        \abs{F^G_s(x)-F^G_s(y)}\leq C\abs{x-y},
    \end{equation*}
    \ie~$G^\pi(s)$ has $C$-Lipschitz continuous distribution function.
    As Lipschitz continuity leads to absolute continuity we know $F_s^G(x)$ is absolute continuous.
    Also, the absolute continuity of $F_s^G(x)$ is equivalent with that the distribution of $G^\pi(s)$ is absolute continuous w.r.t. the Lebesgue measure.
    By Radon-Nikodym theorem we can get the existence of $p_s^G(x)$.
    Apparently $\sup_x\abs{p_s^G(x)}\leq C$ by the $C$-Lipschitz property of $F^G_s(x)$.
\end{proof}

\begin{proof}[Proof of Proposition~\ref{Proposition_value_iteration_KS}.]
    For $\eta,\eta^\prime\in\Delta(\RB)^\gS$, any $f=\ind_A$, $A\in\gB(\RB)$, we have
    \begin{equation*}
    \begin{aligned}
        &\abs{\brk{\gT^\pi \eta}(s) f-\brk{\gT^\pi \eta^\prime}(s) f}
        \\&=\abs{\sum_{a\in\gA,s^\prime\in\gS}\pi(a\mid s)P(s^\prime\mid s,a)\brk{\int_0^1 \prn{b_{r,\gamma}}_\#\eta(s^\prime)d\gP_R\prn{dr\mid s,a}f -\int_0^1 \prn{b_{r,\gamma}}_\#\eta^\prime(s^\prime)d\gP_R\prn{dr\mid s,a}f}}\\
        &\leq \sup_{a\in\gA}\sup_{s^\prime\in\gS} \abs{\int_0^1 \prn{b_{r,\gamma}}_\#\eta(s^\prime)d\gP_R\prn{dr\mid s,a}f -\int_0^1 \prn{b_{r,\gamma}}_\#\eta^\prime(s^\prime)d\gP_R\prn{dr\mid s,a}f}\\
        &\leq\sup_{s^\prime\in\gS}\sup_{x\in\RB} \abs{\eta(s^\prime) \ind_{(-\infty,x]}-\eta^\prime(s^\prime)\ind_{(-\infty,x]}}\\
        &=\sup_{s^\prime\in\gS}\KS(\eta(s^\prime),\eta^\prime(s^\prime)).
    \end{aligned}
    \end{equation*}
    The last inequality is due to
    \begin{equation*}
        \begin{aligned}
            &\abs{\int_0^1 \prn{b_{r,\gamma}}_\#\eta(s^\prime)d\gP_R\prn{dr\mid s,a}f
            -\int_0^1 \prn{b_{r,\gamma}}_\#\eta^\prime(s^\prime)d\gP_R\prn{dr\mid s,a}f}\\
            &=\abs{\int_0^1 \brk{\prn{b_{r,\gamma}}_\#\eta(s^\prime)}\ind_{(-\infty,t]}d\gP_R\prn{dr\mid s,a}-\int_0^1 \brk{\prn{b_{r,\gamma}}_\#\eta^\prime(s^\prime)}\ind_{(-\infty,t]}d\gP_R\prn{dr\mid s,a}}\\
            &=\abs{\int_0^1 \eta(s^\prime)\ind_{(-\infty,(x-r)/\gamma]}d\gP_R\prn{dr\mid s,a}-\int_0^1 \eta^\prime(s^\prime)\ind_{(-\infty,(x-r)/\gamma]}d\gP_R\prn{dr\mid s,a}}\\
            &=\abs{\EB_{R\sim \gP_R(\cdot\mid s,a)}\eta(s^\prime)\ind_{(-\infty,(x-R)/\gamma]}-\EB_{R\sim \gP_R(\cdot\mid s,a)}\eta^\prime(s^\prime)\ind_{(-\infty,(x-R)/\gamma]}}\\
            &\leq \EB_{R\sim \gP_R(\cdot\mid s,a)}\abs{\eta(s^\prime)\ind_{(-\infty,(x-R)/\gamma]}-\eta^\prime(s^\prime)\ind_{(-\infty,(x-R)/\gamma]}}\\
            &\leq \sup_{x\in\RB} |\eta(s^\prime) \ind_{(-\infty,x]}-\eta^\prime(s^\prime)\ind_{(-\infty,x]}|.
        \end{aligned}
    \end{equation*}
    Here $g_{r,\gamma}(A)\colon=\brc{x\in\RB\colon r+\gamma x\in A}$ and $g_{r,\gamma}(A)\in\gB(\RB)$ as long as $A\in\gB(\RB)$.
    Therefore, we have shown that the operator $\gT^\pi$ is non-expansive in the supreme KS distance.
    Combining Lemma~\ref{Lemma_bounded_density_close}, Lemma~\ref{Lemma_bounded_density_of_return} and Proposition~\ref{Proposition_bound_KS_with_W1} we have for any $s\in\gS$, 
    \begin{equation*}    \KS(\eta^{(k)}(s),\eta^\pi(s))\leq\sqrt{2CW_1(\eta^{(k)}(s),\eta^\pi(s))}.
    \end{equation*}
    Combining with Proposition~\ref{Proposition_value_iteration_Wp}, 
    we complete the proof.
\end{proof}

% \begin{lemma}\label{Lemma_smooth_density_is_bounded}
% Suppose $p$ is the Lebesgue density of some probability distribution supported on $[0,1]$ and $p\in H_\alpha((0,1))$, then we have $\sup_{x\in[0,1]}p(x)\leq \frac{\norm{p}_{H_\alpha((0,1))}}{(\alpha-1)!}+1$.
% \end{lemma}
% \begin{proof}
%     For any $x\in[0,1]$, we have
%     \begin{equation*}
%         \begin{aligned}
%             |p(x)-p(x_0)|&=\abs{\int_{x_0}^x D^1 p(t_1)dt_1}\\
%             &\leq \int_{0}^x |D^1 p(t_1)|dt_1\\
%             &=\int_0^x \int_0^{t_1}|D^2p(t_2)| dt_2dt_1\\
%             &\leq {\int_0^x\int_0^{t_1}\cdots\int_0^{t_{\alpha-2}}\int_0^{t_{\alpha-1}}\abs{D^\alpha p(t_\alpha)}dt_\alpha dt_{\alpha-1}\dots dt_2 dt_1}\\
%             &\leq {\int_0^x\int_0^{t_1}\cdots\int_0^{t_{\alpha-2}}\norm{p}_{H_\alpha((0,1))} dt_{\alpha-1}\dots dt_2 dt_1}\\
%             &=\frac{x^{\alpha-1}\norm{p}_{H_\alpha((0,1))}}{(\alpha-1)!}\\
%             &\leq \frac{\norm{p}_{H_\alpha((0,1))}}{(\alpha-1)!}.
%         \end{aligned}
%     \end{equation*}
%     Thus we must have that $\sup_{x\in[0,1]}p(x)\leq \frac{\norm{p}_{H_\alpha((0,1))}}{(\alpha-1)!}+1$. 
%     Otherwise, $p(x)>1$, $\forall x\in[0,1]$, which contradicts the fact $\int_0^1 p(x)=1$. 
% \end{proof}

\begin{lemma}\label{Lemma_smooth_density_close}
    Suppose $\mu\in\Delta\prn{\brk{0,\frac{1}{1-\gamma}}}^\gS$ is a vector of distributions and $\mu(s)$ has density $p_s(\cdot)\in H^1_1\prn{\RB}$.
    Then under Assumption~\ref{Assumption_reward_smooth}, for any $s\in\gS$, $\brk{\gT^\pi\mu}(s)$ has density $\tilde {p_s}(\cdot)\in H^1_1\prn{\RB}$.
    Here $\gT^\pi$ can be replaced by any valid distributional Bellman operator.
\end{lemma}
\begin{proof}
    % Let $\mu^\prime=\gT^\pi\mu$. 
    By definition of $\gT^\pi$, for any $s\in\gS$,
    \begin{equation*}
        \tilde p_s(x)=\frac{1}{\gamma}\sum_{a\in\gA}\pi(a\mid s)\sum_{s^\prime\in\gS}P(s^\prime\mid s,a)\prn{p_{s,a}^R(\cdot)\ast p_{s^\prime}(\cdot/\gamma)}(x).
    \end{equation*}
    % Since we only assume $p_{s,a}^R(\cdot)\in H_1^1((0,1))$ and $p_{s,a}^R(\cdot)$ may not be continuous at $0$ or $1$, we abuse notation and define
    % \begin{equation*}
    %     \wtilde D^1p_{s,a}^R(\cdot)(x)=\begin{cases}
    %         D^1p_{s,a}^R(\cdot)(x)& x\in (0,1)\\
    %         % p_{s,a}^R(0)\delta(0) & x=0\\
    %         % -p_{s,a}^R(1)\delta(1) & x=1\\
    %         0 & \text{else}
    %     \end{cases}.
    % \end{equation*}
    % % where $\delta(0)$, $\delta(1)$ denotes dirac function at $0$, $1$, respectively.
    % Then we define 
    % \begin{equation*}
    %     D^1\prn{p_{s,a}^R(\cdot)\ast p_{s^\prime}(\cdot/\gamma)}(x)\colon=\prn{\wtilde D_1p_{s,a}^R(\cdot)\ast p_{s^\prime}(\cdot/\gamma)}(x).
    % \end{equation*}
    % Now we verify our definition is indeed the weak derivative of $p_{s,a}^R(\cdot)\ast p_{s^\prime}(\cdot/\gamma)$.
    % For any $f$ that is infinitely differentiable,
    % \begin{equation*}
    %     \begin{aligned}
    %         &\int D^1\prn{p_{s,a}^R(\cdot)\ast p_{s^\prime}(\cdot/\gamma)}(x) f(x)dx\\
    %         &=
    %     \end{aligned}
    % \end{equation*}
    
    \begin{equation*}
\begin{aligned}
    \norm{D^1 \tilde p_s}_1&\leq \frac{1}{\gamma}\sum_{a\in\gA}\pi(a\mid s)\sum_{s^\prime\in\gS} P(s^\prime\mid s,a)\norm{ D^1\prn{p_{s,a}^R(\cdot)\ast p_{s^\prime}(\cdot/\gamma)}}_1\\
    &=\frac{1}{\gamma}\sum_{a\in\gA}\pi(a\mid s)\sum_{s^\prime\in\gS} P(s^\prime\mid s,a)\norm{ \prn{D^1p_{s,a}^R(\cdot)}\ast p_{s^\prime}(\cdot/\gamma)}_1\\
    &\leq \frac{1}{\gamma}\sum_{a\in\gA}\pi(a\mid s)\sum_{s^\prime\in\gS} P(s^\prime\mid s,a)\norm{ D^1 p_{s,a}^R(\cdot)}_1\norm {p_{s^\prime}(\cdot/\gamma)}_1\\
    % &\leq \frac{M-1}{\gamma},
    &\leq M-1,
\end{aligned}
\end{equation*}
The first equality holds because $(f\ast g)^\prime=f^\prime\ast g$, the second inequality holds by Young's convolution inequality (Lemma~\ref{Lemma_Young_convolution_inequality}), and the last inequality holds due to $\norm{D^1 p_{s,a}^R}_1=\norm{p_{s,a}^R}_{H_1^1}-\norm{p_{s,a}^R}_1\leq M-1$.
Finally, we may conclude that $\norm{\tilde p_s}_{H_1^1}=\norm{D^1 \tilde p_s}_1+\norm{\tilde p_s}_1\leq M$.
% \begin{equation*}
%     \norm{\tilde p_s}_{H_1^1}=\norm{D^1 \tilde p_s}_1+\norm{\tilde p_s}_1\leq \frac{M-1}{\gamma}+1\leq \frac{M}{\gamma}.
% \end{equation*}
\end{proof}

\begin{lemma}\label{Lemma_smooth_density_of_return}
    If Assumption~\ref{Assumption_reward_smooth} is true, $\forall s\in\gS$, $\eta^\pi(s)$ has density $p^G_{s}(\cdot)\in H^1_1\prn{\RB}$.
    Specifically, we have
    \begin{equation*}
        % \norm{p_s^G}_{H_1^1}\leq \frac{M}{\gamma}.
        \norm{p_s^G}_{H_1^1}\leq M.
    \end{equation*}
\end{lemma}
\begin{proof}
% Combining Assumption~\ref{Assumption_reward_smooth} and Lemma~\ref{Lemma_smooth_density_is_bounded} yileds $\sup_{x\in[0,1]} p_{s,a}^R(x)\leq \frac{M}{(\alpha - 1)!}+1$ for any $s\in\gS$m $s\in\gA$, which further implies $\sup_{x\in[0,1]}p^G_s(x)\leq \frac{M}{(\alpha - 1)!}+1$ by Lemma~\ref{Lemma_bounded_density_of_return}.
For any $s\in\gS$, $p_s^G$ satisfies the following version of distributional Bellman equation:
\begin{equation*}
    p_s^G(x)=\frac{1}{\gamma}\sum_{a\in\gA}\pi(a\mid s)\sum_{s^\prime\in\gS} P(s^\prime\mid s,a)\prn{p_{s,a}^R(\cdot)\ast p_{s^\prime}^{G}(\cdot/\gamma)}(x).
\end{equation*}
Then we have 
\begin{equation*}
\begin{aligned}
    \norm{D^1 p_s^G}_1&\leq \frac{1}{\gamma}\sum_{a\in\gA}\pi(a\mid s)\sum_{s^\prime\in\gS} P(s^\prime\mid s,a)\norm{ D^1\prn{p_{s,a}^R(\cdot)\ast p_{s^\prime}^{G}(\cdot/\gamma)}}_1\\
    &=\frac{1}{\gamma}\sum_{a\in\gA}\pi(a\mid s)\sum_{s^\prime\in\gS} P(s^\prime\mid s,a)\norm{ \prn{D^1 p_{s,a}^R(\cdot)}\ast p_{s^\prime}^{G}(\cdot/\gamma)}_1\\
    &\leq \frac{1}{\gamma}\sum_{a\in\gA}\pi(a\mid s)\sum_{s^\prime\in\gS} P(s^\prime\mid s,a)\norm{ D^1 p_{s,a}^R(\cdot)}_1\norm {p_{s^\prime}^{G}(\cdot/\gamma)}_1\\
    % &\leq \frac{M-1}{\gamma}.
    &\leq M-1.
\end{aligned}
\end{equation*}
The first equality holds because $(f\ast g)^\prime=f^\prime\ast g$, the second inequality holds by Young's convolution inequality (Lemma~\ref{Lemma_Young_convolution_inequality}), and the last inequality holds due to $\norm{D^1 p_{s,a}^R}_1=\norm{p_{s,a}^R}_{H_1^1}-\norm{p_{s,a}^R}_1\leq M-1$.
Finally, we may conclude that $\norm{p_s^G}_{H_1^1}=\norm{D^1 p_s^G}_1+\norm{p_s^G}_1\leq M$.
% \begin{equation*}
%     \norm{p_s^G}_{H_1^1}=\norm{D^1 p_s^G}_1+\norm{p_s^G}_1\leq \frac{M-1}{\gamma}+1\leq \frac{M}{\gamma}.
% \end{equation*}
\end{proof}

\begin{proof}[Proof of Proposition~\ref{Proposition_value_iteration_TV}.]
    For $\eta,\eta^\prime\in\Delta(\RB)^\gS$, any $f=\ind_A$, $A\in\gB(\RB)$, we have
    \begin{equation*}
    \begin{aligned}
        &\abs{\brk{\gT^\pi \eta}(s) f-\brk{\gT^\pi \eta^\prime}(s) f}
        \\&=\abs{\sum_{a\in\gA,s^\prime\in\gS}\pi(a\mid s)P(s^\prime\mid s,a)\brk{\int_0^1 \prn{b_{r,\gamma}}_\#\eta(s^\prime)d\gP_R\prn{dr\mid s,a}f -\int_0^1 \prn{b_{r,\gamma}}_\#\eta^\prime(s^\prime)d\gP_R\prn{dr\mid s,a}f}}\\
        &\leq \sup_{a\in\gA}\sup_{s^\prime\in\gS} \abs{\int_0^1 \prn{b_{r,\gamma}}_\#\eta(s^\prime)d\gP_R\prn{dr\mid s,a}f -\int_0^1 \prn{b_{r,\gamma}}_\#\eta^\prime(s^\prime)d\gP_R\prn{dr\mid s,a}f}\\
        &\leq\sup_{s^\prime\in\gS}\sup_{A\in\gB(\RB)} |\eta(s^\prime) \ind_A-\eta^\prime(s^\prime)\ind_A|\\
        &=\sup_{s^\prime\in\gS}\TV(\eta(s^\prime),\eta^\prime(s^\prime)).
    \end{aligned}
    \end{equation*}
    The last inequality is due to
    \begin{equation*}
        \begin{aligned}
            &\abs{\int_0^1 \prn{b_{r,\gamma}}_\#\eta(s^\prime)d\gP_R\prn{dr\mid s,a}f
            -\int_0^1 \prn{b_{r,\gamma}}_\#\eta^\prime(s^\prime)d\gP_R\prn{dr\mid s,a}f}\\
            &=\abs{\int_0^1 \brk{\prn{b_{r,\gamma}}_\#\eta(s^\prime)}\ind_A d\gP_R\prn{dr\mid s,a}-\int_0^1 \brk{\prn{b_{r,\gamma}}_\#\eta^\prime(s^\prime)}\ind_A d\gP_R\prn{dr\mid s,a}}\\
            &=\abs{\int_0^1 \eta(s^\prime)\ind_{g_{r,\gamma}(A)}d\gP_R\prn{dr\mid s,a}-\int_0^1 \eta^\prime(s^\prime)\ind_{g_{r,\gamma}(A)}d\gP_R\prn{dr\mid s,a}}\\
            &=\abs{\EB_{R\sim \gP_R(\cdot\mid s,a)}\eta(s^\prime)\ind_{g_{R,\gamma}(A)}-\EB_{R\sim \gP_R(\cdot\mid s,a)}\eta^\prime(s^\prime)\ind_{g_{R,\gamma}(A)}}\\
            &\leq \EB_{R\sim \gP_R(\cdot\mid s,a)}\abs{\eta(s^\prime)\ind_{g_{r,\gamma}(A)}-\eta^\prime(s^\prime)\ind_{g_{r,\gamma}(A)}}\\
            &\leq \sup_{A\in\gB(\RB)} |\eta(s^\prime) \ind_{A}-\eta^\prime(s^\prime)\ind_{A}|.
        \end{aligned}
    \end{equation*}
    Therefore, we have shown that the operator $\gT^\pi$ is non-expansive in the supreme TV distance.
    For some $\eta\in\Delta\prn{\brk{0,\frac{1}{1-\gamma}}}$ with Lebesgue density, let $p_s^{\eta}$ denote its density function.
    When Assumption~\ref{Assumption_reward_smooth} holds true, we have 
    % $\norm{p^{\eta^\pi}_s}_{H_\alpha((0,1/(1-\gamma)))}\leq\frac{M}{\gamma}$, 
    $\norm{p^{\eta^\pi}_s}_{H_1^1}\leq M$, 
    $\forall s\in\gS$ by Lemma~\ref{Lemma_smooth_density_of_return} and 
% $\norm{p^{\eta^{(k)}}_s}_{H_\alpha((0,1/(1-\gamma)))}\leq\frac{M}{\gamma}$ 
$\norm{p^{\eta^{(k)}}_s}_{H_1^1}\leq M$ 
by Lemma~\ref{Lemma_smooth_density_close}.
    Therefore, for any $s\in\gS$, 
    % \begin{equation*}    \TV(\eta_k(s),\eta^\pi(s))\leq K\prn{\frac{2M}{\gamma}}^{\frac{1}{\alpha+1}}W_1(\eta_k(s),\eta^\pi(s))^{\frac{\alpha}{\alpha+1}}.
    % \end{equation*}
        \begin{equation*}    \TV(\eta^{(k)}(s),\eta^\pi(s))\leq \sqrt{2MKW_1(\eta^{(k)}(s),\eta^\pi(s))}.
    \end{equation*}
    Combining with Proposition~\ref{Proposition_value_iteration_Wp}, 
    we complete the proof.
\end{proof}

\section{Omitted Proofs and Examples in Section~\ref{Section_analysis}}\label{Appendix_proof_analysis}
\begin{proof}[Proof of Corollary~\ref{Corollary_bound_of_Wdist_p}]
\begin{equation*}
    \begin{aligned}
\max_{s\in\gS} W_p(\hat\eta_n^\pi(s),\eta^\pi(s))\leq \brk{\frac{1}{(1-\gamma)^{p-1}}\max_{s\in\gS}W_1(\hat\eta_n^\pi(s),\eta^\pi(s))}^{\frac{1}{p}}.
\end{aligned}
\end{equation*}
Therefore we have
\begin{equation*}
    \begin{aligned}
\EB\max_{s\in\gS} W_p(\hat\eta_n^\pi(s),\eta^\pi(s))&\leq \frac{1}{(1-\gamma)^{1-\frac{1}{p}}}\EB\brk{\max_{s\in\gS}W_1(\hat\eta_n^\pi(s),\eta^\pi(s))}^{\frac{1}{p}}\\
&\leq \frac{1}{(1-\gamma)^{1-\frac{1}{p}}} \brk{\EB \max_{s\in\gS}W_1(\hat\eta_n^\pi(s),\eta^\pi(s))}^{\frac{1}{p}}\\
&\leq \brk{\frac{9\log |\gS|}{n(1-\gamma)^{2p+2}}}^{\frac{1}{2p}}
\end{aligned}
\end{equation*}
and for any $\delta>0$, with probability at least $1-\delta$,
\begin{equation*}
    \begin{aligned}
\max_{s\in\gS} W_p(\hat\eta_n^\pi(s),\eta^\pi(s))&\leq \frac{1}{(1-\gamma)^{1-\frac{1}{p}}}\brk{\max_{s\in\gS}W_1(\hat\eta_n^\pi(s),\eta^\pi(s))}^{\frac{1}{p}}\\
&\leq \brk{\frac{\sqrt{9\log|\gS|}+\sqrt{\log(1/\delta)/2}}{\sqrt{n(1-\gamma)^{p+1}}}}^{\frac{1}{p}}.
\end{aligned}
\end{equation*}
\end{proof}

\begin{proof}[Proof of Theorem~\ref{Theorem_offline_non_asymp}.]
Let $n(s,a)=\sum_{i=1}^m\ind\{(s_i,a_i)=(s,a)\}$.
By Theorem \ref{Lemma_Chernoff_Inequality}, we have for any fixed $(s,a)\in \gS\times \gA$
$$
\begin{aligned}
\PB(n(s,a)<m\xi_{\min}/2)&\leq \PB(n(s,a)<m\xi(s,a)/2)\\
&\leq e^{-m\xi(s,a)}\left(\frac{em\xi(s,a)}{em\xi(s,a)/2}\right)^{\frac{m\xi(s,a)}{2}}\\
&=\left(\sqrt{\frac{e}{2}}\right)^{-\xi(s,a)m}\\
&\leq \left(\sqrt{\frac{e}{2}}\right)^{-\xi_{\min}m}
\end{aligned}
$$
Let $m\geq 8\log(2|\gS||\gA|/\delta)/\xi_{\min}$, we have $\PB(n(s,a)\geq m\xi_{\min})\geq 1-\delta/2|\gS||\gA|$.
Therefore, with probability at least $1-\delta/2$, we can get
$$
n(s,a)>m\xi_{\min}, \forall (s,a)\in \gS\times \gA.
$$
Then our problem is reduced to the case where the offline dataset is generated by a generative model.
The proof is completed by using results in Theorem~\ref{Theorem_bound_of_Wdist}, Theorem~\ref{Theorem_bound_of_KS} and Theorem~\ref{Theorem_bound_of_TV}.
\end{proof}

\begin{example}\label{Example_binary_tree}
    We consider a complete binary tree with $2^H$ nodes.
    Then we may define an MRP with $\gS=\{\text{nodes of the binary tree}\}$.
    At each node $s$ we let $\gP_R(s)=\delta_1$ if it is the right child of its parent, and $\gP_R(s)=\delta_0$ otherwise.
    Suppose $s$ is a leaf node, we set $P(\text{terminal}\mid s)=1$.
    Otherwise, we set $P(\text{left child of $s$}\mid s)=P(\text{right child of $s$}\mid s)=0.5$.
    $\gamma$ is set to be $0.5$.
    Suppose we want to perform policy evaluation in this MRP with dataset $\{(s_i,s_i^\prime)\}_{i=1}^m$, where $s_i\sim \xi$ and $s_i^\prime\sim P(\cdot\mid s_i)$.
    Define $n(s)=\sum_{i=1}^m\ind\{s_i=s\}$, it is straightforward to see for some small constant $\varepsilon$, with constant probability $W_1(\eta(s),\hat\eta(s))\geq \varepsilon/\sqrt{n(s)}$.
    Therefore, there also exist some small constant $\varepsilon^\prime$ such that with constant probability we have $\sup_{s\in\gS}W_1(\eta(s),\hat\eta(s))\geq \varepsilon^\prime/\sqrt{\xi_{\min}m}$.
\end{example}

\begin{proof}[Proof of Theorem~\ref{Theorem_unknown_reward_non_asymp}.]
\textbf{Proof of part (a).} We begin by proving (a), mirroring the methodology applied in the proof of Theorem~\ref{Theorem_bound_of_Wdist} in Appendix~\ref{Appendix_analysis_W1}.
\begin{equation*}
    \begin{aligned}
        \sup_{s\in\gS}W_1(\tilde\eta^\pi_n(s),\eta^\pi(s))&=\sup_{s\in\gS} \norm{\tilde\eta^\pi_n(s)-\eta^\pi(s)}_{\FW}\\
        &=\sup_{s\in\gS}\norm{\brk{\prn{\gI-\wtilde \gT_n^\pi}^{-1}\prn{\wtilde\gT_n^\pi-\gT^\pi}\eta^\pi}(s)}_{\FW}\\
        &\leq \frac{1}{1-\gamma}\sup_{s\in\gS}\norm{\brk{\prn{\wtilde\gT_n^\pi-\gT^\pi}\eta^\pi}(s)}_{\FW}.
    \end{aligned}
\end{equation*}
We define $\wtilde{F}_s$ as the cumulative distribution function of $\brk{\wtilde\gT^\pi \eta^\pi}(s)$.
Apparently $\EB\wtilde F_s(x)=F_s(x)$ for each $x\in\brk{0,\frac{1}{1-\gamma}}$.
Note that
\begin{equation*}
    \begin{aligned}
    \wtilde F_s(x)- F_s(x)=\sum_{a\in\gA} \pi(a\mid s)\frac{1}{n}\sum_{i=1}^{n}\brk{F_{X_{i}^{(s,a)}}\prn{\frac{x-R_{i}^{(s,a)}}{\gamma}}-\EB F_{X_{i}^{(s,a)}}\prn{\frac{x-R_{i}^{(s,a)}}{\gamma}}},
    \end{aligned}
\end{equation*}
and 
\begin{equation*}
    \begin{aligned}
    \abs{F_{X_{i}^{(s,a)}}\prn{\frac{x-R_{i}^{(s,a)}}{\gamma}}-\EB F_{X_{i}^{(s,a)}}\prn{\frac{x-R_{i}^{(s,a)}}{\gamma}}}\leq 1,
    \end{aligned}
\end{equation*}
we can follow the proof of Lemma~\ref{Lemma_cdf_sub_gaussian_1} and easily show that $\sqrt{n}(\wtilde F_s(x)-F_s(x))$ is $\frac{1}{\sqrt{n}}$-sub-gaussian with zero mean, therefore,
\begin{equation*}
    \begin{aligned}
        \EB  \sup_{s\in\gS}\norm{\brk{\prn{\wtilde \gT_n^\pi-\gT^\pi}\eta^\pi}(s)}_{\FW}\leq \sqrt{\frac{9\log|\gS|}{n(1-\gamma)^2}}.
    \end{aligned}
\end{equation*}
Following the proof of Lemma B.1, to arrive at the desired conclusion, we need to employ McDiarmid's inequality (Lemma~\ref{Lemma_McDiarmid_inequality}).
Note that for any fixed $i\in\brc{1,\dots,n}$, substituting data vector $\prn{X_i^{(s,a)},R_i^{(s,a)}}_{s\in\gS,a\in\gA}$ can change $\sup_{s\in\gS}\norm{\brk{\prn{\wtilde \gT_n^\pi-\gT^\pi}\eta^\pi}(s)}_{\FW}$ by at most $\frac{1}{n(1-\gamma)}$.
We can obtain the desired conclusion.

\textbf{Proof of (b) and (c).}
Now, we proceed to prove (b) and (c).
Recall the certainty-equivalence estimator $\hat{\eta}_n^\pi$ with $\what{P}_n$ and ground-truth $\gP_{R}$.
According to the results in Theorem~\ref{Theorem_bound_of_KS} and Theorem~\ref{Theorem_bound_of_TV}, we have the non-asymptotic bounds for  $\sup_{s\in\gS}\KS(\hat\eta_n^\pi(s),\eta^\pi(s))$ and $\sup_{s\in\gS}\TV(\hat\eta_n^\pi(s),\eta^\pi(s))$.
Hence, we only need to deal with the distance between $\tilde{\eta}_n^\pi$ and $\hat{\eta}_n^\pi$.

Note that $\tilde{\eta}_n^\pi$ and $\hat{\eta}_n^\pi$ use the same estimated transition dynamics $\what{P}_n$, which is independent of the estimated reward distribution $\what{\gP}_{R,n}$.
It is easy to verify that, if we replace the conditional expectation w.r.t. $\what{P}_n$ with the ordinary expectation, we can replace $\what{P}_n$ with the ground-truth $P$ in the proof which will not affect the theoretical analysis. 

For the sake of brevity, we redefine $\tilde{\eta}^\pi_n$ as the certainty-equivalence estimator with ground-truth $P$ and estimated $\what\gP_{R,n}$, and we define the corresponding distributional Bellman operator $\widetilde{\gT}^\pi$.
At this time, we can turn to analyze the distance between $\tilde{\eta}_n^\pi$ and $\eta^\pi$.

We deal with (b) first.
\begin{equation*}
    \begin{aligned}
        &\norm{\tilde\eta^\pi_n(s)-\eta^\pi(s)}_{\FKS}\\
        &=\norm{\brk{\prn{\gI-\wtilde\gT_n^\pi}^{-1}\prn{\what\gT_n^\pi-\gT^\pi}\eta^\pi}(s)}_{\FKS}\\
        &=\norm{\brk{\sum_{i=0}^\infty \prn{\wtilde\gT_n^\pi}^i\prn{\wtilde\gT_n^\pi-\gT^\pi}\eta^\pi}(s)}_{\FKS}\\
        &\leq \norm{\brk{\prn{\wtilde\gT_n^\pi-\gT^\pi}\eta^\pi}(s)}_{\FKS} +\sum_{i=1}^\infty\norm{\brk{ \prn{\wtilde\gT_n^\pi}^i\prn{\wtilde\gT_n^\pi-\gT^\pi}\eta^\pi}(s)}_{\FKS}\\
        &\leq \norm{\brk{\prn{\wtilde\gT_n^\pi-\gT^\pi}\eta^\pi}(s)}_{\FKS} +\sum_{i=1}^\infty\sqrt{2\norm{\brk{ \prn{\wtilde\gT_n^\pi}^i\prn{\wtilde\gT_n^\pi-\gT^\pi}\eta^\pi}(s)}_{\FW}\norm{\brk{ \prn{\wtilde\gT_n^\pi}^i\prn{\wtilde\gT_n^\pi-\gT^\pi}\eta^\pi}(s)}_{\loc}}\\
        &\leq\norm{\brk{\prn{\wtilde\gT_n^\pi-\gT^\pi}\eta^\pi}(s)}_{\FKS} +\sum_{i=1}^\infty\sqrt{2\gamma^i\sup_{s^\prime\in\gS}\norm{\brk{ \prn{\wtilde\gT_n^\pi-\gT^\pi}\eta^\pi}(s^\prime)}_{\FW}\norm{\brk{ \prn{\wtilde\gT_n^\pi}^i\prn{\wtilde\gT_n^\pi-\gT^\pi}\eta^\pi}(s)}_{\loc}}\\
        &\leq \norm{\brk{\prn{\wtilde\gT_n^\pi-\gT^\pi}\eta^\pi}(s)}_{\FTV} +\\
        &\quad\sqrt{2\sup_{s^\prime\in\gS,a\in\gA}\norm{\what{\gP}_{R,n}(\cdot\mid s^\prime,a)-\gP_R(\cdot\mid s^\prime,a)}_{\FTV}}\sum_{i=1}^\infty\sqrt{\gamma^i\norm{\brk{ \prn{\wtilde\gT_n^\pi}^i\prn{\wtilde\gT_n^\pi-\gT^\pi}\eta^\pi}(s)}_{\loc}}\\
        &\leq\prn{\frac{4\sqrt{C}}{1-\gamma}+1}\sup_{s^\prime\in\gS,a\in\gA}\norm{\what{\gP}_{R,n}(\cdot\mid s^\prime,a)-\gP_R(\cdot\mid s^\prime,a)}_{\FTV}.
    \end{aligned}
\end{equation*}
Here the second inequality is due to Proposition~\ref{Proposition_bound_KS_with_W1_signed}; the third inequality is by the contraction property of $\what\gT_n^\pi$; the fourth inequality holds from the fact $\norm{\cdot}_{\FKS}\leq\norm{\cdot}_{\FTV}$, and
\begin{equation*}
    \begin{aligned}
        \brk{ \prn{\wtilde\gT_n^\pi-\gT^\pi}\eta^\pi}(s)=\sum_{a\in\gA,s^\prime\in\gS}\pi(a\mid s)P(s^\prime\mid s,a)\int_0^1 \prn{b_{r,\gamma}}_\#\eta^\pi(s^\prime)\prn{\what{\gP}_{R,n}(dr\mid s,a)-\gP_R(dr\mid s,a)},
    \end{aligned}
\end{equation*}
\begin{equation*}
    \begin{aligned}
        \sup_{s\in\gS}\norm{\brk{ \prn{\wtilde\gT_n^\pi-\gT^\pi}\eta^\pi}(s)}_{\FW}&\leq\sup_{s\in\gS}\sum_{a\in\gA,s^\prime\in\gS}\pi(a\mid s)P(s^\prime\mid s,a)\norm{\what{\gP}_{R,n}(\cdot\mid s,a)-\gP_R(\cdot\mid s,a)}_{\FW}\\
        &\leq \sup_{s\in\gS,a\in\gA}\norm{\what{\gP}_{R,n}(\cdot\mid s,a)-\gP_R(\cdot\mid s,a)}_{\FTV},
    \end{aligned}
\end{equation*}
here we use the fact that $W_1(\mu_1*\nu_1,\mu_2*\nu_2)\leq W_1(\mu_1,\mu_2)+W_1(\nu_1,\nu_2)$ and $\norm{\mu}_{\FW}\leq\norm{\mu}_{\FTV}$ for $\mu$ supported on $[0, 1]$;
and the last inequality is due to Lemma~\ref{Lemma_bound_loc_by_TV}, our assumption that $\what{\gP}_{R,n}(\cdot\mid s,a)$ is endowed with a Lebesgue density upper-bounded by $2C$, and 
\begin{equation*}
    \begin{aligned}
        \sup_{s\in\gS}\norm{\brk{ \prn{\wtilde\gT_n^\pi-\gT^\pi}\eta^\pi}(s)}_{\FTV}&\leq\sup_{s\in\gS}\sum_{a\in\gA,s^\prime\in\gS}\pi(a\mid s)P(s^\prime\mid s,a)\norm{\what{\gP}_{R,n}(\cdot\mid s,a)-\gP_R(\cdot\mid s,a)}_{\FTV}\\
        &\leq \sup_{s\in\gS,a\in\gA}\norm{\what{\gP}_{R,n}(\cdot\mid s,a)-\gP_R(\cdot\mid s,a)}_{\FTV},
    \end{aligned}
\end{equation*}
here we use the fact that $\TV(\mu_1*\nu_1,\mu_2*\nu_2)\leq\TV(\mu_1,\mu_2)+\TV(\nu_1,\nu_2)$.
% Now we need an upper bound for $\sup_{s\in\gS}\norm{\brk{ \prn{\wtilde\gT_n^\pi-\gT^\pi}\eta^\pi}(s)}_{\FTV}$.
% Note that

In summary, we arrive at the desired conclusion: there exists a constant $C^\prime$ which only depends on $C$ in Assumption~\ref{Assumption_reward_bounded_density}, such that with probability at least $1-\delta$
\begin{equation*}
    \begin{aligned}
        \sup_{s\in\gS}\KS(\tilde\eta_n^\pi(s),\eta^\pi(s))\leq \frac{C^\prime\prn{\sqrt{\log|\gS|}+\sqrt{\log(1/\delta)}}}{\sqrt{n(1-\gamma)^4}}+\frac{C^\prime}{1-\gamma}\sup_{s\in\gS,a\in\gA}\norm{\what{\gP}_{R,n}(\cdot\mid s,a)-\gP_R(\cdot\mid s,a)}_{\FTV}.
    \end{aligned}
\end{equation*}

Finally, let us proceed to prove (c).
\begin{equation*}
    \begin{aligned}
        &\norm{\tilde\eta^\pi_n(s)-\eta^\pi(s)}_{\FTV}\\
        &=\norm{\brk{\prn{\gI-\wtilde\gT_n^\pi}^{-1}\prn{\what\gT_n^\pi-\gT^\pi}\eta^\pi}(s)}_{\FTV}\\
        &=\norm{\brk{\sum_{i=0}^\infty \prn{\wtilde\gT_n^\pi}^i\prn{\wtilde\gT_n^\pi-\gT^\pi}\eta^\pi}(s)}_{\FTV}\\
        &\leq \norm{\brk{\prn{\wtilde\gT_n^\pi-\gT^\pi}\eta^\pi}(s)}_{\FTV} +\sum_{i=1}^\infty\norm{\brk{ \prn{\wtilde\gT_n^\pi}^i\prn{\wtilde\gT_n^\pi-\gT^\pi}\eta^\pi}(s)}_{\FTV}\\
        &\leq \norm{\brk{\prn{\wtilde\gT_n^\pi-\gT^\pi}\eta^\pi}(s)}_{\FTV} +\sum_{i=1}^\infty\sqrt{K\norm{\brk{ \prn{\wtilde\gT_n^\pi}^i\prn{\wtilde\gT_n^\pi-\gT^\pi}\eta^\pi}(s)}_{\FW}\norm{\brk{ \prn{\wtilde\gT_n^\pi}^i\prn{\wtilde\gT_n^\pi-\gT^\pi}\eta^\pi}(s)}_{H_1^1}}\\
        &\leq\norm{\brk{\prn{\wtilde\gT_n^\pi-\gT^\pi}\eta^\pi}(s)}_{\FTV} +\sum_{i=1}^\infty\sqrt{K\gamma^i\sup_{s^\prime\in\gS}\norm{\brk{ \prn{\wtilde\gT_n^\pi-\gT^\pi}\eta^\pi}(s^\prime)}_{\FW}\norm{\brk{ \prn{\wtilde\gT_n^\pi}^i\prn{\wtilde\gT_n^\pi-\gT^\pi}\eta^\pi}(s)}_{H_1^1}}\\
        &\leq \norm{\brk{\prn{\wtilde\gT_n^\pi-\gT^\pi}\eta^\pi}(s)}_{\FTV} +\sqrt{K\sup_{s^\prime\in\gS}\norm{\brk{ \prn{\wtilde\gT_n^\pi-\gT^\pi}\eta^\pi}(s^\prime)}_{\FTV}}\sum_{i=1}^\infty\sqrt{\gamma^i\norm{\brk{ \prn{\wtilde\gT_n^\pi}^i\prn{\wtilde\gT_n^\pi-\gT^\pi}\eta^\pi}(s)}_{H_1^1}}\\
        &\leq\prn{\frac{4\sqrt{MK}}{1-\gamma}+1}\sup_{s^\prime\in\gS}\norm{\brk{ \prn{\wtilde\gT_n^\pi-\gT^\pi}\eta^\pi}(s^\prime)}_{\FTV}\\
        &\leq \prn{\frac{4\sqrt{MK}}{1-\gamma}+1}\sup_{s^\prime\in\gS,a\in\gA}\norm{\what{\gP}_{R,n}(\cdot\mid s^\prime,a)-\gP_R(\cdot\mid s^\prime,a)}_{\FTV}.
    \end{aligned}
\end{equation*}
Here the second inequality is due to Proposition~\ref{Proposition_bound_TV_with_W1_signed}, the third inequality is by the contraction property of $\what\gT_n^\pi$, the fourth inequality holds from the fact $\norm{\cdot}_{\FW}\leq\norm{\cdot}_{\FTV}$, the fifth inequality is due to Lemma~\ref{Lemma_bound_Sobolev_by_TV} and our assumption that $\what{\gP}_{R,n}(\cdot\mid s,a)$ is endowed with a Lebesgue density $\hat p_{s,a}^R\in H_1^1(\RB)$ with $\norm{\hat p_{s,a}^R}_{H_1^1(\RB)}\leq 2M$.

Now we arrive at the desired conclusion: there exists a constant $K^\prime$ which only depends on $M$ in Assumption~\ref{Assumption_reward_smooth}, such that with probability at least $1-\delta$
\begin{equation*}
    \begin{aligned}
\sup_{s\in\gS}\TV(\tilde\eta_n^\pi(s),\eta^\pi(s))\leq \frac{K^\prime\prn{\sqrt{\log|\gS|}+\sqrt{\log(1/\delta)}}}{\sqrt{n(1-\gamma)^4}}+\frac{K^\prime}{1-\gamma}\sup_{s\in\gS,a\in\gA}\norm{\what{\gP}_{R,n}(\cdot\mid s,a)-\gP_R(\cdot\mid s,a)}_{\FTV}.
    \end{aligned}
\end{equation*}

\begin{remark}\label{Remark_density_estimator}
\textit{(a) Non-parametric density estimation.}

The problem of non-parametric density estimation can be formalized as follows: Given i.i.d. random variables $\brc{X_i}_{i=1}^n$ with law of Lebesgue density $f$, we need an estimator $f_n$ for $f$.
Kernel density estimation and wavelet density estimation are common non-parametric estimation methods. 
Due to the extensive notations required to introduce these methods and their theoretical results, we omit the introduction here.
For an concise introduction to these methods, please refer to Section 5.1 in \cite{gine2016mathematical}.

In Theorem~\ref{Theorem_unknown_reward_non_asymp} (b) and (c), we need an upper-bound for $\norm{\what{\gP}_{R,n}(\cdot\mid s,a)-\gP_R(\cdot\mid s,a)}_{\FTV}$.
This is in fact the $L_1$ error $\norm{f_n-f}_{L_1}$ in the density estimation literature.
Concentration bounds for $\norm{f_n-f}_{L_1}$ under various assumptions have been well-studied, for example, see Theorem 5.1.13 and Equation 5.39 in \cite{gine2016mathematical}.
It is noteworthy that the typical convergence rate of $\norm{f_n-f}_{L_1}$ is slower than the standard rate $\frac{1}{\sqrt{n}}$.

In Theorem~\ref{Theorem_unknown_reward_non_asymp} (b), we also require the Lebesgue density of $\what{\gP}_{R,n}(\cdot\mid s,a)$ to be upper-bounded by $2C$.
The condition holds if $\norm{\what{\gP}_{R,n}(\cdot\mid s,a)-\gP_R(\cdot\mid s,a)}_{\loc}\leq C$.
The error term is in fact the $L_\infty$ error $\norm{f_n-f}_{L_\infty}$ in the density estimation literature.
See also Theorem 5.1.13 in \cite{gine2016mathematical} for the concentration inequality for $\norm{f_n-f}_{L_\infty}$.
Hence, we can find a sufficiently large $N$ that depends only on $C$, such that the error term does not exceed $C$ when $n\geq N$.

In Theorem~\ref{Theorem_unknown_reward_non_asymp} (c), we also require that $\what{\gP}_{R,n}(\cdot\mid s,a)$ has a Lebesgue density $\hat p_{s,a}^R\in H_1^1(\RB)$ with $\norm{\hat p_{s,a}^R}_{H_1^1(\RB)}\leq 2M$.
The condition holds if $\norm{\hat p_{s,a}^R-p_{s,a}^R}_{H_1^1(\RB)}\leq M$.
See Proposition 5.1.10 in \cite{gine2016mathematical} for the concentration inequality for $\norm{\hat p_{s,a}^R-p_{s,a}^R}_{H_1^1(\RB)}$.
Hence, we can find a sufficiently large $N$ that depends only on $M$, such that the error term does not exceed $M$ when $n\geq N$.

\textit{(b) Parametric estimation.}

Here, we consider the case where $\gP_R(\cdot \mid s, a)$ belongs to a parametric family, in which case the standard convergence rate of $\frac{1}{\sqrt{n}}$ applies. 
For convenience, we consider the family of normal distributions with known variance $\sigma^2>0$ and unknown mean $\mu\in\RB$.
Although the normal distribution is unbounded, we can extend the following analysis to bounded distributions, such as beta distribution or truncated normal distribution.

Without loss of generality, we assume $\mu=0$ and $\sigma^2=1$.
Let $f_\mu$ be the density function of $N(\mu,\sigma^2)$ for any $\mu\in\RB$, $\brc{X_i}_{i=1}^n$ be i.i.d. random variables with law $N(0, 1)$, and $\hat{\mu}_n=\frac{1}{n}\sum_{i=1}^n X_i\sim N\prn{0, \frac{1}{n}}$.

We need to derive a concentration inequality for $\norm{f_0-f_{\hat{\mu}_n}}_{L_1}$.
\begin{equation*}
\begin{aligned}
        \norm{f_0-f_{\mu}}_{L_1}&=\frac{1}{\sqrt{2\pi}}\int_{-\infty}^{+\infty}\abs{e^{-\frac{1}{2}x^2}-e^{-\frac{1}{2}(x-\mu)^2}}dx\\
        &=\sqrt{\frac{2}{\pi}}\abs{\int_{\frac{\mu}{2}}^{+\infty}\prn{e^{-\frac{1}{2}(x-\mu)^2}-e^{-\frac{1}{2}x^2}}dx}\\
        &=\sqrt{\frac{2}{\pi}}\abs{\Phi\prn{\frac{\mu}{2}}-\Phi\prn{-\frac{\mu}{2}}}\\
        &=\sqrt{\frac{2}{\pi}}\prn{2\Phi\prn{\frac{\abs{\mu}}{2}}-1}\\
        &\leq \frac{\abs{\mu}}{\pi}.
\end{aligned}
\end{equation*}
Hence, the standard concentration inequality for $\abs{\hat{\mu}_n}$ can be translated to a concentration inequality for the $L_1$ error $\norm{f_0-f_{\hat{\mu}_n}}_{L_1}$.
Specifically, with probability at least $1-\delta$,
\begin{equation*}
    \norm{f_0-f_{\hat{\mu}_n}}_{L_1}\leq \frac{1}{\pi}\sqrt{\frac{2\log\prn{\frac{\delta}{2}}}{n}}.
\end{equation*}
In addition, it is easy to check that $\norm{f_{\hat{\mu}_n}}_{L_\infty}=\norm{f_0}_{L_\infty}$, and $\norm{f_{\hat{\mu}_n}}_{H_1^1(\RB)}=\norm{f_0}_{H_1^1(\RB)}$.
\end{remark}

\end{proof}
\begin{lemma}\label{Lemma_normalized_model_asymptotic_gaussian}
    Given dataset $\gD=\{(s_i,a_i,s_i^\prime)^m_{i=1}\}$, where $(s_i,a_i)\simiid\xi$ and $s_i^\prime\sim P(\cdot\mid s_i,a_i)$. 
    Define 
    \begin{equation*}
        \wtilde P(s^\prime\mid s,a)=\frac{\sum_{i=1}^m\ind\{(s_i,a_i,s_i^\prime)=(s,a,s^\prime)\}}{\xi(s,a)m}.
    \end{equation*}
    Then we have 
    \begin{equation*}
        \prn{\sqrt{\xi(s,a)m}\prn{\wtilde P(s^\prime\mid s,a)-P(s^\prime\mid s,a)}}_{(s,a,s^\prime)\in \gS\times\gA\times\gS}\cweak \prn{\mathring Z_{s,a,s^\prime}}_{(s,a,s^\prime)\in \gS\times\gA\times\gS},
    \end{equation*}
    where $\prn{\mathring Z_{s,a,s^\prime}}_{(s,a,s^\prime)\in\gS\times\gA\times\gS}$ is zero mean Gaussian random vector with
    \begin{equation*}
    \cov(\mathring Z_{s_1,a_1,s_1^\prime},\mathring Z_{s_2,a_2,s_2^\prime})=\ind\brc{(s_1,a_1)=(s_2,a_2)}P(s_1^\prime\mid s_1,a_1)\prn{\ind\brc{s_1^\prime=s_2^\prime}-\xi(s_1,a_1)P(s_2^\prime\mid s_1,a_1)}.
    \end{equation*}
\end{lemma}
\begin{proof}
We have
\begin{equation*}
\begin{aligned}
    \sqrt{\xi(s,a)m}\prn{\wtilde P(s^\prime\mid s,a)-P(s^\prime\mid s,a)}&=\sqrt{\xi(s,a)m}\frac{\sum_{i=1}^m\ind\{(s_i,a_i,s^\prime_i)=(s,a,s^\prime)\}-\xi(s,a)P(s^\prime\mid s,a)\}}{\xi(s,a)m}\\
    &=\frac{1}{\sqrt{\xi(s,a)}}\underbrace{\frac{\sum_{i=1}^m\ind\{(s_i,a_i,s^\prime_i)=(s,a,s^\prime)\}-\xi(s,a)P(s^\prime\mid s,a)\}}{\sqrt{m}}}_{X_{s,a,s^\prime}}.
\end{aligned}
\end{equation*}
The result follows since $\prn{X_{s,a,s^\prime}}_{(s,a,s^\prime)\in\gS\times\gA\times\gS}\cweak\prn{\wtilde{Z}_{s,a,s^\prime}}_{(s,a,s^\prime)\in\gS\times\gA\times\gS}$ by multivariate CLT, where $\prn{\wtilde{Z}_{s,a,s^\prime}}_{(s,a,s^\prime)\in\gS\times\gA\times\gS}$ is zero mean Gaussian random vector with
    \begin{equation*}
    \cov(\wtilde{Z}_{s_1,a_1,s_1^\prime},\wtilde{Z}_{s_2,a_2,s_2^\prime})=\ind\brc{(s_1,a_1)=(s_2,a_2)}\xi(s_1,a_1)P(s_1^\prime\mid s_1,a_1)\prn{\ind\brc{s_1^\prime=s_2^\prime}-\xi(s_1,a_1)P(s_2^\prime\mid s_1,a_1)}.
    \end{equation*}
\end{proof}

\begin{proof}[Proof of Theorem~\ref{Theorem_weak_convergence_empirical_process_offline_data}]
We have for any $(s,a,s^\prime)\in\gS\times\gA\times\gS$,
\begin{equation*}
      \sqrt{\xi(s,a)m}\prn{\what P_m(s^\prime\mid s,a)-P(s^\prime\mid s,a)}=\sqrt{\xi(s,a)m}\prn{\wtilde P(s^\prime\mid s,a)-P(s^\prime\mid s,a)}\cdot \frac{\xi(s,a)m}{\sum_{i=1}^m\ind\{(s_i,a_i)=(s,a)\}}.
\end{equation*}
Since by Lemma~\ref{Lemma_normalized_model_asymptotic_gaussian},
\begin{equation*}
    \prn{\sqrt{\xi(s,a)m}\prn{\wtilde P(s^\prime\mid s,a)-P(s^\prime\mid s,a)}}_{(s,a,s^\prime)\in \gS\times\gA\times\gS}\cweak \prn{\mathring Z_{s,a,s^\prime}}_{(s,a,s^\prime)\in \gS\times\gA\times\gS}
\end{equation*} 
and $\frac{\xi(s,a)m}{\sum_{i=1}^m\ind\{(s_i,a_i)=(s,a)\}}\cas 1$, $\forall (s,a)\in\gS\times\gA$.
We can get 
\begin{equation*}
    \prn{\sqrt{\xi(s,a)m}\prn{\what P_m(s^\prime\mid s,a)-P(s^\prime\mid s,a)}}_{(s,a,s^\prime)\in \gS\times\gA\times\gS}\cweak \prn{\mathring Z_{s,a,s^\prime}}_{(s,a,s^\prime)\in \gS\times\gA\times\gS}
\end{equation*} 
via the Slutsky's lemma.
Then we may finish the proof using exactly the same strategy as in the proof of Theorem~\ref{Theorem_weak_convergence_empirical_process}.
\end{proof}

\begin{proof}[Proof of Theorem~\ref{Theorem_weak_convergence_empirical_process_unknown_reward}]
\textbf{Proof of part (a).} We begin by proving (a).
Mirroring the methodology applied in the proof of Theorem~\ref{Theorem_weak_convergence_empirical_process} in Appendix~\ref{Appendix_analysis_asymptotic}, we have
\begin{equation*}
    \sqrt{n}\prn{\gI-\gT^\pi}(\tilde\eta_n^\pi-\eta^\pi)=\underbrace{\sqrt{n}\prn{\wtilde \gT_n^\pi-\gT^\pi}\eta^\pi}_{(1)}+\underbrace{\sqrt{n}\prn{\wtilde \gT_n^\pi-\gT^\pi}(\tilde\eta_n^\pi-\eta^\pi)}_{(2)}.
\end{equation*}
It can be easily verified that term $(2)$ is $o_P(1)$ in $\brk{\ell^\infty(\FW)}^\gS$.
We only need to show that term $(1)$ converges weakly to $\bar{\GB}^\pi$ in $\brk{\ell^\infty(\FW)}^\gS$.

To this end, we first introduce the Cram\'er norm $\norm{\cdot}_{\ell_2}$.
For any zero-mass sign measure $\mu$, $\norm{\mu}_{\ell_2}:=\ell_2\prn{\mu_+,\mu_-}=\sqrt{\int_0^{\frac{1}{1-\gamma}}\prn{F_{\mu_+}(x)-F_{\mu_-}(x)}^2 dx}$.
Next, we define the Cram\'er space $\gM^0_{\ell_2}$ equipped with the Cram\'er norm:
\begin{equation*}
    M^0_{\ell_2}:=\brc{\mu\text{ signed measure on }\prn{\brk{0,\frac{1}{1-\gamma}},\gB_0}\mid\norm{\mu}_{\ell_2}<\infty,\mu\prn{\brk{0,\frac{1}{1-\gamma}}}=0}.
\end{equation*}
We can verify that $\norm{\cdot}_{\FW}\leq\frac{1}{\sqrt{1-\gamma}}\norm{\cdot}_{\ell_2}$ by Cauchy-Schwarz inequality, and $\overline{M^0_{\ell_2}}$ is a Hilbert space.
By CLT in Hilbert spaces (See Theorem 10.5 in \cite{ledoux2013probability}), we have term $(1)$ converges weakly to $\bar{\GB}^\pi$ in $\prn{M^0_{\ell_2}}^\gS$. 
Since $\norm{\cdot}_{\FW}\leq\frac{1}{\sqrt{1-\gamma}}\norm{\cdot}_{\ell_2}$, we have term $(1)$ also converges weakly to $\bar{\GB}^\pi$ in $\brk{\ell^\infty(\FW)}^\gS$. 

\textbf{Proof of part (b) and (c).}
We will only prove (b), as (c) can be shown in the same way.
According to the analysis in the proof of Theorem~\ref{Theorem_unknown_reward_non_asymp}, we have
\begin{equation*}
\begin{aligned}
    \sqrt{n}\prn{\tilde\eta_n^\pi-\eta^\pi}&=\sqrt{n}\prn{\tilde\eta_n^\pi-\hat\eta^\pi_n}+\sqrt{n}\prn{\hat\eta^\pi_n-\eta^\pi}.
\end{aligned}
\end{equation*}
It is known in Theorem~\ref{Theorem_unknown_reward_non_asymp} that $\sqrt{n}\prn{\hat\eta^\pi_n-\eta^\pi}$  converges weakly to $\prn{\gI-\gT^\pi}^{-1}\wtilde\GB^\pi$ in $\brk{\ell^\infty(\FKS)}^\gS$.
We only need to analyze $\sqrt{n}\prn{\tilde\eta_n^\pi-\hat\eta^\pi_n}$.
\begin{equation*}
    \sqrt{n}\prn{\gI-\what\gT^\pi_n}(\tilde\eta_n^\pi-\hat\eta^\pi_n)=\underbrace{\sqrt{n}\prn{\wtilde \gT_n^\pi-\what\gT^\pi_n}\hat\eta^\pi_n}_{(1)}+\underbrace{\sqrt{n}\prn{\wtilde \gT_n^\pi-\what\gT^\pi_n}(\tilde\eta_n^\pi-\what\eta^\pi_n)}_{(2)}.
\end{equation*}
It can be easily verified that term $(2)$ is $o_P(1)$ in $\brk{\ell^\infty(\FKS)}^\gS$, $\prn{\gI-\what\gT^\pi}\eta$ converges in probability to $\prn{\gI-\gT^\pi}\eta$ for any $\eta$ in $\brk{\gM^0_{\FKS}}^\gS$, and $\hat{\eta}^\pi_n$ converges in probability to $\eta^\pi$ in $\brk{\ell^\infty(\FKS)}^\gS$.
By Slutsky's theorem, we only need to show that $\sqrt{n}\prn{\wtilde \gT_n^\pi-\what\gT^\pi_n}\eta^\pi$ converges weakly to $\GB^\pi_R$ in $\brk{\ell^\infty(\FKS)}^\gS$.
Note that $\what{P}_n$ converges in probability to $P$, hence for any $s\in\gS$,
\begin{equation*}
    \begin{aligned}
        \brk{ \sqrt{n}\prn{\wtilde\gT_n^\pi-\what\gT^\pi_n}\eta^\pi}(s)=\sum_{a\in\gA,s^\prime\in\gS}\pi(a\mid s)\what{P}_n(s^\prime\mid s,a)\int_0^1 \prn{b_{r,\gamma}}_\#\eta^\pi(s^\prime)\sqrt{n}\prn{\what{\gP}_{R,n}(dr\mid s,a)-\gP_R(dr\mid s,a)}
    \end{aligned}
\end{equation*}
converges weakly to $\GB^\pi_R(s)$ in $\ell^\infty(\FKS)$ by Slutsky's theorem.
\end{proof}
% \section{Omitted Proofs in Section~\ref{Section_asymp}}\label{Appendix_proof_asymp}
\section{Omitted Proofs in Section~\ref{Section_inference}}\label{Appendix_proof_app}
\begin{proof}[Proof of Theorem~\ref{Theorem_inference_W1_ball}]
    % $\PB\prn{W_1(\hat\eta_n^\pi(s),\eta^\pi(s))\in C_1}\geq1-\alpha$ is a direct corollary of Theorem~\ref{Theorem_bound_of_Wdist}.
    Note that $\sqrt{n}W_1(\hat\eta_n^\pi(s),\eta^\pi(s))=\sup_{f\in\FW}\sqrt{n}\prn{\hat\eta_n^\pi(s)-\eta^\pi(s)}f $ and also $\sqrt{n}\prn{\hat\eta_n^\pi(s)-\eta^\pi(s)}\cweak\brk{\prn{\gI-\gT^\pi}^{-1}\wtilde{\GB}^\pi}(s)$ in $\ell^\infty(\FW)$, thus 
    \begin{equation*}
        \sqrt{n}W_1(\hat\eta_n^\pi,\eta^\pi)\cweak \sup_{f\in\FW}\brk{\prn{\gI-\gT^\pi}^{-1}\wtilde{\GB}^\pi}(s) f
    \end{equation*} 
    by continuous mapping theorem. It follows that $\lim_{n\to\infty}\PB\prn{W_1(\hat\eta_n^\pi(s),\eta^\pi(s))\in C_1(\alpha)}=1-\alpha$.
    
    Also, $\sqrt{n}W_1(\hat\eta_n^\pi(s),\eta^\pi(s))=\int_0^{\frac{1}{1-\gamma}}\abs{\sqrt{n}\prn{\hat\eta_n^\pi(s)-\eta^\pi(s)}\ind{(-\infty,x]}}dx $.
    and $\sqrt{n}\prn{\hat\eta_n^\pi(s)-\eta^\pi(s)}\cweak\brk{\prn{\gI-\gT^\pi}^{-1}\wtilde{\GB}^\pi}(s)$ in $\ell^\infty(\FKS)$ when Assumption~\ref{Assumption_reward_bounded_density} holds.
    Therefore, 
    \begin{equation*}
        \sqrt{n}W_1(\hat\eta_n^\pi(s),\eta^\pi(s))\cweak \int_0^{\frac{1}{1-\gamma}}\abs{\brk{\prn{\gI-\gT^\pi}^{-1}\wtilde\GB^\pi}(s)\ind_{(-\infty,x]}}dx
    \end{equation*} 
    by continuous mapping theorem. 
    % It follows that $\lim_{n\to\infty}\PB\prn{W_1(\hat\eta_n^\pi(s),\eta^\pi(s))\in C_3}=1-\alpha$ when Assumption~\ref{Assumption_reward_bounded_density} holds.
\end{proof}

\begin{proof}[Proof of Proposition~\ref{Proposition_plug_in_W1}]
Suppose $p\in\prn{\Delta(\gS)}^{\gS\times\gA}$ is a transition dynamic, $\eta\in\prn{\Delta\prn{\brk{0,\frac{1}{1-\gamma}}}}^\gS$, and $z\in\RB^{\gS\times\gA\times\gS}$.
We define $\gT^\pi(p)$ as the distributional Bellman operator associated with $p$, and $\Sigma(p)\in\RB^{\prn{\gS\times\gA\times\gS}\times \prn{\gS\times\gA\times\gS}}$ as the covariance matrix associated with $p$ defined in Theorem~\ref{Theorem_weak_convergence_empirical_process}.
Note that the covariance matrix is degenerate such that $\sum_{s,s^\prime\in\gS,a\in\gA}\brk{\Sigma(p)^{\frac{1}{2}} z}_{s,a,s^\prime}=0$, $\forall z\in\RB^{\gS\times\gA\times\gS}$.
\begin{equation*}
    \begin{aligned}
        \brk{g_1(\eta,p,z)}(s)&:=\sum_{a\in\gA}\pi(a\mid s)\sum_{s^\prime\in\gS} \brk{\Sigma(p)^{\frac{1}{2}} z}_{s,a,s^\prime}\int_0^1 \prn{b_{r,\gamma}}_\#\eta(s^\prime)d\gP_R(dr\mid s,a)\in M^0_{\FW},\\
        g_2(p,\nu)&:=\prn{\gI-\gT^\pi(p)}^{-1}\nu\in \ell^\infty(\FW),\\
        G(p,\eta,z)&:=\sup_{f\in\FW}g_2(p,g_1(\eta,p,z))f\in\RB.
    \end{aligned}
\end{equation*}
Define the metric $d(\nu_1,\nu_2):=\sup_{s\in\gS}\norm{\nu_1(s)-\nu_2(s)}_{\FW}$ for $\nu_1,\nu_2\in \prn{\ell^{\infty}(\FW)}^\gS$.
% Then it is easy to verify that when measured by metric $d$, $g_2$ is continuous in $(p,\nu)$ when we consider the total variation distance for $p$ and the metric $d$ for $\nu$.
% Also, one can verify that given any $z$, when measured by metric $d$, $g_1$ is continuous in $(\eta,p)$ when we consider the total variation distance for $p$ and the metric $d$ for $\nu$.
Thus given any $z$, $G$ is continuous in $(\eta,p)$ when we consider the total variation distance for $p$ and the metric $d$ for $\nu$.
According to Lemma~\ref{Lemma_continuous_distribution_function_and_quantile}, the $\alpha$-quantile of $G(p,\eta,Z)$ where $Z$ is standard gaussian is continuous in $p$ and $\eta$ as long as the quantile function of $G(p,\eta,Z)$ is continuous at $\alpha$.
We have 1) $z_1$ is the quantile function of $G\prn{ P,\eta^\pi,Z}$ and $\hat z_1$ is the quantile function of $G\prn{ \what P,\hat\eta_n^\pi,Z}$; 2) $\what P\cp P$ and $d(\hat\eta_n^\pi,\eta^\pi)\cp 0$ by Theorem~\ref{Theorem_bound_of_Wdist}, the conclusion is true by the continuous mapping theorem.

% The conclusion for $z_2$ can be shown via similar arguments.
% Here the subscript $\eta^\pi$ means we use the distribution of return $\eta^\pi$ to construct the $\wtilde \GB^\pi$ and the subscript $P$ means we use the transition dynamic $P$ to construct the operator $\gT^\pi$.
% We may also define $\mu_{\what P,\eta^\pi}(s)$, $\mu_{P,\hat\eta^\pi}(s)$, $\mu_{\what P,\hat \eta^\pi}(s)$ accordingly. 
% We have
% \begin{equation*}
% \begin{aligned}
%         W_1\prn{\mu_{P,\eta^\pi}(s),\mu_{\what P,\hat \eta^\pi}(s)}\leq     W_1\prn{\mu_{P,\eta^\pi}(s),\mu_{\what P,\eta^\pi}(s)}+    W_1\prn{\mu_{\what P,\eta^\pi}(s),\mu_{\what P,\hat \eta^\pi}(s)}.
% \end{aligned}
% \end{equation*}
    
\end{proof}

\begin{proof}[Proof of Proposition~\ref{Proposition_plug_in_KS}]
    Let Assumption~\ref{Assumption_reward_bounded_density} hold.
    Suppose $p\in\prn{\Delta(\gS)}^{\gS\times\gA}$ is a transition dynamic, $\eta\in\prn{\Delta\prn{\brk{0,\frac{1}{1-\gamma}}}}^\gS$, and $z\in\RB^{\gS\times\gA\times\gS}$.
    % such that $\sum_{s^\prime\in\gS} z_{s,a,s^\prime}=0,\forall s\in\gS,a\in\gA$,
    We define $\gT^\pi(p)$ as the distributional Bellman operator associated with $p$, and $\Sigma(p)\in\RB^{\prn{\gS\times\gA\times\gS}\times \prn{\gS\times\gA\times\gS}}$ as the covariance matrix associated with $p$ defined in Theorem~\ref{Theorem_weak_convergence_empirical_process}.
    Note that the covariance matrix is degenerate such that $\sum_{s,s^\prime\in\gS,a\in\gA}\brk{\Sigma(p)^{\frac{1}{2}} z}_{s,a,s^\prime}=0$, $\forall z\in\RB^{\gS\times\gA\times\gS}$.
\begin{equation*}
    \begin{aligned}
        \brk{g_1(\eta,p,z)}(s)&:=\sum_{a\in\gA}\pi(a\mid s)\sum_{s^\prime\in\gS} \brk{\Sigma(p)^{\frac{1}{2}} z}_{s,a,s^\prime}\int_0^1 \prn{b_{r,\gamma}}_\#\eta(s^\prime)d\gP_R(dr\mid s,a)\in M^0_{{\FKS}},\\
        g_2(p,\nu)&:=\prn{\gI-\gT^\pi(p)}^{-1}\nu\in \ell^\infty(\FKS),\\
        G(p,\eta,z)&:=\sup_{f\in\FKS}g_2(p,g_1(\eta,p,z))f\in\RB.
    \end{aligned}
\end{equation*}
Define the metric $d(\nu_1,\nu_2):=\sup_{s\in\gS}\norm{\nu_1(s)-\nu_2(s)}_{\FKS}$ for $\nu_1,\nu_2\in \prn{\ell^{\infty}(\FKS)}^\gS$.
% Then it is easy to verify that when measured by metric $d$, $g_2$ is continuous in $(p,\nu)$ when we consider the total variation distance for $p$ and the metric $d$ for $\nu$.
% Also, one can verify that given any $z$, when measured by metric $d$, $g_1$ is continuous in $(\eta,p)$ when we consider the total variation distance for $p$ and the metric $d$ for $\nu$.
Thus given any $z$, $G$ is continuous in $(\eta,p)$ when we consider the total variation distance for $p$ and the metric $d$ for $\nu$.
According to Lemma~\ref{Lemma_continuous_distribution_function_and_quantile}, the $\alpha$-quantile of $G(p,\eta,Z)$ where $Z$ is standard gaussian is continuous in $p$ and $\eta$ as long as the quantile function of $G(p,\eta,Z)$ is continuous at $\alpha$.
We have 1) $z_2$ is the quantile function of $G\prn{ P,\eta^\pi,Z}$ and $\hat z_2$ is the quantile function of $G\prn{ \what P,\hat\eta_n^\pi,Z}$; 2) $\what P\cp P$ and $d(\hat\eta_n^\pi,\eta^\pi)\cp 0$ by Theorem~\ref{Theorem_bound_of_KS}, the conclusion for $z_2$ is true by the continuous mapping theorem. The conclusion for $z_3$ can be shown via similar arguments.
\end{proof}

\section{Technical Lemmas}
\begin{lemma}[Young's Convolution Inequality]\label{Lemma_Young_convolution_inequality}
    Suppose $f$ is in the Lebesgue space $L^p(\RB^d)$ and $g$ is in $L^q(\RB^d)$ and $\frac{1}{p}+\frac{1}{q}=\frac{1}{r}+1$ with $1\leq p,q,r\leq \infty$. Then $\norm{f\ast g}_r\leq\|f\|_p\|g\|_q$.
\end{lemma}
\begin{proof}
    See Theorem 3.9.4 in \cite{bogachev2007measure}.
\end{proof}

\begin{lemma}[Hoeffding's Lemma]\label{Lemma_Hoeffding_lemma}
    Suppose $X\in [a,b]$ is a random variable with $\EB X=0$, then for any $\lambda\in\RB$, 
    \begin{equation*}
        \EB \exp{(\lambda X)}\leq \exp\prn{\frac{\lambda^2(b-a)^2}{8}}
    \end{equation*}
\end{lemma}
\begin{proof}
    By Jensen's inequality, for any $x\in [a.b]$,
    \begin{equation*}
        \exp{(\lambda x)}\leq \frac{b-x}{b-a}\exp{(\lambda a)} + \frac{x-a}{b-a}\exp{(\lambda b)}.
    \end{equation*}
    So
    \begin{equation*}
        \EB\exp{(\lambda X)}\leq \frac{b}{b-a}\exp{(\lambda a)} - \frac{a}{b-a}\exp{(\lambda b)}:=\exp{(L(\lambda(b-a)))},
    \end{equation*}
    where $L(h)=\frac{ha}{b-a}+\log\prn{1+\frac{(1-e^h)a}{b-a}}$.
    We may find $L(0)=L^\prime(0)=0$ and $L^{\prime\prime}(h)=-\frac{abe^h}{(b-ae^h)^2}\leq \frac{1}{4}$.
    By Taylor's theorem, there is some $\theta\in[0,1]$ such that
    \begin{equation*}
        L(h)=L(0)+L^\prime(0)h+\frac{1}{2}h^2 L^{\prime\prime}(\theta h)\leq \frac{1}{8}h^2.
    \end{equation*}
    We complete the proof.
\end{proof}

\begin{lemma}[Hoeffding's Inequality]
\label{Lemma_Hoeffding_Inequality}
Suppose $\{X_1,...,X_n\}$ be $n$ i.i.d. random variables with values in $[0,1]$. 
Then for any $\delta\in(0,1)$, with probability at least $1-\delta$ we have
$$
\left|\EB X_1-\frac{\sum_{i=1}^n X_i}{n}\right|\leq\sqrt{\frac{\log (2/\delta)}{2n}}.
$$
\end{lemma}
\begin{proof}
See Theorem 2.2.6 in \cite{vershynin_2018}.
\end{proof}

% \begin{lemma}[Dvoretzky-Kiefer-Wolfowitz Inequality]
% \label{Lemma_DKW_Inequality}
%     Let $X_1,\dots,X_n$ be real-valued \textup{i.i.d.} random variables with cumulative distribution function $F(\cdot)$. Let $F_n$ denote the associated empirical distribution function defined by $F_n(x)=\frac{1}{n}\sum_{i=1}^n\ind\brc{X_i\leq x}$.
%     Then we have for any $\epsilon>0$,
%     \begin{equation*}
%         \PB\prn{\sup_{x\in\RB}\abs{F(x)-F_n(x)}>\epsilon}\leq 2e^{-2n\epsilon^2}.
%     \end{equation*}
% \end{lemma}
% \begin{proof}
%     See \cite{massart1990tight}.
% \end{proof}

\begin{lemma}[Chernoff's Inequality]
\label{Lemma_Chernoff_Inequality}
Let $X_i$ be independent Bernoulli random variables with parameter $p_i$. 
Consider their sum $S_N=\sum_{i=1}^N X_i$ and denote its mean by $\mu=\EB S_N$. 
Then, for any $t<\mu$, we have
$$
\PB\left(S_N<t\right)\leq e^{-\mu}\left(\frac{e\mu}{t}\right)^t.
$$
\end{lemma}
\begin{proof}
    See Theorem 2.3.1 in \cite{vershynin_2018}.
\end{proof}

\begin{lemma}[McDiarmid's inequality]\label{Lemma_McDiarmid_inequality}
    Let $f\colon \gX_1\times\cdots\times\gX_N$ satisfy the bounded differences property with bounds $c_1,\dots,c_N$, \ie~there are constants $c_1,\dots,c_N$ such that for all $i\in\brc{1,\dots,N}$, $x_1\in\gX_1,\dots,x_N\in\gX_N$,    $\sup_{x_i^\prime\in\gX_i}\abs{f\prn{x_1,\dots,x_i,\dots,x_N}-f\prn{x_1,\dots,x_i^\prime,\dots,x_N}}\leq c_i$.
    Consider independent random variables $X_1,\dots,X_N$ where $X_i\in\gX_i$ for all $i$, then for any $\delta\in(0,1)$, with probability at least $1-\delta$,
    \begin{equation*}
        \abs{f(X_1,\dots,X_N)-\EB f(X_1,\dots,X_N)}\leq \sqrt{\frac{\prn{\sum_{i=1}^N c_i^2} \log (2/\delta)}{2}}.
    \end{equation*}
\end{lemma}
\begin{proof}
    See Theorem 2.9.1 in \cite{vershynin_2018}.
\end{proof}

\begin{lemma}\label{Lemma_bounded_density_after_convolution}
Suppose $X_1,...,X_N$ are a sequence of independent r.v.s, $X_i$ has density $p_i(x)$ and $\sum_{i=1}^N X_i$ has density $p(x)$.
If $\sup_{x\in \RB}p_1(x)\leq C$, then $\sup_{x\in \RB}p(x)\leq C$.
\end{lemma}
\begin{proof}
    We have
    \begin{equation*}
        p(x)=\brk{\prn{\prn{\prn{p_1\ast p_2}\ast p_3}\cdots}\ast p_N}(x),
    \end{equation*}
    where $\brk{f\ast g}(x):=\int_{\RB} f(x-y)g(y) dy$.
    Therefore, $\brk{p_1\ast p_2}(x)=\EB p_1(x-X_2)$ and $\sup_{x\in\RB}\brk{p_1\ast p_2}(x)\leq C$.
    We may complete the proof by deduction.
\end{proof}

\begin{lemma}\label{Lemma_expected_maximal_of_sub_gaussian}
Let $X_1,...,X_N$ be any $N\geq2$ random variables such that for any $\lambda>0$,
\begin{equation*}
    \EB \exp{\prn{\lambda X_i}}\leq \exp{\prn{\frac{\sigma^2\lambda^2}{2}}}.
\end{equation*}
Then
\begin{equation*}
\EB\max_{i\in\brc{1,\dots,N}} \abs{X_i}\leq \sqrt{9\sigma^2\log N}.
\end{equation*}
\end{lemma}
% \begin{proof}
%     By Jensen's inequality, we have for any $\lambda>0$,
%     \begin{equation*}
%     \begin{aligned}
%         e^{\lambda\EB\max_{i\in\brc{1,\dots,N}} X_i}
%         &\leq \EB\brk{e^{\lambda\max_{i\in\brc{1,\dots,N}} X_i}}\\
%         &=\EB\brk{\max_{i\in\brc{1,\dots,N}}e^{\lambda X_i}}\\
%         &\leq\sum_{i=1}^N \EB\brk{e^{\lambda X_i}}\\
%         &\leq Ne^{\frac{\sigma^2 \lambda^2}{2}}.
%     \end{aligned}
%     \end{equation*}
%     Hence
%     \begin{equation*}
%         \EB\max_{i\in\brc{1,\dots,N}} X_i\leq \frac{\log N}{\lambda}+\frac{\sigma^2 \lambda}{2},
%     \end{equation*}
%     take $\lambda=\frac{\sqrt{2\log N}}{\sigma}$, we have
%         \begin{equation*}
%         \EB\max_{i\in\brc{1,\dots,N}} X_i\leq \sqrt{2\sigma^2\log N}.
%     \end{equation*}
% \end{proof}

\begin{proof}
    Fix $t>0$, for any $\lambda >0$,
    \begin{equation*}
    \begin{aligned}
        \PB\prn{X_i> t}&\leq \frac{\EB \exp{\prn{\lambda X}}}{\exp{(\lambda t)}}\\
        &\leq \exp{\prn{\frac{\sigma^2\lambda^2}{2}-\lambda t}}.
    \end{aligned}
    \end{equation*}
    Setting $\lambda=\frac{t}{\sigma^2}$ yields 
    \begin{equation*}
        \PB\prn{X_i> t}\leq \exp{\prn{-\frac{t^2}{2\sigma^2}}}
    \end{equation*}
    and
    \begin{equation*}
        \PB\prn{|X_i|> t}\leq 2\exp{\prn{-\frac{t^2}{2\sigma^2}}}.
    \end{equation*}
    Therefore,
    \begin{equation*}
        \PB\prn{\frac{|X_i|}{\sqrt{1+\log i}}> t}\leq 2\exp{\prn{-\frac{(1+\log i)t^2}{2\sigma^2}}}
    \end{equation*}
    and
    \begin{equation*}
        \PB\prn{\max_{i\in\brc{1,\dots,N}}\frac{|X_i|}{\sqrt{1+\log i}}> t}\leq \sum_{i=1}^N 2\exp{\prn{-\frac{(1+\log i)t^2}{2\sigma^2}}}=2\exp{\prn{-\frac{t^2}{2\sigma^2}}}\sum_{i=1}^N i^{-\frac{t2}{2\sigma^2}} 
    \end{equation*}
    Therefore,
    \begin{equation*}
    \begin{aligned}
        \EB \max_{i\in\brc{1,\dots,N}}\frac{|X_i|}{\sqrt{1+\log i}}&=\int_{0}^\infty \PB\prn{\max_{i\in\brc{1,\dots,N}}\frac{|X_i|}{\sqrt{1+\log i}}> t} dt\\
        &\leq 2\sigma + \int_{2\sigma}^\infty \PB\prn{\max_{i\in\brc{1,\dots,N}}\frac{|X_i|}{\sqrt{1+\log i}}> t} dt\\
        &\leq 2\sigma + \int_{2\sigma}^\infty 4\exp{\prn{-\frac{t^2}{2\sigma^2}}} dt\\
        &\leq 3\sigma.
    \end{aligned}
    \end{equation*}
    Then we have
    \begin{equation*}
        \begin{aligned}
            \EB\max_{i\in\brc{1,\dots,N}} \abs{X_i}\leq \sqrt{1+\log N}  \brk{\EB \max_{i\in\brc{1,\dots,N}}\frac{|X_i|}{\sqrt{1+\log i}}}\leq \sqrt{1+\log N}3\sigma =\sqrt{9\sigma^2\log N}.
        \end{aligned}
    \end{equation*}
\end{proof}

% \begin{lemma}\label{Lemma_expected_maximal_of_sub_gaussian}
% Let $X_1,...,X_N$ be any $N\geq2$ random variables such that for any $\lambda>0$,
% \begin{equation*}
%     \EB \exp{\prn{\lambda X_i}}\leq \exp{\prn{\frac{\sigma^2\lambda^2}{2}}}.
% \end{equation*}
% Then
% \begin{equation*}
% \EB\max_{i\in\brc{1,\dots,N}} X_i\leq \sqrt{2\sigma^2\log N}.
% \end{equation*}
% \end{lemma}
% \begin{proof}
% By Jensen's inequality, we have
% \begin{equation*}
%     \begin{aligned}
%     \exp{\prn{\lambda\EB \max_{i\in\brc{1,\dots,N}} X_i}} &\leq\EB\exp{\prn{\lambda\max_{i\in\brc{1,\dots,N}} X_i}}\\
%         &\leq \sum_{i=1}^N \EB\exp{\prn{\lambda X_i}}\\
%         &\leq N \exp{\prn{\frac{\sigma^2\lambda^2}{2}}}.
%     \end{aligned}
% \end{equation*}
% Hence for any $\lambda>0$,
% \begin{equation*}
%     \EB\max_{i\in\brc{1,\dots,N}} X_i\leq \frac{\log N}{\lambda}+\frac{\sigma^2\lambda}{2}.
% \end{equation*}
% We may obtain the inequality desired by setting $\lambda=\sqrt{\frac{2\log N}{\sigma^2}}$.
% \end{proof}

\begin{lemma}\label{Lemma_simple_weak_convergence}
    Suppose $\brc{X^{(k)},k=1,2,\dots}$ is a sequence of random vectors in $\RB^d$ and $X^{(k)}\cweak X$ in $\RB^d$.
    Then for $d$ fixed elements $v_1,\dots,v_d$ in a normed vector space $B$, we have ${\sum_{i=1}^d X^{(k)}_i v_i\cweak \sum_{i=1}^d X_i v_i}$ in $B$.
\end{lemma}

\begin{proof}
    It suffices to verify that for any bounded continuous functions $f\colon B\to\RB$, we have 
    \begin{equation*}
        \EB f\prn{\sum_{i=1}^d X^{(k)}_i v_i}\to \EB f\prn{\sum_{i=1}^d X_i v_i}.
    \end{equation*}
    We may define $g(x)=f\prn{\sum_{i=1}^d x_iv_i}$ as a function on $\RB^d$. 
    Since $f$ is continuous, we have $\forall \epsilon>0$, $\exists \delta>0$ such that $\norm{u-v}\leq \delta$ implies $|f(u)-f(v)|\leq \epsilon$.
    Let $V:=\max_{i\in\{1,\dots,d\}}\norm{v_i}$, we have ${\norm{\sum_{i=1}^d x^{(1)}_i v_i-\sum_{i=1}^d x^{(2)}_i v_i}\leq \delta}$ as long as $\norm{x^{(1)}-x^{(2)}}_1\leq \delta/V$.
    Therefore, we may conclude $g$ is a bounded continuous function on $\RB^d$, and
    \begin{equation*}
             \EB f\prn{\sum_{i=1}^d X^{(k)}_i v_i}=\EB g\prn{X^{(k)}}\to \EB g\prn{X}=\EB f\prn{\sum_{i=1}^d X_i v_i}
    \end{equation*}
    due to $X^{(k)}\cweak X$.
\end{proof}
\begin{lemma}\label{Lemma_conv_of_quantile_from_conv_of_dist}
    Let $F^{-1}(p):=\inf\brc{t\mid F(t)\geq p}$ be the quantile function of cumulative distribution function $F$.
    For any sequence of cumulative distribution functions, $F_n^{-1}\cweak F^{-1}$ if and only if $F_n\cweak F$.
    Here $F_n^{-1}\cweak F^{-1}$ means that $F_n^{-1}(p)\cweak F^{-1}(p)$ for every $p$ where $F^{-1}$ is continuous.
\end{lemma}
\begin{proof}
    See Lemma 21.2 in \cite{van2000asymptotic}.
\end{proof}
\begin{lemma}\label{Lemma_continuous_distribution_function_and_quantile}
    Assume $\gX$ is a normed vector space and  $g(x,z)\colon \gX\times \RB^d\to\RB$ is a function such that for any fixed $z_0\in\RB^d$, $g(x,z_0)$ is continuous in $x$ at $x_0$.
    Suppose $Z$ is a random vector taking values in $\RB^d$, $F_x(t):=\PB\prn{g(x,Z)\leq t}$ is the distribution function of $g(x,Z)$ and $F^{-1}_x(p):=\inf\brc{t\mid F_x(t)\geq p}$ is the quantile function of $g(x,Z)$.
    Then we have
    \begin{enumerate}
        \item $F_x(t)$ is continuous in $x$ at $x_0$ whenever $t$ is a continuous point of $F_{x_0}(t)$;
        \item $F^{-1}_x(p)$ is continuous in $x$ at $x_0$ whenever $p$ is a continuous point of $F^{-1}_{x_0}(p)$.
    \end{enumerate}
\end{lemma}
\begin{proof}
    For any sequence $\{x_n\}$ such that $x_n\to x_0$, we have $g(x_n,Z)\cas g(x_0,Z)$, thus $F_{x_n}\cweak F_{x_0}$, which further implies $F_x(t)$ is continuous at $x_0$ for every $t$ where $F_{x_0}$ is continuous.
    Then we may use Lemma~\ref{Lemma_conv_of_quantile_from_conv_of_dist} to get $F^{-1}_{x_n}\cweak F^{-1}_{x_0}$, which implies $F^{-1}_{x_n}(p)\to F^{-1}_{x_0}(p)$ whenever $p$ is a continuous point of $F^{-1}$.
    Therefore, our conclusion follows.
\end{proof}

% \begin{lemma}[Functional Delta method]\label{Lemma_Functional_Delta}
%     Let $\DB$ and $\ZB$ be normed vector spaces.
%     Let $\phi\colon \DB_\phi\subset\DB\to\ZB$ be Hadamard differentiable at $\theta$ tangentially to $\DB_0$.
%     Let $T_n\colon \Omega_n\to\DB_\phi$ be maps such that $r_n(T_n-\theta)\cweak T$ for some sequence of numbers $r_n\to\infty$ and a random element $T$ that takes values in $\DB_0$.
%     Then $r_n(\phi(T_n)-\phi(\theta))\cweak \phi^\prime(T)$.
% \end{lemma}
% \begin{proof}
%     See Theorem 20.8 in \cite{van2000asymptotic}.
% \end{proof}

\end{document}